%% file: iclr2020_choco.tex
\documentclass{article} %
\usepackage{iclr2020_conference,times}

\PassOptionsToPackage{hyphens}{url}
\usepackage{hyperref}
\usepackage{url}
\usepackage{breakurl} %

\usepackage[utf8]{inputenc} %
\usepackage[T1]{fontenc}    %
\usepackage{hyperref}       %
\usepackage{url}            %
\usepackage{booktabs}       %
\usepackage{amsfonts}       %
\usepackage{nicefrac}       %
\usepackage{microtype}      %
\usepackage{xspace}
\usepackage{enumitem}
\usepackage{algorithm,algorithmic}
\usepackage{nicefrac}
\usepackage{multirow}
\usepackage{subfigure}
\usepackage{wrapfig}

\usepackage{amssymb}%
\usepackage{pifont}%
\newcommand{\cmark}{\ding{51}}%
\newcommand{\xmark}{\ding{55}}%

\usepackage{tikz}
\usetikzlibrary{decorations.pathreplacing,calc}
\newcommand{\tikzmark}[1]{\tikz[overlay,remember picture] \node (#1) {};}

\newcommand*{\SpaceReservedForComments}{0.7cm}%
\newcommand*{\HorizontalOffset}{-1em}%
\newcommand*{\VerticalOffset}{0.5ex}%
\newcommand*{\AddNote}[4][]{%
    \begin{tikzpicture}[overlay, remember picture]
        \draw [decoration={brace,amplitude=0.5em},decorate, thick,red, #1]
            ($(#3)+(\HorizontalOffset,-\VerticalOffset)$) --  ($(#2)+(\HorizontalOffset,\dimexpr\VerticalOffset+0.5em\relax)$)
            node [align=left, text width=\SpaceReservedForComments-1.0em+0.5em, pos=0.5, anchor=east] {#4};
    \end{tikzpicture}
}%
\newcommand*{\AddNoteNoB}[4][]{%
    \begin{tikzpicture}[overlay, remember picture]
        \draw [thick,red, white]
            ($(#3)+(\HorizontalOffset,-\VerticalOffset)$) --  ($(#2)+(\HorizontalOffset,\dimexpr\VerticalOffset+0.5em\relax)$)
            node [align=left, text width=\SpaceReservedForComments-1.0em+0.5em, pos=0.5, anchor=east, #1] {#4};
    \end{tikzpicture}
}%
\newlength\myindent
\usepackage{amsmath,amsthm}
\newtheorem{theorem}{Theorem}[section]
\newtheorem{lemma}[theorem]{Lemma}
\newtheorem{corollary}[theorem]{Corollary}
\newtheorem{definition}[theorem]{Definition}
\newtheorem{remark}[theorem]{Remark}
\newtheorem{assumption}{Assumption}

\newcommand{\choco}{\textsc{Choco-SGD}\xspace}
\newcommand{\sgap}{\rho}
\newcommand{\compr}{\delta}

\providecommand{\lin}[1]{\ensuremath{\left\langle #1 \right\rangle}}
\providecommand{\abs}[1]{\left\lvert#1\right\rvert}
\providecommand{\norm}[1]{\left\lVert#1\right\rVert}

\usepackage{xspace}  %
\newcommand{\algopt}{\textsc{Choco-SGD}\xspace} 
\newcommand{\algcons}{\textsc{Choco-Gossip}\xspace}

  \providecommand{\R}{\mathbb{R}} %

  \providecommand{\E}{{\mathbb E}}
  \providecommand{\EE}[2]{{\mathbb E}_{#1}\left.#2\right. }  %

  \providecommand{\0}{\mathbf{0}}
  \providecommand{\1}{\mathbf{1}}
  \renewcommand{\aa}{\mathbf{a}}
  \providecommand{\bb}{\mathbf{b}}

  \renewcommand{\gg}{\mathbf{g}}

  \providecommand{\mm}{\mathbf{m}}

  \providecommand{\qq}{\mathbf{q}}

  \providecommand{\uu}{\mathbf{u}}
  \providecommand{\vv}{\mathbf{v}}
  
  \providecommand{\xx}{\mathbf{x}}
  \providecommand{\yy}{\mathbf{y}}

  \providecommand{\cD}{\mathcal{D}}

  \providecommand{\cO}{\mathcal{O}}

  \usepackage[textwidth=3cm]{todonotes}

\title{Decentralized Deep Learning with Arbitrary Communication Compression}

\author{Anastasia Koloskova\thanks{Equal contribution.} \\
	\texttt{anastasia.koloskova@epfl.ch}
	\And
	 Tao Lin$^*$ \\
	 \texttt{tao.lin@epfl.ch} \hspace{10mm}
	\AND
	Sebastian U. Stich \\
	\texttt{sebastian.stich@epfl.ch}
	\And
	Martin Jaggi \\
	\texttt{martin.jaggi@epfl.ch}
	\AND
	\qquad \qquad \qquad \qquad \qquad \qquad \qquad \qquad \qquad EPFL\\
	\qquad \qquad \qquad \qquad \qquad \qquad \qquad \quad Lausanne, Switzerland \\
}

\iclrfinalcopy %
\begin{document}

\maketitle

\input{main.tex}

\clearpage
\appendix

\input{supp.tex}

\end{document}

%% file: main.tex
\begin{abstract}
Decentralized training of deep learning models is a key element for enabling data privacy and on-device learning over networks, as well as for efficient scaling to large compute clusters.
As current approaches are limited by network bandwidth, we propose the use of communication compression in the decentralized training context.
We show that \algopt achieves linear speedup in the number of workers for arbitrary high compression ratios on general \emph{non-convex} functions, and non-IID training data. 
We demonstrate the practical performance of the algorithm in two key scenarios: the training of deep learning models (i) over decentralized user devices, connected by a peer-to-peer network and (ii) in a datacenter. 
\end{abstract}

\section{Introduction}

Distributed machine learning---i.e.\ the training of machine learning models using distributed optimization algorithms---has recently enabled many successful applications in research and industry. %
Such methods offer two of the key success factors:\ 1) \emph{computational scalability} by leveraging the simultaneous computational power of many devices, and 2)~\emph{data-locality}, the ability to perform joint training while keeping each part of the training data local to each participating device.
Recent theoretical results indicate that decentralized schemes can be as efficient as the centralized approaches, at least when considering convergence of training loss vs.\ iterations~\citep{Scaman2017:optimal,Scaman2018:non-smooth,Lian2017:decentralizedSGD,Tang2018:decentralized,Koloskova:2019choco,Assran:2018sdggradpush}.

Gradient compression techniques have been proposed for the standard distributed training case~\citep{Alistarh2017:qsgd,Wen2017:terngrad,Lin2018:deep,
Wangni2018:sparsification,Stich2018:sparsifiedSGD}, to reduce the amount of data that has to be sent over each communication link in the network.
For decentralized training of deep neural networks, \citet{Tang2018:decentralized} introduce two algorithms (DCD, ECD) which allow for communication compression.
However, both %
these algorithms are restrictive with respect to the used compression operators, 
only allowing for unbiased compressors and---more significantly---so far not supporting arbitrarily high compression ratios. We here study \algopt---recently introduced for convex problems only~\citep{Koloskova:2019choco}---which overcomes these constraints. 

For the evaluation of our algorithm we in particular focus on the generalization performance (on the test-set) on standard machine learning benchmarks, hereby departing from previous work such as e.g.~\citep{Tang2018:decentralized,Wang2019:matcha,Tang2019:squeeze,Reisizadeh2019:quantimed} that mostly considered training performance (on the train-set). 
We study two different scenarios: firstly, (i)~training on a challenging peer-to-peer setting, where the training data is distributed over the training devices (and not allowed to move), similar to the federated learning setting~\citep{McMahan:2017fedAvg,Kairouz2019:federated}. We are again able to show speed-ups for \algopt over the decentralized baseline~\citep{Lian2017:decentralizedSGD} with much less communication overhead. Secondly, (ii)~training in a datacenter setting, where decentralized communication patterns allow better scalability than centralized approaches. For this setting we show that communication efficient \algopt can improve time-to-accuracy on large tasks, such as e.g.\ ImageNet training. %
However, when investigating the scaling of decentralized algorithms to larger number of nodes  we observe that (all) decentralized schemes encounter difficulties and often do not reach the same (test and train) performance as centralized schemes. As these findings point out some deficiencies of current decentralized training schemes (and are not particular to our scheme) we think that reporting these results is a  helpful contribution to the community to spur further research on decentralized training schemes that scale to large number of peers.

\paragraph{Contributions.}
Our contributions can be summarized as: \vspace{-0.7\baselineskip}
\begin{itemize}[leftmargin=12pt,itemsep=0pt,parsep=0pt]
 \item On the theory side, we are the first to show that \algopt converges at rate $\smash{\cO\big(\nicefrac{1}{\sqrt{nT}} + \nicefrac{1}{(\rho^2 \delta T)^{2/3}} \bigr)}$ on non-convex smooth functions, 
 where $n$ denotes the number of nodes, $T$~the number of iterations, $\rho$ the spectral gap of the mixing matrix and $\delta$ the compression ratio. The main term, $\smash{\cO\big(\nicefrac{1}{\sqrt{nT}}\big)}$, matches with the centralized baselines with exact communication and shows a linear speedup in the number of workers $n$. Both $\rho$ and $\delta$ only affect the asymptotically smaller second term.
\item On the practical side, we present a version of \algopt with momentum and analyze its practical performance on two relevant scenarios: 
 \begin{itemize}[nosep,leftmargin=12pt]
 \item[$\circ$] for \emph{on-device training} over a realistic peer-to-peer social network, where lowering the bandwidth requirements of joint training is especially impactful 
 \item[$\circ$] in a datacenter setting for \emph{computational scalability} of training deep learning models for resource efficiency and improved time-to-accuracy  
 \end{itemize}
 \item Lastly, we systematically investigate performance of the decentralized schemes when scaling to larger number of nodes and we point out some (shared) difficulties encountered by current decentralized learning approaches.
 \vspace{-1mm}
 
\end{itemize}

\section{Related Work}
\label{sec:related}
For the training in communication restricted settings a variety of methods have been proposed.
For instance, decentralized schemes~\citep{Lian2017:decentralizedSGD,Nedic2018:toplogy,Koloskova:2019choco},
gradient compression~\citep{Seide2015:1bit,Strom:2015wc,Alistarh2017:qsgd,Wen2017:terngrad,Lin2018:deep,Wangni2018:sparsification,Bernstein2018:sign,Lin2018:deep,Alistarh2018:topk,%
Stich2018:sparsifiedSGD,KarimireddyRSJ2019feedback},
asynchronous methods~\citep{Recht2011:hogwild,Assran:2018sdggradpush},
coordinate updates~\cite{Nesterov:2012,Hydra,StichRJ17safe,StichRJ17steepest,cola2018nips},
or performing  multiple local SGD steps before averaging~\citep{zhang2016parallelLocalSGD,McMahan:2017fedAvg,Stich2018:LocalSGD,lin2020dont}. This especially covers learning over decentralized data, as extensively studied in the federated learning literature for the centralized algorithms \citep{McMahan16:FedLearning,Kairouz2019:federated}.
In this paper we advocate for combining decentralized SGD schemes with gradient compression.

\textbf{Decentralized SGD.} 
We in particular focus on approaches based on gossip averaging~\citep{Kempe2003:gossip,Xiao2014:averaging,%
Boyd2006:randgossip}
whose convergence rate typically depends on the spectral gap $\rho \geq 0$ of the mixing matrix~\citep{Xiao2014:averaging}.
\citet{Lian2017:decentralizedSGD} combine SGD with gossip averaging and show that the leading term in the convergence rate $\smash{\cO\big(\nicefrac{1}{\sqrt{nT}}\bigr)}$ is consistent with the convergence of the centralized mini-batch SGD~\citep{Dekel2012:minibatch} and the spectral gap only affects the asymptotically smaller terms. Similar results have been observed very recently for related schemes~\citep{Scaman2017:optimal,Scaman2018:non-smooth,Koloskova:2019choco,Yu2019momentum}.

\textbf{Quantization.} 
Communication compression with quantization has been popularized in the deep learning community by the reported successes in \citep{Seide2015:1bit,Strom:2015wc}. Theoretical guarantees were first established for schemes with unbiased compression~\citep{Alistarh2017:qsgd,Wen2017:terngrad,Wangni2018:sparsification} 
but soon extended to biased compression~\citep{Bernstein2018:sign} as well. Schemes with error correction work often best in practice and give the best theoretical gurantees~\citep{Lin2018:deep,Alistarh2018:topk,%
Stich2018:sparsifiedSGD,KarimireddyRSJ2019feedback,stich2019error}.
Recently, also proximal updates and variance reduction have been studied in combination with quantized updates~\citep{Mishchenko2019:diana,Horvath2019:vr}.

\textbf{Decentralized Optimization with Quantization.}
It has been observed that gossip averaging can diverge (or not converge to the correct solution) in the presence of quantization noise~\citep{Xiao2005:drift,Carli2007:noise,%
Nedic2008:quantizationeffects,Dimakis2010:survey,Carli2010:quantizedconsensus,Yuan2012:distributedquant}.
\citet{Reisizadeh2018:DQGD} propose an algorithm that can still converge, though at a slower rate than the exact scheme. Another line of work proposed adaptive schemes (with increasing compression accuracy) that converge at the expense of higher communication cost~\citep{Carli2010:codingschemes,Thinh:2018AdaptiveDecentralizedQuantized,Berahas2019:adaptive}.  
For deep learning applications, 
\citet{Tang2018:decentralized} proposed the DCD and ECD algorithms that converge at the same rate as the centralized baseline %
though only 
for \emph{constant} compression ratio.
The \algopt algorithm that we consider in this work can deal with \emph{arbitrary} high compression, and has been introduced in~\citep{Koloskova:2019choco} but only been analyzed for convex functions.
For non-convex functions we show a rate of $\smash{\cO\big(\nicefrac{1}{\sqrt{nT}} + \nicefrac{1}{(\rho^2 \delta T)^{\frac{2}{3}}}  \bigr)}$, where $\delta > 0$ measures the compression quality.
Simultaneous work of \cite{Tang2019:squeeze} introduced DeepSqueeze, an alternative method which also converges with arbitrary compression ratio. In our experiments, under the same amount of tuning, \algopt achieves higher test accuracy.

\section{\choco}
\label{sec:basics}
In this section we formally introduce the decentralized optimization problem, compression operators, and the gossip-based stochastic optimization algorithm \algopt from~\citep{Koloskova:2019choco}.

\textbf{Distributed Setup.} We consider optimization problems distributed across $n$ nodes of the form%
\vspace{-1mm}
\begin{align}
f^\star :=  \min_{\xx \in \R^d} \bigg[  f(\xx) := \frac{1}{n} \sum_{i=1}^n f_i(\xx) \bigg]\,, && f_i(\xx) := \EE{\xi_i\sim D_i}{F_i(\xx, \xi_i)}, && \forall i\in [n]\,, \label{eq:prob}
\end{align}
where $D_1,\dots D_n$ are local distributions for sampling data which can be different on every node, $F_i \colon \R^d \times \Omega \to \R$ are possibly non-convex (and non-identical) loss functions. This setting covers the important case of empirical risk minimization in distributed machine learning and deep learning applications. 

\textbf{Communication.} Every device is only allowed to communicate with its local neighbours defined by the network \emph{topology}, given as a weighted graph $G = ([n], E)$, with edges $E$ representing the communication links along which  messages (e.g.\ model updates) can be exchanged. We assign a positive weight $w_{ij}$ to every edge ($w_{ij} = 0$ for disconnected nodes $\{i,j\} \notin E$). 
\begin{assumption}[Mixing matrix]\label{assump:W}
	We assume that $W \in [0,1]^{n \times n}$, $(W)_{ij} = w_{ij}$ is a symmetric ($W=W^\top$) doubly stochastic ($W\1=\1$,$\1^\top W = \1^\top$) matrix with eigenvalues  $1 = |\lambda_1(W)| > |\lambda_2(W)| \geq \dots \geq |\lambda_n(W)|$ and spectral gap $\sgap := 1 - |\lambda_2(W)| \in (0,1]\,. \label{def:spectral_gap}$

\end{assumption}
In our experiments we set the weights based on the local node degrees: $w_{ij} = \max\{\deg(i),\deg(j)\}^{-1}$ for $\{i,j\} \in E$. This will not only guarantee $\rho > 0$ but these weights can easily be computed in a local fashion on each node~\citep{Xiao2014:averaging}.

\textbf{Compression.} We aim to only transmit \emph{compressed} (e.g.\ quantized or sparsified) messages. We formalized this through the notion of compression operators that was e.g.\ also used in~\citep{Tang2018:decentralized,Stich2018:sparsifiedSGD}.
\begin{definition}[Compression operator]\label{assump:q}
	$Q \colon \R^d \to \R^d$ is a compression operator if it satisfies
	\begin{align}
	\EE{Q}{\norm{Q(\xx) - \xx}}^2 &\leq (1 - \compr) \norm{\xx}^2, & &\forall \xx \in \R^d \,, \label{def:omega}
	\end{align}
	for a parameter $\compr > 0$. Here $\mathbb{E}_Q$ denotes the expectation over the internal randomness of operator $Q$.
\end{definition}
In contrast to the quantization operators used in e.g.~\citep{Alistarh2017:qsgd,Horvath2019:vr}, compression operators defined as in~\eqref{def:omega} are not required to be unbiased and therefore supports a larger class of compression operators. Some examples can be found in~\citep{Koloskova:2019choco} and we further discuss specific compression schemes in Section~\ref{sect:exp}.

\textbf{Algorithm.} \algopt is summarized in Algorithm~\ref{alg:choco}.%
\begin{figure*}[tb]
\newlength{\mycorrect}
\setlength{\mycorrect}{\widthof{$\hat{\xx}^{(t)}$}}
\algsetup{
  linenodelimiter = {\makeatletter\tikzmark{\arabic{ALC@line}}\makeatother:\phantom{$\hat{\xx}^{(t)}$}\hspace{-\mycorrect}}
}
\algsetup{
  linenodelimiter = {\makeatletter\tikzmark{\arabic{ALC@line}}\makeatother:}
}
\resizebox{1\linewidth}{!}{
\begin{minipage}{1.1\linewidth}
\begin{algorithm}[H]
	\caption{\choco~\citep{Koloskova:2019choco}}
	\text{\textbf{input:} Initial value $\overline{\xx}^{(-\frac{1}{2})} \in \R^d$, $\xx_i^{(-\frac{1}{2})} = \overline{\xx}^{(-\frac{1}{2})}$ on each node $i \in [n]$, consensus stepsize $\gamma$, SGD stepsize $\eta$,}\\
	\text{\qquad \quad communication graph $G = ([n], E)$ and mixing matrix $W$, initialize $\hat{\xx}_i^{(0)} := \0$ $\forall i \in [n]$}  \par\vspace{1mm}
	\hspace*{\SpaceReservedForComments}{}%
\begin{minipage}{\dimexpr\linewidth-\SpaceReservedForComments\relax}
\fontsize{10}{14}\selectfont %
	\begin{algorithmic}[1]
		\FOR[{{\it in parallel for all workers $i \in [n]$}}]{$t$\textbf{ in} $0\dots T-1$}
		\STATE $\xx_i^{(t)} := \xx_i^{(t - \frac{1}{2})} + \gamma \textstyle\sum_{j: \{i, j\}\in E} w_{ij} \bigl(\hat{\xx}^{(t)}_j \!- \hat{\xx}^{(t)}_i\bigr)$ \hfill $\triangleleft$ modified gossip averaging
		\STATE  $\qq_i^{(t)} := Q(\xx_i^{(t)} - \hat{\xx}_i^{(t)})$	\hfill$\triangleleft$ compression
		\FOR{neighbors $j \colon \{i,j\} \in E$ (including $\{i\} \in E$)} 
		\STATE Send $\qq_i^{(t)}$ and receive $\qq_j^{(t)}$ %
		\hfill$\triangleleft$ communication 
		\STATE $\hat{\xx}^{(t+1)}_j := \qq^{(t)}_j + \hat{\xx}_j^{(t)}$  \hfill$\triangleleft$ local update
		\ENDFOR

		\STATE Sample $\xi_i^{(t)}$, compute gradient $\gg_i^{(t)} \!:= \nabla F_i(\xx_i^{(t)}\!, \xi_i^{(t)})$\!
		\STATE $\xx_i^{(t + \frac{1}{2})} := \xx_i^{(t)} - \eta \gg_i^{(t)}$ %
				\hfill$\triangleleft$ stochastic gradient update
		\ENDFOR 
	\end{algorithmic}\label{alg:choco}
	\end{minipage}
\AddNote[gray]{3}{7}{\raisebox{.5pt}{\textcircled{\raisebox{-.9pt} {1}}}}
\AddNote[gray]{8}{9}{\raisebox{.5pt}{\textcircled{\raisebox{-.9pt} {2}}}}
\end{algorithm}
\end{minipage}
}
\vspace{-0.5cm}
\end{figure*}
Every worker $i$ stores its own private variable $\xx_i \in \R^d$ that is updated by a stochastic gradient step in part \raisebox{.5pt}{\textcircled{\raisebox{-.9pt} {2}}} and a modified gossip averaging step on line 2. This step is a key element of the algorithm as it preserves the averages of the iterates even in presence of quantization noise (the compression errors are not discarded, but aggregated in the local variables $\xx_i$, see also~\citep{Koloskova:2019choco}).
 The nodes communicate with their neighbors in part \raisebox{.5pt}{\textcircled{\raisebox{-.9pt} {1}}} and update the variables $\hat{\xx}_j \in \R^d$ for all their neighbors $\{i,j\}\in E$ only using compressed updates. These $\hat{\xx}_i$ are available to all the neighbours of the node $i$ and represent the `publicly available' copies of the private $\xx_i$, in general $\xx_i \neq \hat \xx_i$, due to the communication restrictions.

From an implementation aspect, it is worth highlighting that the communication part %
\raisebox{.5pt}{\textcircled{\raisebox{-.9pt} {1}}}
and the gradient computation part %
\raisebox{.5pt}{\textcircled{\raisebox{-.9pt} {2}}}
can both be executed in parallel because they are independent. %
Moreover, each node only needs to store 3 vectors at most, independent of the number of neighbors (this might not be obvious from the notation used here for additinal clarity, for further details c.f.\ \citep{Koloskova:2019choco}). We further propose a momentum-version of \algopt in Algorithm~\ref{alg:choco_with_momentum} (see Section~\ref{sec:momentum} for further details).
\section{Convergence of \algopt on Smooth Non-Convex Problems}
As the first main contribution, we extend the analysis of \algopt to non-convex problems. For this we make the following technical assumptions:
\begin{assumption}\label{assump:f}
	Each function $f_i \colon \R^d \to \R$ for $i \in [n]$ is $L$-smooth, that is
	\begin{align}
	&\norm{\nabla f_i(\yy)-\nabla f_i(\xx)} \leq L \norm{\yy-\xx}\,, & & & & \forall \xx,\yy \in \R^d, i \in [n], \notag \\
\intertext{ and the variance of the stochastic gradients is bounded on each worker:}
\label{eq:assump_second_momentum}
	&\EE{\xi_i}{\norm{\nabla F_i(\xx, \xi_i) - \nabla f_i(\xx)}^2}\leq \sigma_i^2\,, &	&\EE{\xi_i}{\norm{\nabla F_i(\xx, \xi_i)}}^2 \leq G^2\,,   & &\forall \xx \in \R^d, i \in [n], 
	\end{align}
	where $\mathbb{E}_{\xi_i}[\cdot]$ denotes the expectation over $\xi_i \sim \cD_i$.
	We also denote $\overline{\sigma}^2 := \frac{1}{n}\sum_{i = 1}^n\sigma_i^2$ for convenience.
\end{assumption}

\begin{theorem}\label{th:non-convex-sigma}
	Under Assumptions \ref{assump:W}--\ref{assump:f} there exists a constant stepsize $\eta$ and the consensus stepsize from~\citep{Koloskova:2019choco}, $\gamma := \frac{\sgap^2 \compr}{16\sgap + \sgap^2 + 4 \beta^2 + 2 \sgap\beta^2 - 8 \sgap \compr}$ with $\beta = \norm{I-W}_2 \in [0,2]$, such that the averaged iterates $\overline{\xx}^{(t)} := \tfrac{1}{n} \sum_{i=1}^n \xx_i^{(t)}$ of Algorithm~\ref{alg:choco} %
	satisfy:
	\begin{align*}
	\frac{1}{T + 1}\sum_{t = 0}^{T}\Big\|\nabla f \bigl(\overline{\xx}^{(t)} \bigr)\Big\|_2^2
	& =\cO \Bigg( \left( \frac{L F_0 \overline{\sigma}^2 }{n (T + 1)} \right)^{1/2} +  \left(\frac{ G L F_0}{c (T + 1 )}\right)^{2/3} + \frac{ L F_0}{T + 1}\Bigg)
	\end{align*}
	where
	$c := \tfrac{\sgap^2 \compr}{82}$ denotes the convergence rate of the underlying %
	consensus averaging scheme of~\citep{Koloskova:2019choco}, $F_0 := f(\overline{\xx}^{(0)}) - f^\star$.
\end{theorem}
This result shows that \algopt converges as $\smash{\cO\big(\nicefrac{1}{\sqrt{nT}} + \nicefrac{1}{(\rho^2 \delta T)^{2/3}} \bigr)}$. The first term shows a linear speed-up compared to SGD on a single node, while compression and graph topology affect only  the higher order second term. 
In the special case when exact averaging without compression is used ($\delta = 1)$ , then $c=\rho$ and the rate improves to $\smash{\cO\big(\nicefrac{1}{\sqrt{nT}} + \nicefrac{1}{(\rho T)^{2/3}} \bigr)}$, recovering the rate in~\citep{Wang2018:cooperativeSGD}.
This upper bound improves slightly over \citep{Lian2017:decentralizedSGD} that shows $\smash{\cO\big(\nicefrac{1}{\sqrt{nT}} + \nicefrac{n}{(n \rho T)^{2/3}} \bigr)}$.\footnote{Theorem 1 of \cite{Lian2017:decentralizedSGD} and stepsize tuned with Lemma~\ref{lem:tuning_stepsize}.}
For the proofs and convergence of the individual iterates $\xx_i$ we refer to Appendix~\ref{sec:theorymain}.

\section{Comparison to Baselines for Various Compression Schemes}\label{sect:exp}
In this section we experimentally compare \algopt to the relevant baselines for a selection of commonly used compression operators. 
For the experiments we further leverage momentum in all implemented algorithms. The newly developed momentum version of \algopt is given as Algorithm~\ref{alg:choco_with_momentum}.%

\begingroup
\algsetup{
	linenodelimiter = {\makeatletter\tikzmark{\arabic{ALC@line}}\makeatother:}
}%
\resizebox{1\linewidth}{!}{
\begin{minipage}{1.1\linewidth}
\begin{algorithm}[H]
	\caption{\choco with Momentum}
	\text{\textbf{input:} The same as for Algorithm~\ref{alg:choco}, additionally: weight decay factor $\lambda$, momentum factor $\beta$, }
	\text{\phantom{\textbf{input:}} local momentum memory $\vv_i^{(0)} \!:= \0$, $\forall i \in [n]$} \par\vspace{2mm}
	
	\text{ Lines 1--8 in Algorithm~\ref{alg:choco} are left unmodified} \par\vspace{1mm}
	\text{ Line 9 in Algorithm~\ref{alg:choco} is replaced with the following two lines} \par\vspace{1mm}
	\hspace*{\SpaceReservedForComments}{}%
	\begin{minipage}{\dimexpr\linewidth-\SpaceReservedForComments\relax}
		\fontsize{10}{14}\selectfont %
		\begin{algorithmic}[1]
\setcounter{ALC@line}{8}
			\STATE $\vv^{(t + 1)}_i := (\gg_i^{(t)} + \lambda \xx^{(t)}_i) + \beta \vv^{(t)}_i$
			\hfill$\triangleleft$ local momentum with weight decay
			\STATE $\xx_i^{(t + \frac{1}{2})} := \xx_i^{(t)} - \eta \vv^{(t + 1)}_i$ %
			\hfill$\triangleleft$ stochastic gradient update 
		\end{algorithmic}\label{alg:choco_with_momentum}
	\end{minipage}
\end{algorithm}
\end{minipage}
}
\endgroup

\paragraph{Setup.}
In order to match the setting in~\citep{Tang2018:decentralized} for our first set of experiments,
we use a ring topology with $n=8$ nodes and
train the \texttt{ResNet20} architecture~\citep{He2016:Resnet} on the \texttt{Cifar10} dataset (50K/10K training/test samples)~\citep{Krizhevsky2012:cifar10}. 
We randomly split the training data between workers and shuffle it after every epoch, following standard procedure as e.g.\ in~\citep{goyal2017accurate}. %
We implement DCD and ECD with momentum~\citep{Tang2018:decentralized}, DeepSqueeze with momentum \citep{Tang2019:squeeze}, \algopt with momentum (Algorithm~\ref{alg:choco_with_momentum}) and standard (all-reduce) mini-batch SGD with momentum and without compression \citep{Dekel2012:minibatch}.
Our implementations are open-source and available at \url{https://github.com/epfml/ChocoSGD}.
The momentum factor is set to $0.9$ without dampening.
For all algorithms we fine-tune the initial learning rate and gradually warm it up from a relative small value (0.1)~\citep{goyal2017accurate} for the first $5$ epochs.
The learning rate is decayed by $10$ twice, at $150$ and $225$ epochs, and stop training at 300 epochs.
For \algopt and DeepSqueeze the consensus learning rate $\gamma$ is also tuned. The detailed hyper-parameter tuning procedure refers to Appendix~\ref{sect:parameters}. %
Every compression scheme is applied to every layer of \texttt{ResNet20} separately. We evaluate the top-1 test accuracy on every node separately over the whole dataset and report the average performance over all nodes. %

\paragraph{Compression Schemes.}
We implement two \emph{unbiased} compression schemes: 
(i) $\operatorname{gsgd}_b$ quantization that randomly rounds the weights to $b$-bit representations~\citep{Alistarh2017:qsgd}, and
(ii) $\operatorname{random}_a$ sparsification, which preserves a randomly chosen $a$ fraction of the weights and sets the other ones to zero~\citep{Wangni2018:sparsification}.
\noindent Further two \emph{biased} compression schemes:
(iii) $\operatorname{top}_a$, which selects the $a$ fraction of weights with the largest magnitude and sets the other ones to zero~\citep{Alistarh2018:topk,Stich2018:sparsifiedSGD},
and 
(iv) $\operatorname{sign}$ compression, which compresses each weight to its sign scaled by the norm of the full vector~\citep{Bernstein2018:sign,KarimireddyRSJ2019feedback}.
We refer to Appendix~\ref{sect:compr} for exact definitions of the schemes.

DCD and ECD have been analyzed only for unbiased quantization schemes, thus the combination with the two biased schemes is not supported by theory. In converse, \algopt and DeepSqueeze has been studied only for biased schemes according to Definition~\ref{def:omega}. However, both unbiased compression schemes can be scaled down in order to meet the specification (cf.\ discussions in~\citep{Stich2018:sparsifiedSGD,Koloskova:2019choco}) and we adopt this for the experiments.

\paragraph{Results.}
The results are summarized in Tab.~\ref{tb:toy}. %
For unbiased compression schemes, ECD and DCD  only achieve good performance when the compression ratio is small, and sometimes even diverge when the compression ratio is high. This is consistent\footnote{\label{ft:tang} 
\citet{Tang2018:decentralized} only consider absolute bounds on the quantization error. Such bounds might be restrictive (i.e. allowing only for low compression) when the input vectors are unbounded. This might be the reason for the instabilities observed here and also  in~\cite[Fig.\ 4]{Tang2018:decentralized}, \cite[Figs.\ 5--6]{Koloskova:2019choco}.
}
 with the theoretical and experimental results in~\citep{Tang2018:decentralized}. 
We further observe that the performance of DCD with the biased $\operatorname{top}_a$ sparsification is much better than with the unbiased $\operatorname{random}_a$ counterpart, though this operator is not yet supported by theory.

\algopt can generalize reasonably well in all scenarios (at most 1.65\% accuracy drop) for fixed training budget. The $\operatorname{sign}$ compression achieves state-of-the-art accuracy and requires approximately $32\times$ less bits per weight than the full precision baseline. 

\newlength\mythickline
\newlength\mythinline
\setlength\mythickline{1pt}
\setlength\mythinline{0.5pt}

\begin{table}[bt]
	\caption{\small{Top-1 test accuracy for decentralized DCD, ECD, DeepSqueeze and \algopt with different compression schemes.
			Reported top-1 test accuracies are averaged over three runs with fine-tuned hyper-parameters (learning rate, weight decay, consensus stepsize).
			The fine-tuned all-reduce baseline reaches accuracy $92.64$, with $1.04$ MB gradient transmission per iteration.
			($\star$ indicates that 2 out of 3 runs diverged).
	}}
	\label{tb:toy}
	
	\resizebox{1.\textwidth}{!}{%
	\vbox{
	\hbox{
		\begin{tabular}{lcccccccccc}
			\toprule[\mythickline]
			\textbf{Algorithm}& \multirow{2}{*}{\parbox{13mm}{\textbf{Error-\\feedback}}} & \multicolumn{4}{c}{\textbf{Quantization (QSGD)}}   &  & \multicolumn{3}{c}{\textbf{Sparsification (random-\%)}} \\  \cmidrule[\mythinline]{3-6} \cmidrule[\mythinline]{8-10}
			&   & 16 bits          & 8 bits           & 4 bits           & 2 bits     & & 50\%              & 10\%             & 1\%             \\ \midrule[\mythinline]
			transmitted data/iteration\hspace{-8mm} &      & 0.52 MB          & 0.26 MB          & 0.13 MB          & 0.065 MB       & & 1.04 MB           & 0.21 MB          & 0.031 MB    \\  
			\textbf{DCD-PSGD} & 	 	\xmark	  & $92.51 \pm 0.05$ & $92.36 \pm 0.28$ & $23.56 \pm 2.97$ & diverges  & & $92.05 \pm 0.25$  & diverges          & diverges      \\ 
			\textbf{ECD-PSGD} &  	\xmark		  & $92.02 \pm 0.14$ & $59.11 \pm 1.57$ & diverges         & diverges  &        &  diverges           & diverges          & diverges     \\ 
			\textbf{DeepSqueeze} & 	 \cmark		  & $92.27 \pm 0.21$ & $91.83 \pm 0.35$ & $91.47 \pm 0.21$ & $90.96 \pm 0.19$  &       &   $91.46 \pm 0.09$  & $90.96 \pm 0.16$ & $88.55 \pm 0.11$ \\ 
			\textbf{CHOCO-SGD} & 		 \cmark	  & $92.34 \pm 0.19$ & $92.30 \pm 0.08$ & $91.92 \pm 0.27$ & $91.41 \pm 0.11$ & & $92.54 \pm 0.26$  & $91.87 \pm 0.21$ & $91.32 \pm 0.17$ \\ \bottomrule[\mythickline] \addlinespace
	  \end{tabular}
	  }
	  
	\hbox{ 
	  \begin{tabular}{lccccccc}
            \toprule[\mythickline]
			\textbf{Algorithm} & \multirow{2}{*}{\parbox{13mm}{\textbf{Error-\\feedback}}} & \multicolumn{3}{c}{\textbf{Sparsification (top-\%)}}  &   &  \multicolumn{1}{c}{\textbf{Sign+Norm}} \\ \cmidrule[\mythinline]{3-5}  \cmidrule[\mythinline]{7-7}
			& 		  		& 50\%             & 10\%             & 1\%                 & &      -                               \\ \midrule[\mythinline]
			transmitted data/iteration\hspace{-8mm} & 		& 1.04 MB          & 0.21 MB          & 0.031 MB    & & 0.032 MB                            \\ 
			\textbf{DCD-PSGD} & 	 	\xmark		& $92.40 \pm 0.11$ & $91.97 \pm 0.14$ & $89.79 \pm 0.40$       &  &  $92.40 \pm 0.14$ \\ 
			\textbf{ECD-PSGD} &   	\xmark				& $17.03 \phantom{\pm} \star \phantom{.11}$          & $16.78 \phantom{\pm} \star \phantom{.11}$          & $18.03 \phantom{\pm} \star \phantom{.11}$         & &         diverges                           \\ 
			\textbf{DeepSqueeze} & 		 	\cmark		& $91.55 \pm 0.28$ & $91.31 \pm 0.25$ & $90.47 \pm 0.17$ & & $91.38 \pm 0.19$\\ 
			\textbf{CHOCO-SGD} & 	 	\cmark	& $92.54 \pm 0.26$ & $92.29 \pm 0.05$ & $91.73 \pm 0.11$ & & $92.46 \pm 0.10$\\  \bottomrule[\mythickline]
		\end{tabular}%
		\hfill
		}}
	}
\end{table}

\section{Use case I: On-Device Peer-to-Peer Learning}
\vspace{-1mm}

We now shift our focus to challenging real-world scenarios which are intrinsically decentralized, i.e. each part of the training data remains local to each device, and thus centralized methods either fail or are inefficient to implement. 
Typical scenarios comprise e.g.\ sensor networks, or 
mobile devices or hospitals which jointly train a machine learning model. %
Common to these applications is that i) each device has only access to locally stored or acquired data, ii) communication bandwidth is limited (either physically, or artificially for e.g.\ metered connections), iii) the global network topology is %
typically unknown to a single device, and iv) the number of connected devices is typically large.
Additionally, this fully decentralized setting is also strongly motivated by privacy aspects, enabling to keep the training data private on each device at all times.

\vspace{-2mm}
\paragraph{Modeling.}
To simulate this scenario, we permanently split the training data between the nodes, i.e. %
the data is never shuffled between workers during training, and every node has distinct part of the dataset. To the best of our knowledge, no prior works studied this scenario for decentralized deep learning. %
For the centralized approach, gathering methods such as all-reduce are not efficiently implementable in this setting, hence we compare to the centralized baseline where all nodes route their updates to a central coordinator for aggregation.
For the comparison we consider \algopt with $\operatorname{sign}$ compression 
(this combination achieved the compromise between accuracy and compression level in Tab.~\ref{tb:toy})), 
decentralized SGD without compression~\citep{Lian2017:decentralizedSGD}, and centralized SGD without compression.

\paragraph{Scaling to Large Number of Nodes.}
To study the scaling properties of \algopt, we train on $4, 16, 36$ and $64$ number of nodes. 
We compare decentralized algorithms on two different topologies:\ \emph{ring} as the worst possible topology, %
and on the \emph{torus} with much larger spectral gap.
The corresponding parameters are listed in Table \ref{tab:topologies}. 
 \begin{wraptable}{r}{8cm}
 	\vspace{-1em} %
 	\caption{\small{Summary of communication topologies.}}\label{tab:topologies}
 	\centering
 	\resizebox{1\linewidth}{!}{%
 		\begin{tabular}{llllllll}
 			\toprule[\mythickline]
 			Topology                     & 
 			&  & \phantom{ab}
 			& \multicolumn{4}{c}{spectral gap $\sgap$}\\ \cmidrule{3-3} \cmidrule{5-8}
 			& & max. node degree &  & $n = 4$ & $n = 16$ & $n = 36$ & $n = 64$\\
 			\midrule
 			ring               & & $2$ & & 0.67 & 0.05 & 0.01 & 0.003\\
 			torus       & & $4$ & & 0.67  & 0.4 & 0.2 & 0.12 \\
 			fully-connected & & $d$ & & 1 & 1& 1& 1 \\
 			\bottomrule[\mythickline]
 		\end{tabular}
 		\vspace{-2em}
 	}%
 	\vspace{-0.5em}
 \end{wraptable}
We train \texttt{ResNet8} \citep{He2016:Resnet} ($78$K parameters), on \texttt{Cifar10} dataset (50K/10K training/test samples)~\citep{Krizhevsky2012:cifar10}. For simplicity, we keep the learning rate constant and separately tune it for all methods. We further tune the consensus learning rate for \algopt.

\begin{figure*}[t]
	\vspace{-1em}
	\centering
	\subfigure%
	{
		\includegraphics[width=0.40\textwidth,]{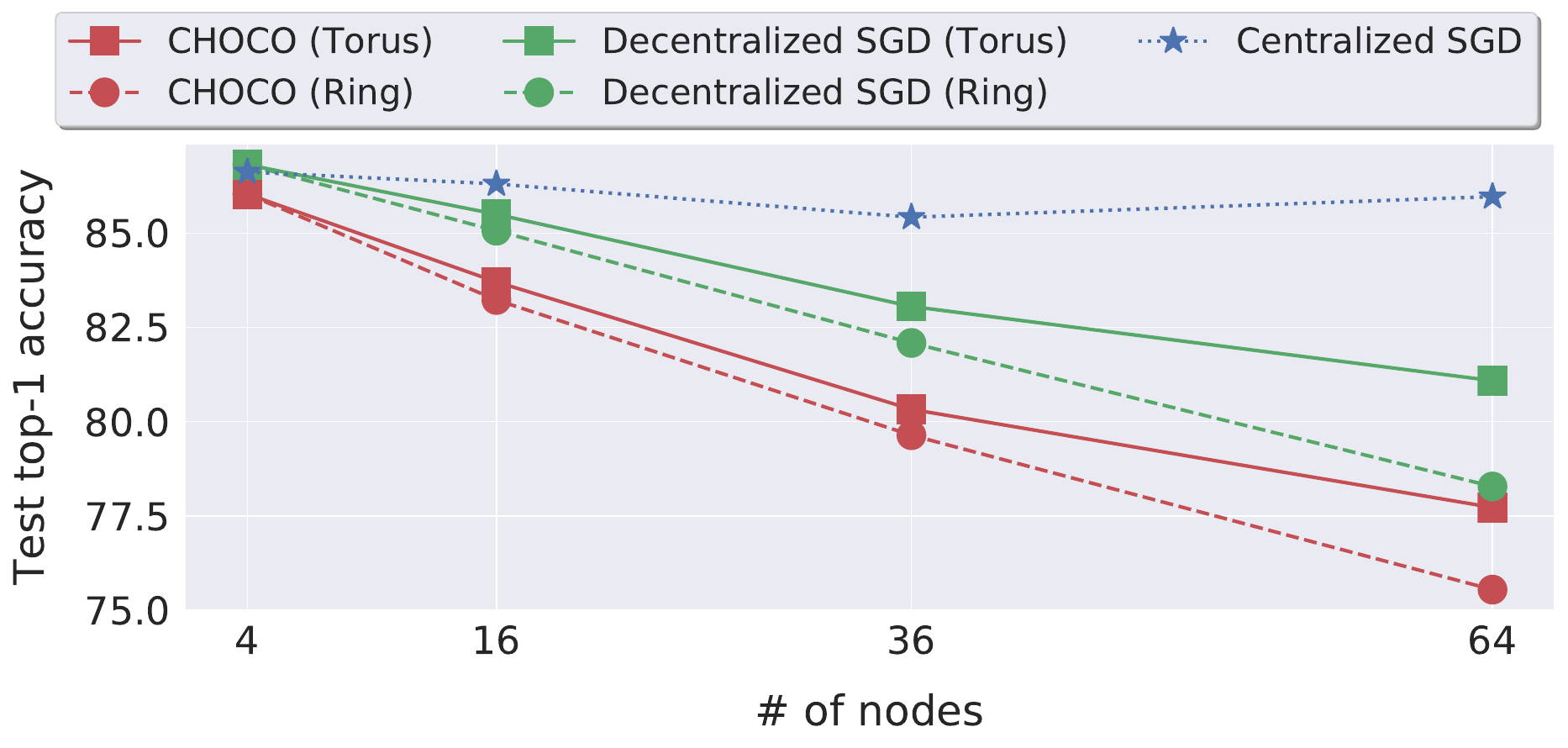}
	}\qquad
	\subfigure%
	{
		\includegraphics[width=0.40\textwidth,]{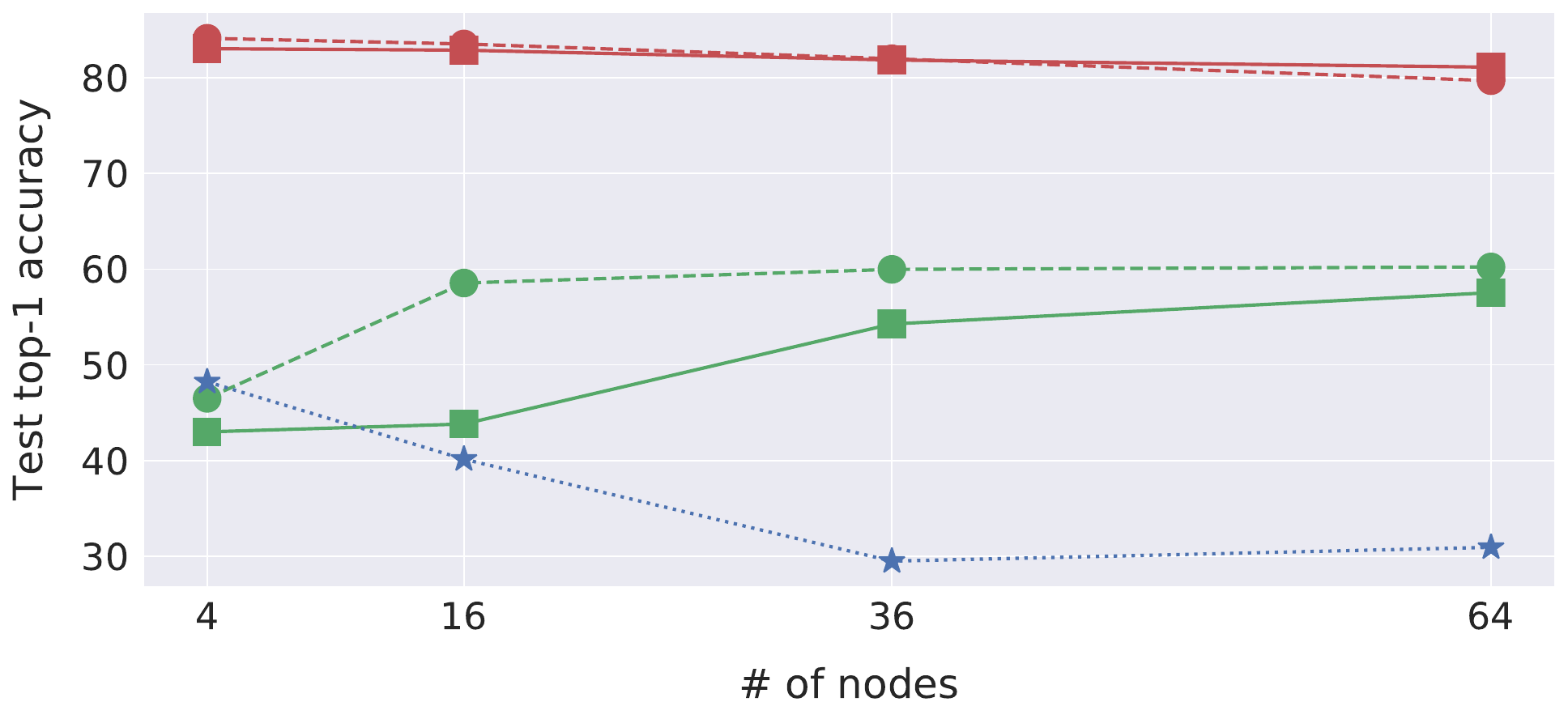}
	}
\hfill\\
\centering \footnotesize{\qquad\qquad\qquad Fix budget of 300 epochs \qquad\qquad\qquad Fixed budget of communication size (1000 MB)\quad}\\
	\vspace{-0.5em}  
	\caption{\normalsize{
			Scaling of \algopt with $\operatorname{sign}$ compression to large number of devices on \texttt{Cifar10} dataset.
			\emph{Left:} best testing accuracy of the algorithms reached after 300 epochs. %
			\emph{Right:} best testing accuracy reached after communicating 1000 MB.
	}}
	\label{fig:social}
	\vspace{-0.5em}
\end{figure*} %

The results are summarized in Fig.~\ref{fig:social} (and Fig.~\ref{fig:social_app}, Tabs.~\ref{tab:epoch_for_same_bits_budget}--\ref{tab:bits_for_same_epoch_budget} in Appendix~\ref{sec:additional_plots}). 
First we compare the testing accuracy reached after 300 epochs (Fig.~\ref{fig:social}, \emph{left}). CentralizedSGD has a good performance for all the considered number of nodes. \algopt slows down due to the influence of the graph topology (\texttt{Decentralized} curve), which is consistent with the spectral gaps order (see Tab.~\ref{tab:topologies}), and also influenced by the communication compression (\texttt{CHOCO} curve), which slows down training uniformly for both topologies. We observed that the train performance is similar to the test on Fig.~\ref{fig:social}, therefore the performance degradation is explained by the slower convergence (Theorem~\ref{th:non-convex-sigma}) and is not a generalization issue. Increasing the number of epochs improves the performance of the decentralized schemes. 
However, even using 10 times more epochs, we were not able to perfectly close the gap between centralized and decentralized algorithms for both train and test performance.

In the real decentralized scenario, the interest is not to minimize the epochs number, but the amount of communication to reduce the cost of the user's mobile data. We therefore fix the number of transmitted bits to 1000 MB and compare the best testing accuracy reached (Fig.~\ref{fig:social}, \emph{right}). \algopt performs the best while having slight degradation due to increasing number of nodes. It is beneficial to use torus topology when the number of nodes is large because it has good mixing properties, for small networks there is not much difference between these two topologies---the benefit of a large spectral gap is canceled by the increased communication due larger node degree for torus topology. Both Decentralized and Centralized SGD requires significantly larger number of bits to reach reasonable accuracy.

\paragraph{Experiments on a Real Social Network Graph.}
We simulate training models on user devices (e.g.\ mobile phones), connected by a real social network. We chosen Davis Southern women social network \citep{Davis:social_network} with 32 nodes. We train \texttt{ResNet20} ($0.27$ million parameters) model on the \texttt{Cifar10} dataset (50K/10K training/test samples)~\citep{Krizhevsky2012:cifar10} for image classification and a three-layer \texttt{LSTM} architecture \citep{Hochreiter1997:LSTM} ($28.95$ million parameters) for a language modeling task on WikiText-2 (600 training and 60 validation articles with a total of $2'088'628$ and $217'646$ tokens respectively)~\citep{merity2016pointer}. 
The depicted curves of the training loss are the averaged local loss over all workers (local model with fixed local data);
the test performance uses the mean of the evaluations for local models on whole test dataset.
For more detailed experimental setup we refer to Appendix~\ref{sect:parameters}.

The results are summarized in Figs.~\ref{fig:social_resnet}--\ref{fig:social_lstm} and in Tab.~\ref{tab:social}. 
For the image classification task,
when comparing the training accuracy reached after the same number of epochs, we observe that the decentralized algorithm performs best, follows by the centralized and lastly the quantized decentralized. However, the test accuracy is highest for the centralized scheme.
When comparing the test accuracy reached for the same transmitted data\footnote{
The figure reports the transmitted data on the busiest node, i.e on the max-degree node (degree 14) node for decentralized schemes, and degree 32 for the centralized one.
}, \algopt significantly outperforms the exact decentralized scheme, with the centralized performing worst. We note a slight accuracy drop, i.e. after the same number of epochs (but much less transmitted data), \algopt does not reach the same level of test accuracy than the baselines. 

For the language modeling task, both decentralized schemes suffer a drop in the training loss when the evaluation reaching the epoch budget;
while our \algopt outperforms the centralized SGD in test perplexity.
When considering perplexity for a fixed data volume (middle and right subfigure of Fig.~\ref{fig:social_lstm}), \algopt performs best, followed by the exact decentralized and centralized algorithms.

\begin{figure*}[!h]
    \vspace{-1em}
		\centering
		\subfigure%
		{
			\includegraphics[width=0.31\textwidth,]{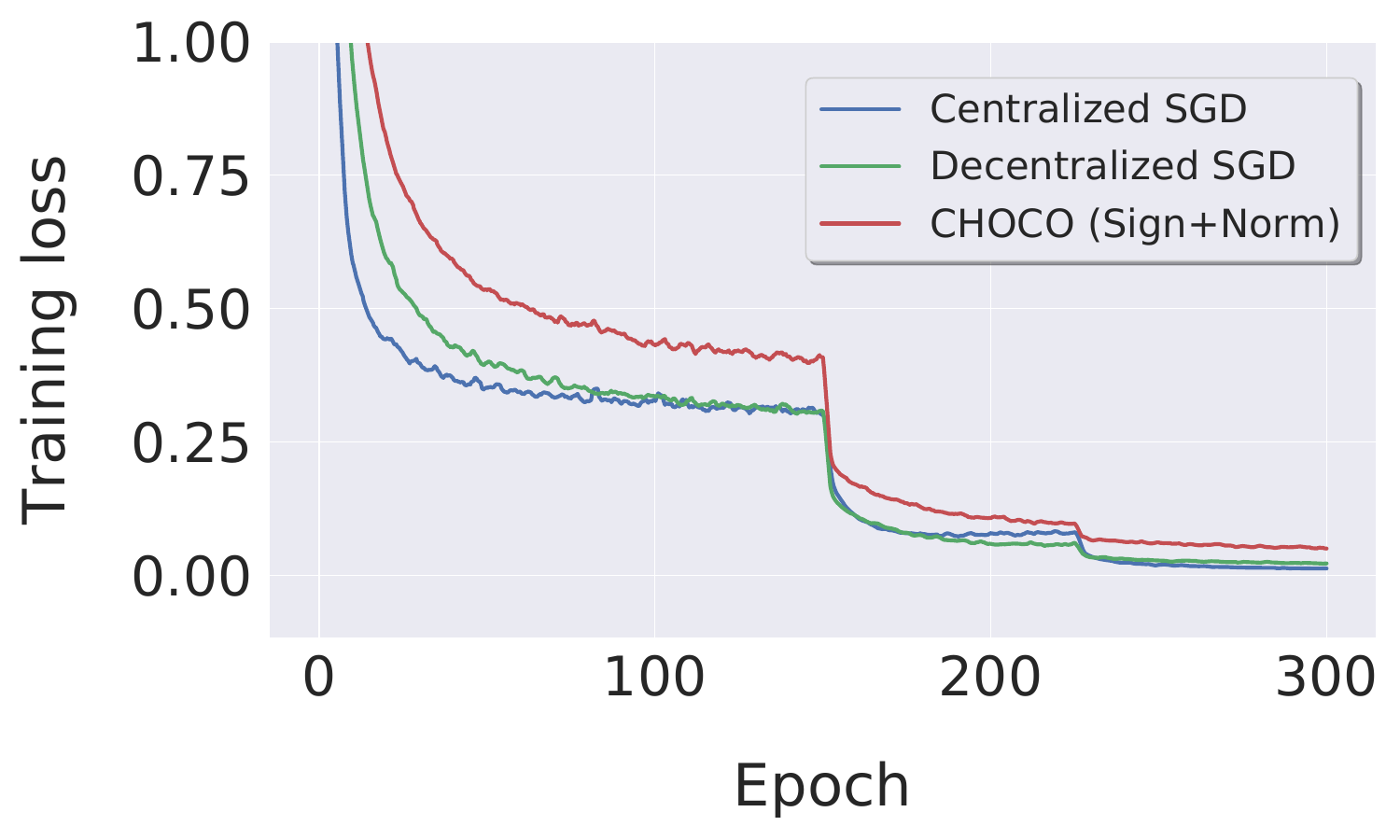}
			\label{fig:resnet20_cifar10_k32_bs32_social_topology_tr_loss_vs_epoch}
		}
		\hfill
    \subfigure%
    {
        \includegraphics[width=0.31\textwidth,]{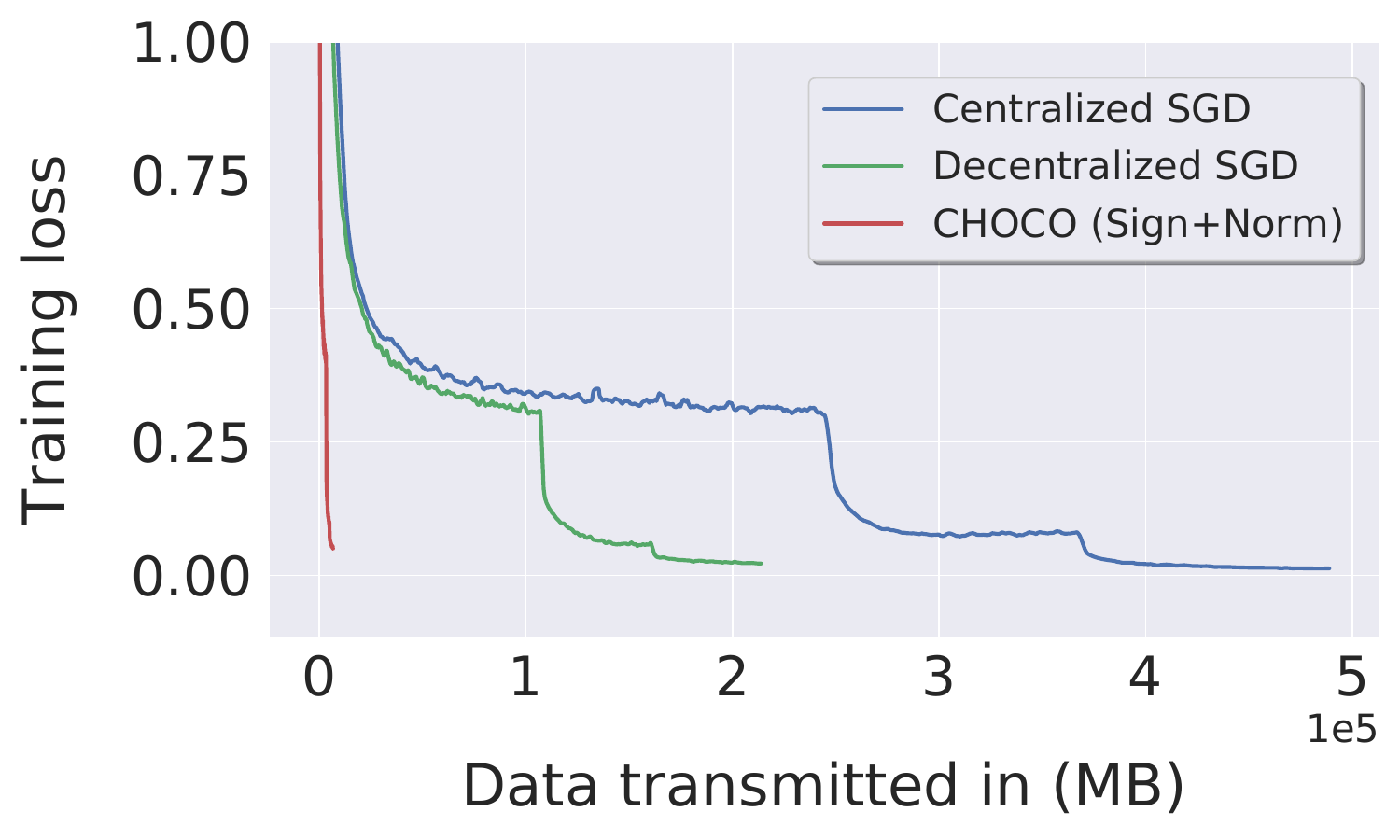}
        \label{fig:resnet20_cifar10_k32_bs32_social_topology_tr_loss_vs_bits}
		}
		\hfill
		\subfigure%
		{
			\includegraphics[width=0.31\textwidth,]{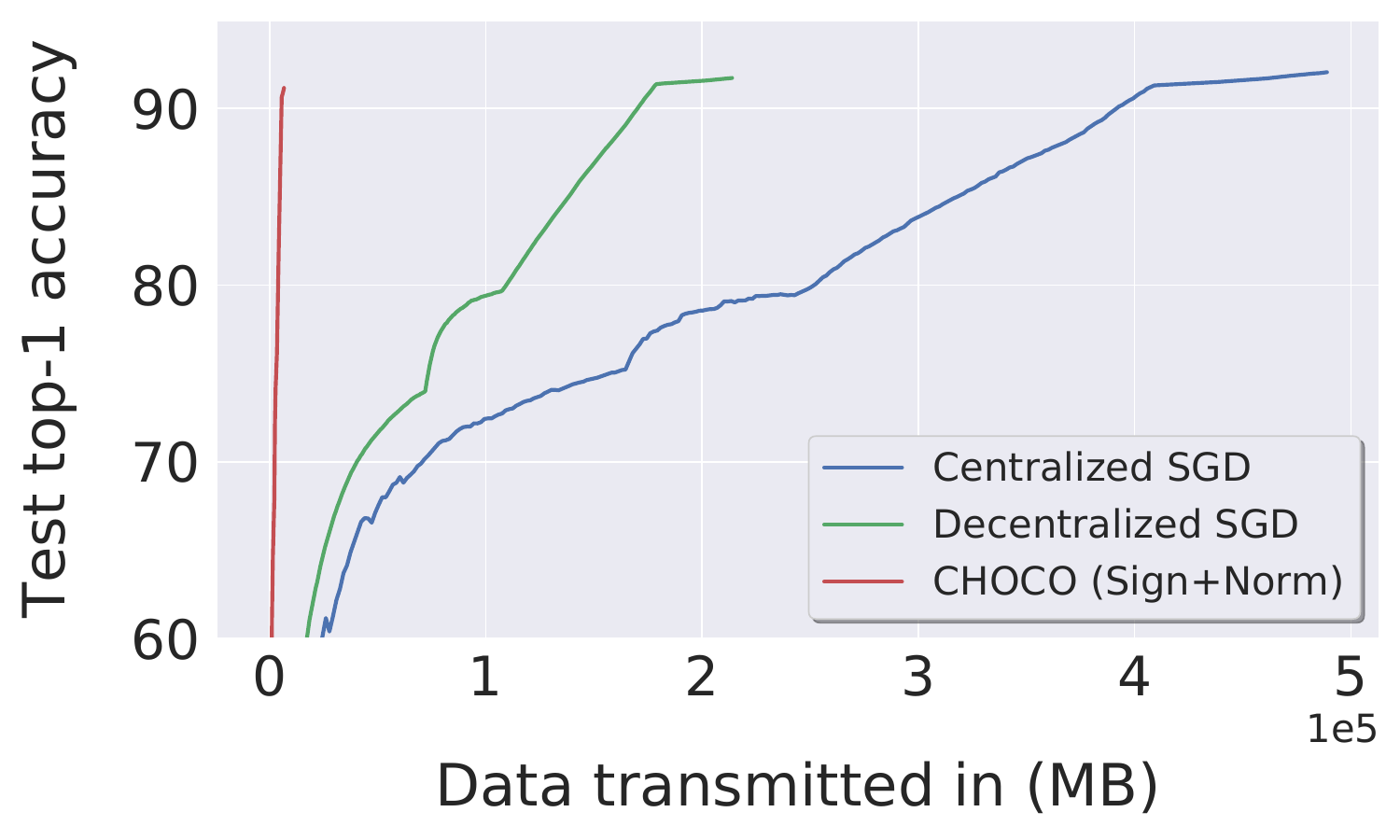}
			\label{fig:resnet20_cifar10_k32_bs32_social_topology_te_top1_vs_bits}
		}
    \vspace{-0.5em}  
    \caption{\normalsize{
			Image classification: ResNet-20 on CIFAR-10 on social network topology.
    }}
    \label{fig:social_resnet}
    \vspace{-0.5em}
\end{figure*}

\begin{figure*}[!h]
	\vspace{-0.5em}
	\centering
	\subfigure%
	{
		\includegraphics[width=0.31\textwidth,]{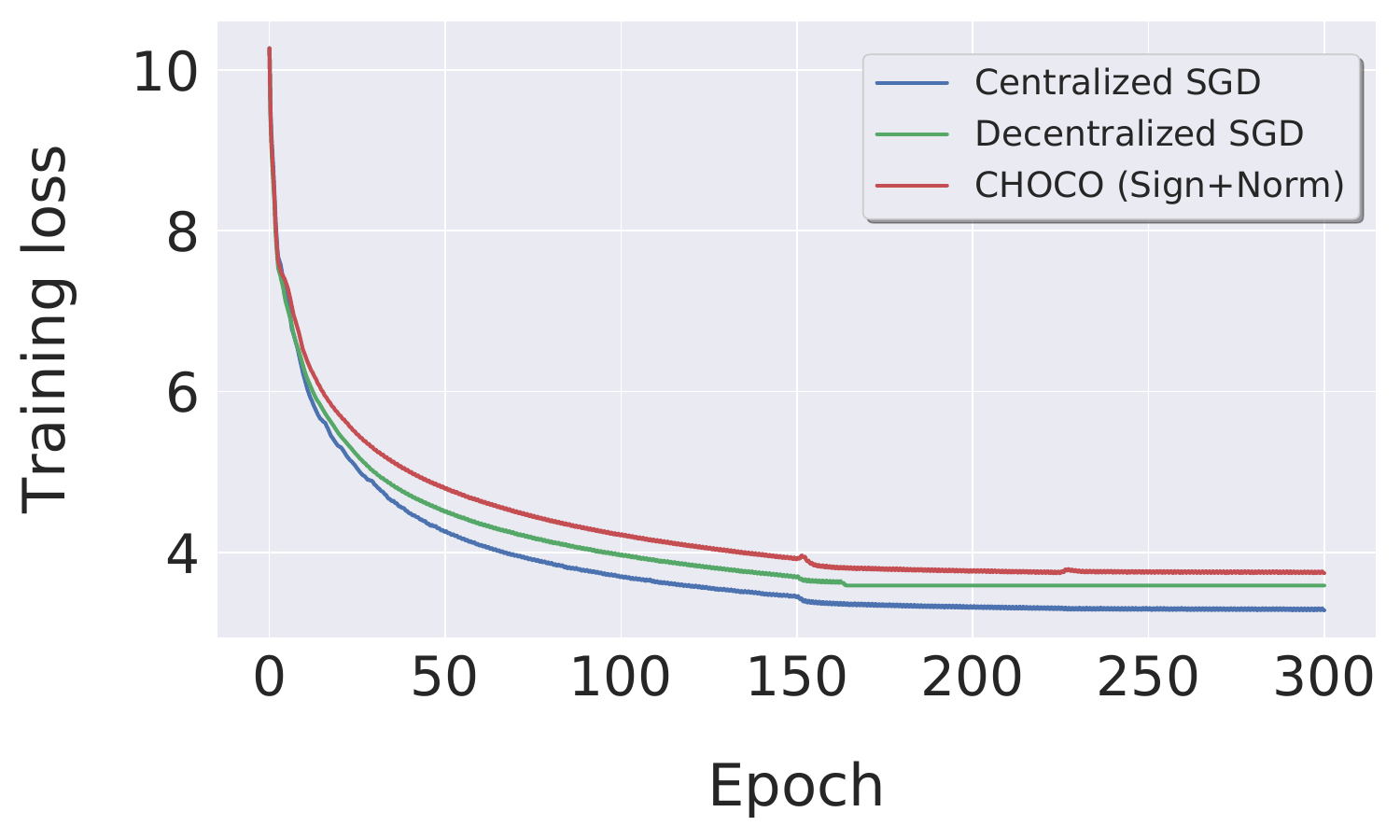}
		\label{fig:lstm_wikitext2_k32_bs32_social_topology_tr_loss_vs_epoch}
	}
	\hfill
	\subfigure%
	{
		\includegraphics[width=0.31\textwidth,]{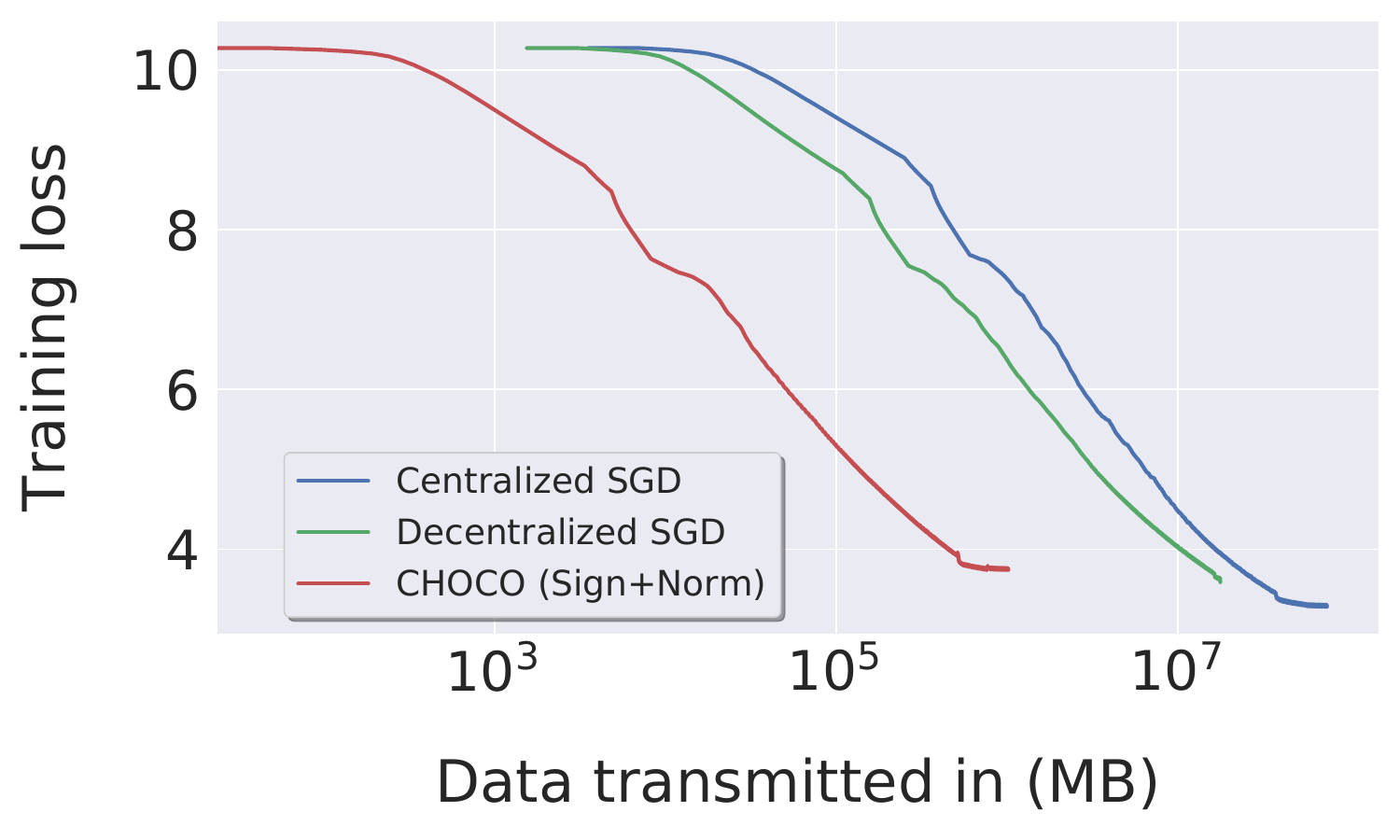}
		\label{fig:lstm_wikitext2_k32_bs32_social_topology_tr_loss_vs_bits}
	}
	\hfill
	\subfigure%
	{
		\includegraphics[width=0.31\textwidth,]{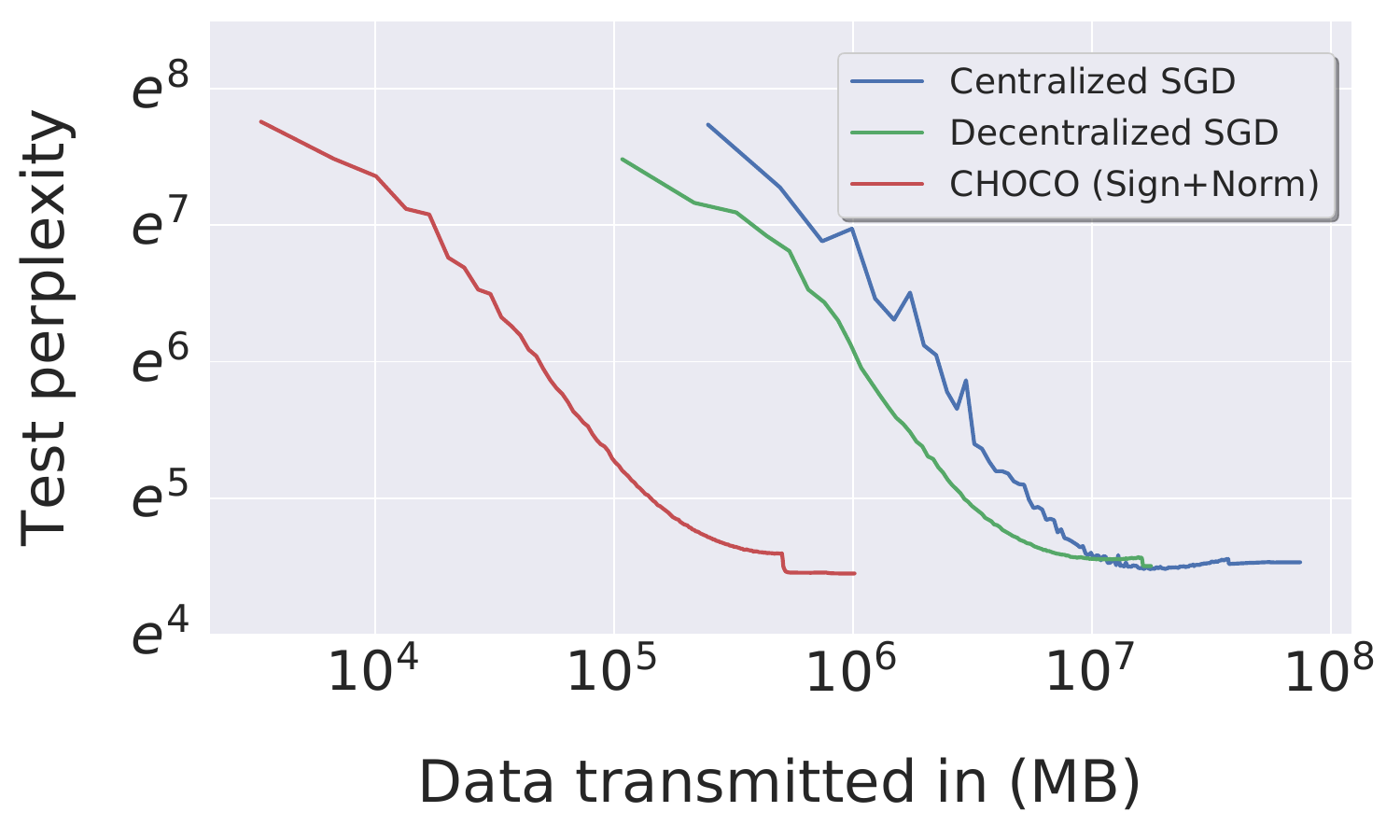}
		\label{fig:lstm_wikitext2_k32_bs32_social_topology_te_ppl_vs_bits}
	}
	\vspace{-0.5em}
	\caption{\normalsize{
			Language modeling: LSTM on WikiText-2 on social network topology.
	}}
	\label{fig:social_lstm}
\end{figure*}

\begin{table}[!h]
\caption{\small{Summary of performance when training with the same epoch budget (as centralized SGD).}}\label{tab:social}
\resizebox{\linewidth}{!}{%
\begin{tabular}{lllllllll}
 \toprule[\mythickline]
 Algorithm                     & \phantom{ab} 
 & & \phantom{ab} 
 & \multicolumn{2}{c}{ResNet-20 (Fig.~\ref{fig:social_resnet})}  & \phantom{ab}
 & \multicolumn{2}{c}{LSTM (Fig.~\ref{fig:social_lstm})} \\ \cmidrule{3-3} \cmidrule{5-6} \cmidrule{8-9}
 & & max. connections/node  & & data/gradient & top-1 test acc. & & data/gradient & test perplexity \\ \midrule
 Centralized SGD               & & 32 & & 1.04 MB  & 93.00 & & 110.43 MB & 89.39 \\
 Exact Decentralized SGD       & & 14 & & 1.04 MB  & 92.12 & & 110.43 MB & 91.38  \\
 \algopt (Sign + Norm) & & 14 & & 0.032 MB & 91.80 & & 3.45 MB   & 86.58  \\
 \bottomrule[\mythickline]
\end{tabular}
}%
\end{table}

\begin{figure*}[!h]
	\centering
	{
		\includegraphics[width=0.45\textwidth,]{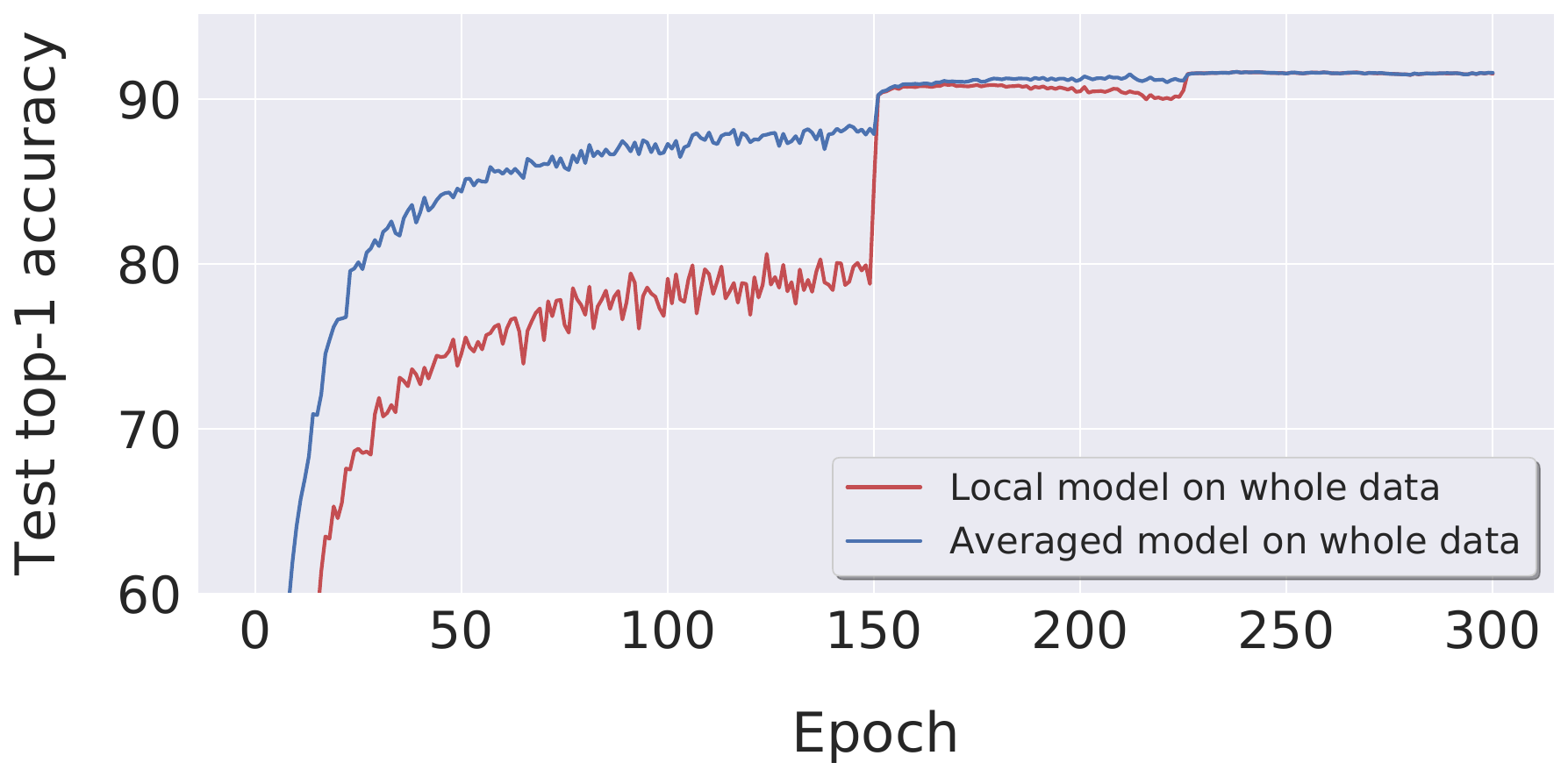}
	}
	\hfill
	{
		\includegraphics[width=0.45\textwidth,]{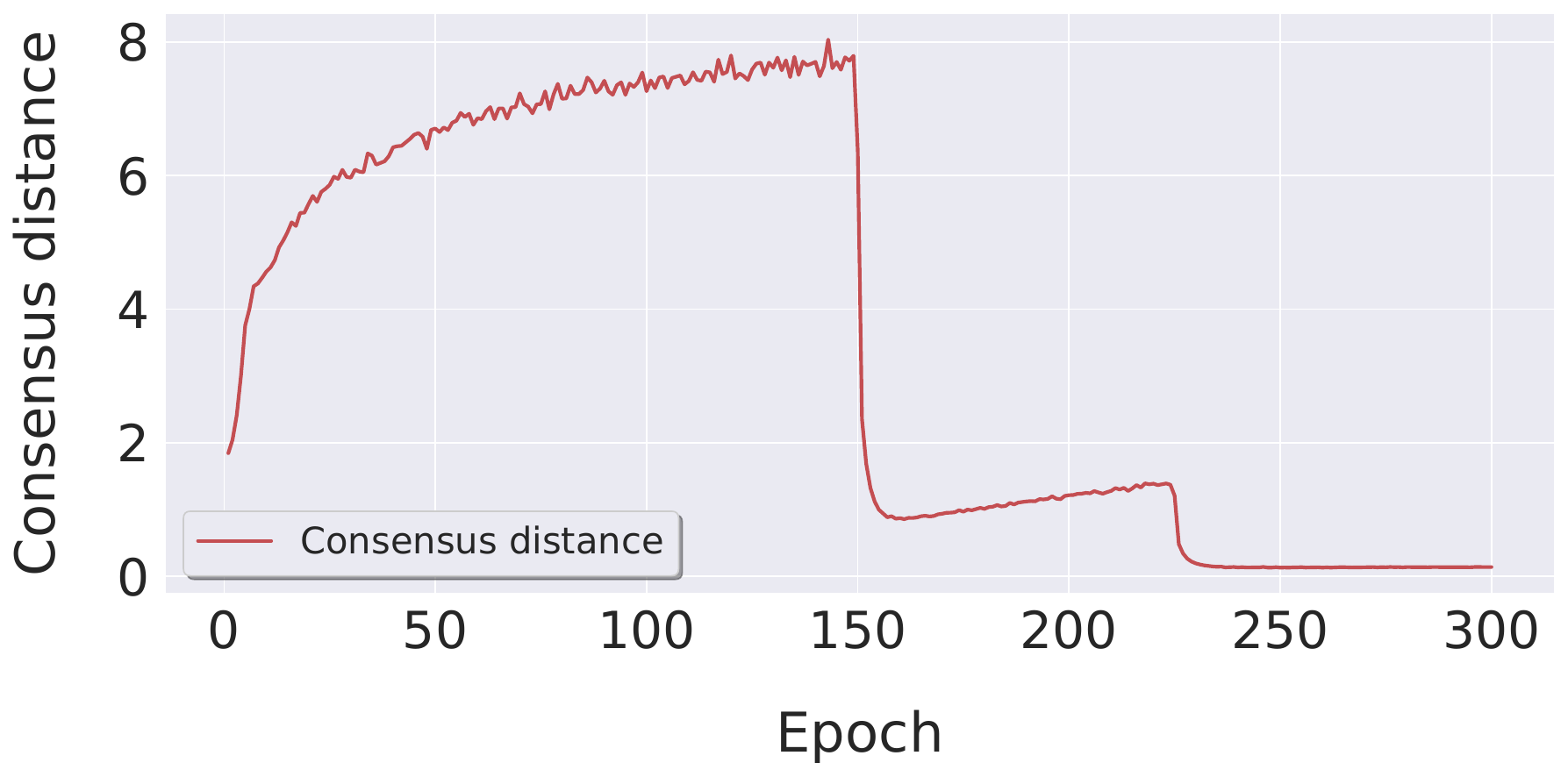}
	}
	\caption{\small{
			Parameter deviations for Resnet20 trained on Cifar10 (using \algopt) on social network topology (32 workers). 
			(Left) performance of the averaged model compared to the average of performances of local models. 
			(Right) parameters divergence: averaged $L_2$ consensus distance between local models $\xx_i$ and the averaged model $\overline{\xx} = \frac{1}{n}\sum_{i = 1}^n \xx_i$,
			i.e., $\frac{1}{n} \sum_{i=1}^n \Vert \xx_i - \bar{\xx} \Vert_2^2$.
	}}
	\label{fig:local_models_vs_averaged}
\end{figure*}

On Figure~\ref{fig:local_models_vs_averaged} we additionally depict the test accuracy of the averaged model $\overline{\xx}^{(t)} = \frac{1}{n}\sum_{i = 1}^n \xx_i^{(t)}$ (\emph{left}) and averaged distance of the local models from the averaged model (right),
for \algopt on image classification task.
Towards the end of the optimization the local models reach consensus (Figure~\ref{fig:local_models_vs_averaged}, \emph{right}), and their individual test performances are the same as performance of averaged model. Interestingly, before decreasing the stepsize at the epoch 225, the local models are in general diverging from the averaged model, while decreasing only when the stepsize decreases. A similar behavior was also reported in \citep{Assran:2018sdggradpush}.

\section{Use case II: Efficient Large-Scale Training in a Datacenter}
Decentralized optimization methods offer a way to address scaling issues even for well connected devices, such as e.g.\ in datacenter with fast InfiniBand (100Gbps) or Ethernet (10Gbps) connections. 
\cite{Lian2017:decentralizedSGD} describe scenarios when decentralized schemes can outperform centralized ones, and recently, \cite{Assran:2018sdggradpush} presented impressive speedups for training on 256 GPUs, for the setting when all nodes can access all training data. The main differences of their algorithm to \algopt are the asynchronous gossip updates, time-varying communication topology and most importantly exact communication, making their setup not directly comparable to ours. 
We note that these properties of asynchronous communication and changing topology for faster mixing are orthogonal to our contribution, and offer promise to be combined.

\paragraph{Setup.}

We train \texttt{ImageNet-1k}  ($1.28$M/$50$K training/validation)~\citep{imagenet_cvpr09} with \texttt{Resnet-50}~\citep{He2016:Resnet}.
We perform our experiments on $8$ machines (n1-standard-32 from Google Cloud with Intel Ivy Bridge CPU platform), 
where each of machines has $4$ Tesla P100 GPUs and each machine interconnected via 10Gbps Ethernet.
Within one machine communication is fast and we rely on the local data parallelism to aggregate the gradients for the later gradients communication (over the machines).
Between different machines we consider centralized (fully connected topology) and decentralized (ring topology) communication,
with and without compressed communication ($\operatorname{sign}$ compression).
Several methods categorized by communication schemes are evaluated:
(i) centralized SGD (full-precision communication),
(ii) error-feedback centralized SGD with compressed communications \cite{KarimireddyRSJ2019feedback} through $\operatorname{sign}$ compression,
(iii) decentralized SGD~\citep{Lian2017:decentralizedSGD} with parallelized forward pass and gradients communication (full-precision communication),
and (iv) \algopt with $\operatorname{sign}$ compressed communications.
The mini-batch size on each GPU is $128$, and we follow the general SGD training scheme in~\citep{goyal2017accurate} and directly use all their hyperparameters for all evaluated methods.
Due to the limitation of the computational resource, we did not heavily tune the consensus stepsize for \algopt\footnote{
	We estimate the consensus stepsize by running \algopt with different values for the first 3 epochs.
}.

\paragraph{Results.}
We depict the training loss and top-1 test accuracy in terms of epochs and time in Fig.~\ref{fig:datacenter}.
\algopt benefits from its decentralized and parallel structure and takes less time than all-reduce to perform the same number of epochs, 
while having only a slight $1.5\%$ accuracy loss\footnote{
	Centralized SGD with full precision gradients achieved test accuracy of $76.37\%$, v.s.\ $76.03\%$ for centralized SGD (with $\operatorname{sign}$ compression), 
	v.s.\ $74.92\%$ for plain decentralized SGD, and vs.\ $75.15\%$ for \algopt (with $\operatorname{sign}$ compression).
}. 
In terms of time per epoch, our speedup does not match that of~\citep{Assran:2018sdggradpush}, 
as the used hardware and the communication pattern\footnote{
	We consider undirected communication, contrary to the directed 1-peer communication 
	(every node sends and receives one message at every iteration) in~\citet{Assran:2018sdggradpush}.
} are very different.
Their scheme is orthogonal to our approach and could be integrated for better training efficiency.
Nevertheless, we still demonstrate a time-wise 20\% gain over the common all-reduce baseline, on our used commodity hardware cluster. 

\begin{figure*}[t]
	\centering
	\subfigure%
	{
		\includegraphics[width=0.31\textwidth,]{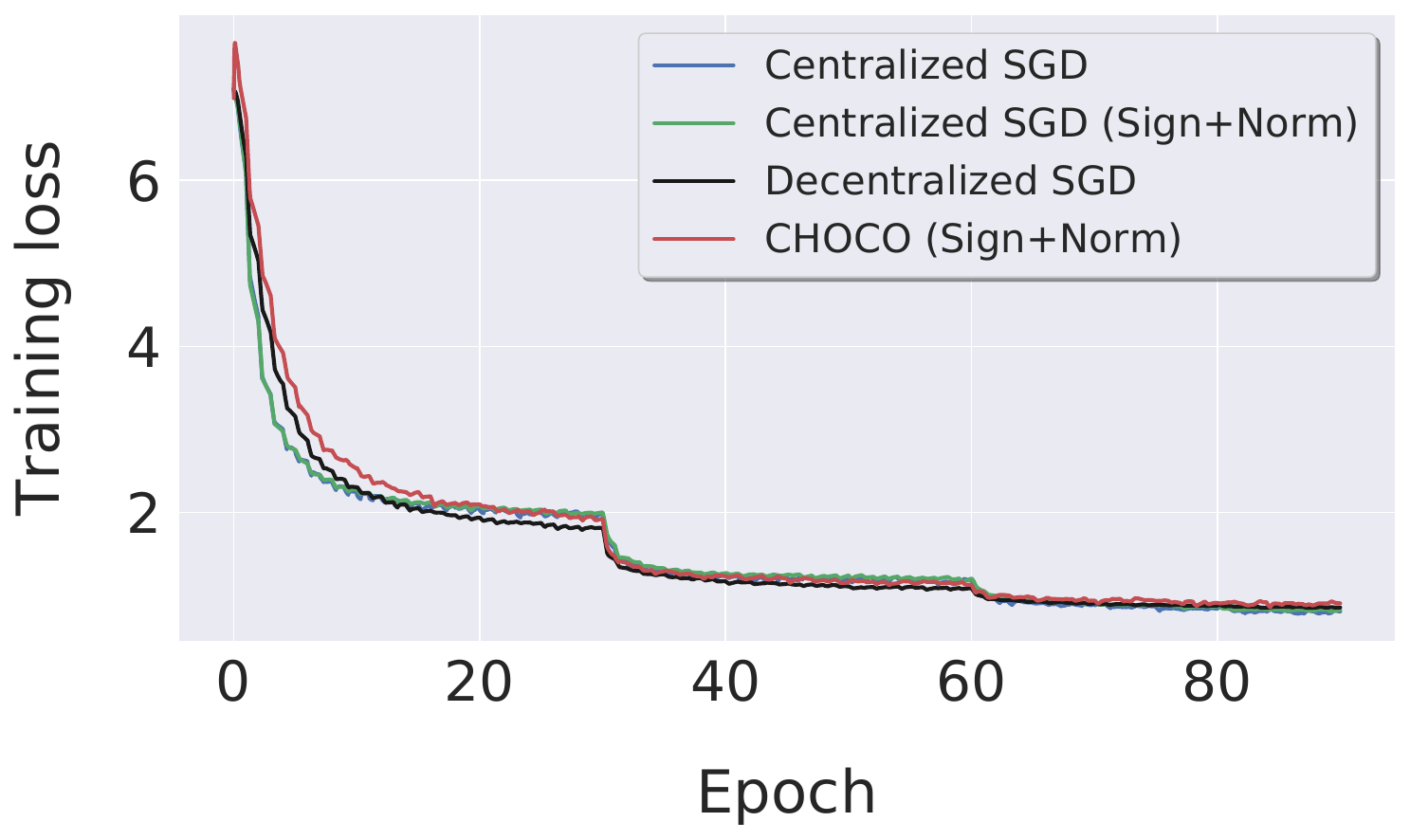}
		\label{fig:resnet50_imagenet_k32_tr_loss_vs_epoch}
	}
	\hfill
	\subfigure%
	{
		\includegraphics[width=0.31\textwidth,]{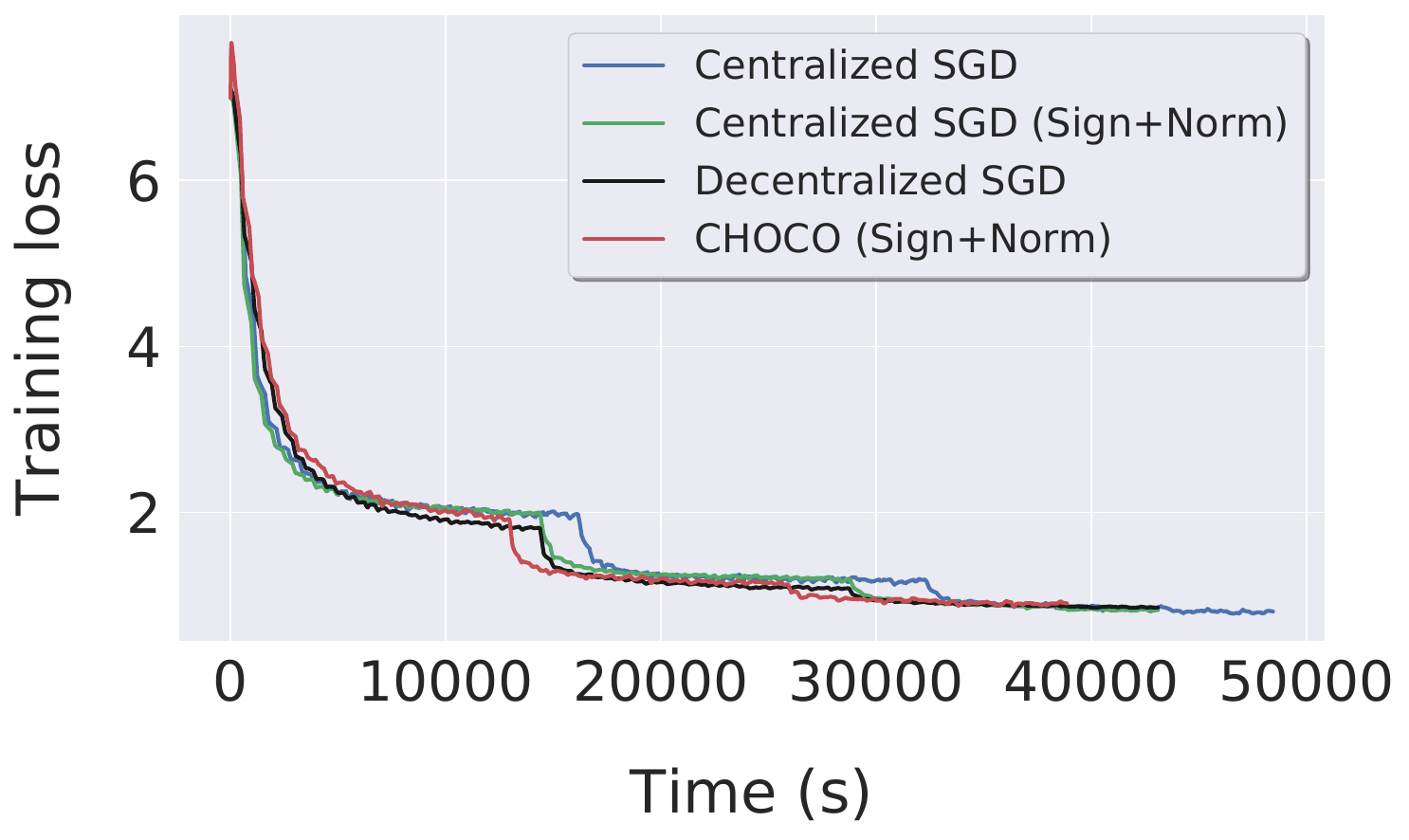}
		\label{fig:resnet50_imagenet_k32_tr_loss_vs_time}
	}
	\hfill
	\subfigure%
	{
		\includegraphics[width=0.31\textwidth,]{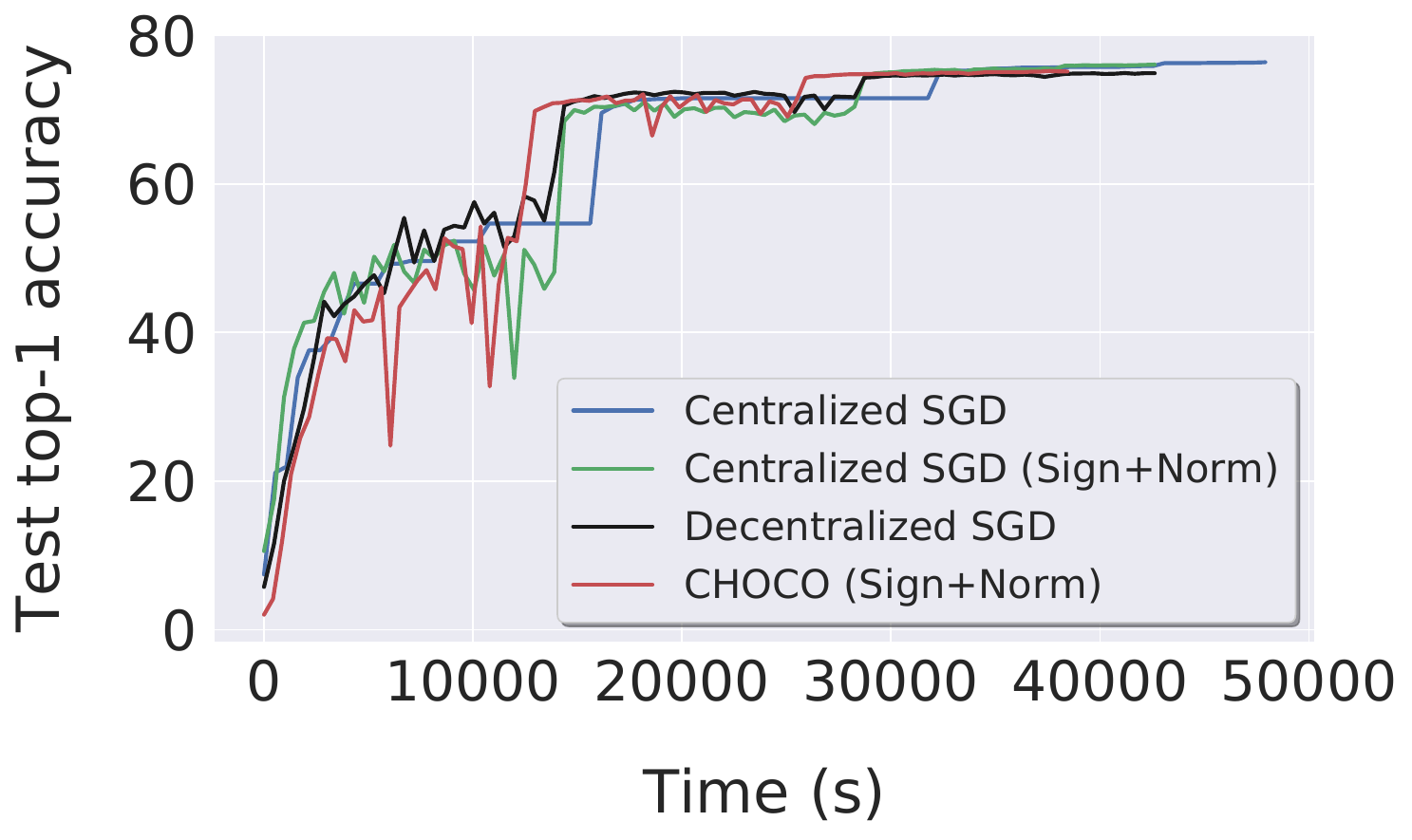}
		\label{fig:resnet50_imagenet_k32_te_top1_vs_time}
	}
	\vspace{-1em}
	\caption{\normalsize{
			Large-scale training: \texttt{Resnet-50} on \texttt{ImageNet-1k} in the datacenter setting.
			The topology has $8$ nodes (each accesses $4$ GPUs).
			We use $\operatorname{sign}$ as the compression scheme, for \algopt and Centralized SGD.
			For centralized SGD baseline without compression, we use all-reduce to aggregate the gradients;
			we use all-gather for centralized SGD with $\operatorname{sign}$ gradients quantization.
			The benefits of \algopt can be further pronounced when scaling to more nodes. 
	}}
	\label{fig:datacenter}
\end{figure*}

\section{Conclusion}
We propose the use of \algopt (and its momentum version) for enabling decentralized deep learning training %
in bandwidth-constrained environments. We provide theoretical convergence guarantees for the non-convex setting and show that the algorithm enjoys linear speedup in the number of nodes. 
We empirically study the performance of the algorithm in a variety of settings on the image classification (ImageNet-1k, Cifar10) and on the language modeling task (WikiText-2).  
Whilst previous work successfully demonstrated that decentralized methods can be a competitive alternative to centralized training schemes when no communication constraints are present~\citep{Lian2017:decentralizedSGD,Assran:2018sdggradpush}, our main contribution is to enable training in strongly communication-restricted environments, and while respecting the challenging constraint of locality of the training data.
We theoretically and practically demonstrate the performance of  decentralized schemes for arbitrary high communication compression, and under data-locality, and thus significantly expand the reach of potential applications of fully decentralized deep learning.

\section*{Acknowledgements}
We acknowledge funding from SNSF grant 200021\_175796, as well as a Google Focused Research Award.

\small
\bibliographystyle{iclr2020_modified}
\bibliography{papers_choco}
\normalsize

%% file: supp.tex
\section{Convergence of \algopt}
\label{sec:theorymain}
In this section we present the proof of Theorem~\ref{th:non-convex-sigma}. For this, we will first derive a slightly more general statement: in Theorem~\ref{th:non-convex-sigma-general} we analyze \algopt for arbitrary stepsizes $\eta$, and then derive Theorem~\ref{th:non-convex-sigma} as a special case. 

The structure of the proof follows~\citet{Koloskova:2019choco}. That is, we first show that Algorithm~\ref{alg:choco} is a special case of a more general class of algorithms (given in Algorithm~\ref{alg:blackbox_sgd_matrix}): Observe that Algorithm~\ref{alg:choco} consists of two main components: \raisebox{.5pt}{\textcircled{\raisebox{-.9pt} {2}}} the stochastic gradient update, performed locally on each node, and \raisebox{.5pt}{\textcircled{\raisebox{-.9pt} {1}}} the (quantized) averaging among the nodes. We can show convergence of all algorithms of this type---i.e.\ stochastic gradient updates \raisebox{.5pt}{\textcircled{\raisebox{-.9pt} {2}}} followed by an arbitrary averaging step \raisebox{.5pt}{\textcircled{\raisebox{-.9pt} {1}}}---as long as the averaging scheme exhibits linear convergence. For the specific averaging used in \algopt, linear convergence has been shown in~\citep{Koloskova:2019choco} and we will use their estimate of the convergence rate of the averaging scheme.

\subsection{A General Framework for Decentralized SGD with Arbitrary Averaging}
For convenience, we use the following matrix notation in this subsection.
\begin{align*}
\begin{split}
X^{(t)} := \left[ \xx_1^{(t)},\dots, \xx_n^{(t)}\right] \in \R^{d\times n}, \qquad \overline{X}^{(t)} := \left[ \overline{\xx}^{(t)},\dots, \overline{\xx}^{(t)}\right]  \in \R^{d\times n}, \\  \partial F(X^{(t)}, \xi^{(t)}) := \left[\nabla F_1(\xx_{1}^{(t)}, \xi_1^{(t)}), \dots,  \nabla F_n(\xx_{n}^{(t)}, \xi_n^{(t)})\right]  \in \R^{d\times n}.
\end{split}
\end{align*}

Decentralized SGD with arbitrary averaging is given in Algorithm~\ref{alg:arbitrary_avg}. %
\begin{figure*}[!h]
\algsetup{
  linenodelimiter = {\makeatletter\tikzmark{\arabic{ALC@line}}\makeatother:}
}%
\begin{algorithm}[H]%
	\caption{\textsc{decentralized SGD with arbitrary averaging scheme}}\label{alg:blackbox_sgd_matrix}
	\text{\textbf{input:} $X^{(0)} = \left[ \xx^{(0)},\dots, \xx^{(0)}\right]$, stepsize $\eta$, averaging function $h : \R^{d\times n}\times \R^{d\times n} \to \R^{d\times n} \times \R^{d\times n}$, }\\ %
	\text{\qquad \quad  initialize $Y^{(0)} = 0$ } %
	\par\vspace{1mm}%
	\hspace*{\SpaceReservedForComments}{}%
	\begin{minipage}{\dimexpr\linewidth-\SpaceReservedForComments\relax}%
		\fontsize{10}{14}\selectfont %
	\begin{algorithmic}[1]%
		\FOR[{{\it in parallel for all workers $i \in [n]$}}]{$t$\textbf{ in} $0\dots T-1$}
		\STATE $X^{(t + \frac{1}{2})} = X^{(t)} - \eta\partial F_i(X^{(t)}, \xi^{(t)})$ \hfill $\triangleleft$ stochastic gradient updates
		\STATE $(X^{(t + 1)}, Y^{(t + 1)}) = h(X^{(t + \frac{1}{2})}, Y^{(t)})$ \hfill  $\triangleleft$ blackbox averaging/gossip 
		\ENDFOR
	\end{algorithmic}\label{alg:arbitrary_avg}%
	\end{minipage}%
\AddNoteNoB[gray]{2}{2}{\raisebox{.5pt}{\textcircled{\raisebox{-.9pt} {2}}\hspace*{2mm}$\{$}}%
\AddNoteNoB[gray]{3}{3}{\raisebox{.5pt}{\textcircled{\raisebox{-.9pt} {1}}\hspace*{2mm}$\{$}}%
\end{algorithm}%
\end{figure*}%
\begin{assumption}\label{assump:avg}
	For an averaging scheme $h \colon \R^{d \times n} \times \R^{d \times n} \to \R^{d \times n} \times \R^{d \times n}$ let $(X^+,Y^+):=h(X,Y)$ for $X,Y \in \R^{d \times n}$. Assume that $h$
	preserves the average of iterates: %
	\begin{align}
	X^+ \frac{\1\1^\top}{n} &= X \frac{\1\1^\top}{n} \,, &\forall X,Y \in \R^{d \times n}\,,
	\intertext{	and that it converges with linear rate for a parameter $0 < c \leq 1$ }
	\EE{h}{\Psi(X^+, Y^+)} &\leq (1 - c) {\Psi(X, Y)}\,, &\forall X,Y \in \R^{d \times n}\,,\label{eq:average}
	\end{align}
	and Laypunov function $\Psi(X, Y) := \|X - \overline{X}\|_F^2 + \|X - Y\|_F^2$ with $\overline{X} := \tfrac{1}{n} X \1\1^\top$, where $\mathbb E_{h}$ denotes the expectation over internal randomness of averaging scheme $h$.
\end{assumption}
\paragraph{Example: Exact Averaging.} Setting $X^+ = XW$ and $Y^+ = X^+$ gives an exact consensus averaging algorithm with mixing matrix $W$ \citep{Xiao2014:averaging}. It converges at the rate $c = \sgap$, where $\sgap$ is an eigengap of mixing matrix $W$, defined in Assumption~\ref{assump:W}. Substituting it into the Algorithm~\ref{alg:arbitrary_avg} we recover D-PSGD algorithm, analyzed in \citet{Lian2017:decentralizedSGD}. 

\paragraph{Example:\ \algopt.} To recover \algopt, we need to choose \algcons~\citep{Koloskova:2019choco} as consensus averaging scheme, which is defined as $X^+ = X + \gamma Y(W - I)$ and $Y^+ = Y + Q(X^+ - Y)$ (in the main text we write $\hat{X}$ instead of $Y$). This scheme converges with $c = \tfrac{\sgap^2\compr}{82}$. The results from the main part can be recovered by substituting this $c = \tfrac{\sgap^2\compr}{82}$ in the more general results below. It is important to note that for Algorithm~\ref{alg:choco} given in the main text, the order of the communication part \raisebox{.5pt}{\textcircled{\raisebox{-.9pt} {1}}} and the gradient computation part \raisebox{.5pt}{\textcircled{\raisebox{-.9pt} {2}}} is exchanged. We did this to better illustrate that both these parts are independent and that they can be executed in parallel. 
The effect of this change can be captured by changing the initial values but does not affect the convergence rate.

\subsection{Proofs}

\begin{remark}[Mini-batch variance]\label{rem:variance} If for functions $f_i$, $F_i$ defined in \eqref{eq:prob} Assumption \ref{assump:f} holds, i.e. $\EE{\xi}{\norm{\nabla F_i(\xx, \xi) - \nabla f_i(\xx)}^2}\leq \sigma_i^2, i \in [n]$, then
	\begin{align}\label{eq:mini-batch}
	\EE{\xi_1^{(t)}, \dots, \xi_n^{(t)}}{\norm{\frac{1}{n}\sum_{i = 1}^{n} \left(\nabla f_i(\xx_i^{(t)}) - \nabla F_i (\xx_i^{(t)}, \xi_i^{(t)})\right)}}^2\leq \dfrac{\overline{\sigma}^2}{n},
	\end{align}
	where $\overline{\sigma}^2 = \frac{\sum_{i = 1}^{n} \sigma_i^2}{n}$.
\end{remark}
\begin{proof} This follows from
	\begin{align*}
	&\E{\norm{\frac{1}{n}\sum_{i = 1}^{n} Y_i}}^2= \dfrac{1}{n^2} \left(\sum_{i = 1}^n\E{\norm{Y_i}}^2 + \sum_{i \neq j} \E{\lin{Y_i, Y_j}}\right) = \dfrac{1}{n^2} \sum_{i = 1}^n\E{\norm{Y_i}}^2 \leq 
	\frac{1}{n^2}\sum_{i = 1}^{n} \sigma_i^2 = \dfrac{\overline{\sigma}^2}{n}
	\end{align*}
	for $Y_i = f_i(\xx_i^{(t)}) - \nabla F_i (\xx_i^{(t)}, \xi_i^{(t)})$. Expectation of scalar product is equal to zero because $\xi_i$ is independent of $\xi_j$ since $i\neq j$.
\end{proof}

\begin{lemma}\label{lem:bound_general}
	Under Assumptions \ref{assump:W}--\ref{assump:avg} the iterates of the Algorithm~\ref{alg:arbitrary_avg} with constant stepsize $\eta$ satisfy 
	\begin{align*}
	\sum_{i = 1}^n \norm{\overline{\xx}^{(t)} -  \xx_i^{(t)}}_2^2 \leq \eta^2 \frac{12 n G^2}{c^2}.
	\end{align*}
\end{lemma}
\begin{proof}[Proof of Lemma~\ref{lem:bound_general}]
	We start by following the proof of Lemma 21 from \citet{Koloskova:2019choco}.
	Define $r_t = \E{\norm{X^{(t)} - \overline{X}^{(t)}}}^2 + \E{\norm{X^{(t)} - Y^{(t)}}}^2$,
	\begin{align*}
	r_{t + 1} & \stackrel{\eqref{eq:average}}{\leq}
	(1-c)\E{\norm{\overline{X}^{(t + \frac{1}{2})} - X^{(t + \frac{1}{2})}}_F^2}  + (1-c)  \E{\norm{Y^{(t)} - {X}^{(t + \frac{1}{2})}}_F^2} \\
	& = (1-c) \E{\norm{\overline{X}^{(t)} - X^{(t)} - \eta\partial F(X^{(t)}, \xi^{(t)})\left(\frac{\1\1^\top}{n} - I\right)}_F^2 }
	\\ &\qquad \qquad + (1-c)\E{\norm{Y^{(t)} - {X}^{(t)} + \eta \partial F(X^{(t)}, \xi^{(t)})}_F^2} \\
	& \stackrel{\eqref{eq:norm_of_sum_of_two}}{\leq} (1-c) (1 + \alpha^{-1}) \E{\left( \norm{\overline{X}^{(t)} - X^{(t)}}_F^2  + \norm{Y^{(t)} - {X}^{(t)}}_F^2 \right)}\\
	&\qquad \qquad + (1-c)(1 + \alpha)\eta^2 \E{\left(\norm{\partial F(X^{(t)}, \xi^{(t)})\left(\frac{\1\1^\top}{n} - I\right)}_F^2 +  \norm{\partial F(X^{(t)}, \xi^{(t)})}_F^2\right)}\\
	& \leq (1-c)\left((1 + \alpha^{-1})\E{\left( \norm{\overline{X}^{(t)} - X^{(t)}}_F^2  + \norm{Y^{(t)} - {X}^{(t)}}_F^2\right)} + 2n(1 + \alpha)\eta^2G^2\right) \\
	&\stackrel{\alpha = \frac{2}{c}}{\leq} \left(1-\frac{c}{2}\right)\E{\left(\norm{\overline{X}^{(t)} - X^{(t)}}_F^2  + \norm{Y^{(t)} - {X}^{(t)}}_F^2\right)} + \frac{6n}{c}\eta^2G^2 \,.
	\end{align*}
	Define $A = 3 n G^2$, we got a recursion 
	\begin{align*}
	r_{t + 1}\leq\left(1 - \frac{c}{2}\right)r_{t} + \frac{2}{c} \eta^2 A,
	\end{align*}
	Verifying that $r_{t} \leq \eta^2 \frac{4A}{c^2}$ satisfy recursion completes the proof as $\E{\norm{X^{(t)} - \overline{X}^{(t)}}}^2 \leq r_t$.
	
	Indeed, $r_0 = 0 \leq \eta^2 \frac{4A}{c^2}$ as $X^{(0)} = \overline{X}^{(0)}$ and $Y^{(0)} = 0$
	\begin{align*}
	r_{t + 1}\leq\left(1 - \frac{c}{2}\right)r_{t} + \eta^2 \frac{2 A}{c} &\leq \left(1 - \frac{c}{2}\right)\eta^2 \frac{4A}{c^2} + \eta^2 \frac{2 A}{c} = \eta^2 \frac{4A}{c^2}. \qedhere
	\end{align*}
\end{proof}

\begin{theorem}\label{th:non-convex-sigma-general}
	Under Assumptions \ref{assump:W}--\ref{assump:avg} with constant stepsize $\eta < \frac{1}{4L}$, the averaged iterates $\overline{\xx}^{(t)} = \frac{1}{n}\sum_{i=1}^n \xx_i^{(t)}$ of Algorithm~\ref{alg:arbitrary_avg} satisfy:
	\begin{align*}
	\frac{1}{T + 1}\sum_{t = 0}^{T}\norm{\nabla f(\overline{\xx}^{(t)})}_2^2
	\leq \frac{4}{\eta (T + 1)} \left(f(\overline{\xx}^{(0)}) - f^\star\right) + \eta  \frac{2 \overline{\sigma}^2 L}{n} + \eta^2 \frac{36 G^2 L^2}{c^2} 
	\end{align*}
	where  $c$ denotes convergence rate of underlying averaging scheme.
\end{theorem}

\begin{proof}[Proof of Theorem \ref{th:non-convex-sigma-general}]
	By $L$-smoothness
	\begin{align*}
	\EE{t + 1}{f(\overline{\xx}^{(t + 1)})} &= \EE{t + 1}{f\left(\overline{\xx}^{(t)} - \frac{\eta}{n}\sum_{i = 1}^n \nabla F_i(\xx_i^{(t)}, \xi_i^{(t)})\right) }\\
	& \leq f(\overline{\xx}^{(t)}) \underbrace{-  \EE{t + 1}{\lin{ \nabla f(\overline{\xx}^{(t)}),  \frac{\eta}{n} \sum_{i = 1}^n \nabla F_i(\xx_i^{(t)}, \xi_i^{(t)}) }}}_{=:T_1} \\& \qquad + \EE{t + 1}{\frac{L}{2} \eta^2 \underbrace{\norm{\frac{1}{n} \sum_{i = 1}^n \nabla F_i(\xx_i^{(t)}, \xi_i^{(t)})}_2^2 }_{=: T_2} }
	\end{align*}
	To estimate the second term, we add and subtract $\nabla f(\overline{\xx}^{(t)})$
	\begin{align*}
	T_1 &= - \eta \norm{\nabla f(\overline{\xx}^{(t)})}^2 + \eta \lin{ \nabla f(\overline{\xx}^{(t)}),  \nabla f(\overline{\xx}^{(t)}) - \frac{1}{n} \sum_{i = 1}^n \nabla f_i(\xx_i^{(t)}) } \\
	&\stackrel{\eqref{eq:scal_product}, \gamma = 1}{\leq} - \frac{\eta}{2}\norm{\nabla f(\overline{\xx}^{(t)})}^2 + \dfrac{\eta}{2 n} \sum_{i = 1}^n \norm{\nabla f(\overline{\xx}^{(t)}) - \nabla f_i(\xx_i^{(t)})}^2
	\end{align*}
	For the last term, we add and subtract $\nabla f(\overline{\xx}^{(t)})$ and the sum of $\nabla f_i(\xx_i^{(t)})$
	\begin{align*}
	T_2 & = 
	\EE{t + 1}{ \norm{\frac{1}{n} \sum_{i = 1}^n \left( \nabla F_i(\xx_i^{(t)}, \xi_i^{(t)})  - \nabla f_i(\xx_i^{(t)}) \right) }_2^2 +  \norm{\frac{1}{n} \sum_{i = 1}^n \nabla f_i(\xx_i^{(t)}) \pm \nabla f(\overline{\xx}^{(t)}) }_2^2 }\\
	& \stackrel{\eqref{eq:mini-batch}, \eqref{eq:norm_of_sum_of_two}, \eqref{eq:norm_of_sum}}{\leq} \dfrac{\overline{\sigma}^2}{n} + \frac{2}{n} \sum_{i = 1}^n \norm{\nabla f(\overline{\xx}^{(t)}) - \nabla f_i(\xx_i^{(t)})}_2^2 + 2 \norm{\nabla f(\overline{\xx}^{(t)})}^2
	\end{align*}
	Combining this together and using $L$-smoothness to estimate $\norm{\nabla f(\overline{\xx}^{(t)}) - \nabla f_i(\xx_i^{(t)})}_2^2$,
	\begin{align*}
	\EE{t + 1}{f(\overline{\xx}^{(t + 1)})} &\leq f(\overline{\xx}^{(t)}) - \eta\left(\frac{1}{2} - L\eta\right)\norm{\nabla f(\overline{\xx}^{(t)})}_2^2 \\ &\qquad + \left(\frac{1}{2} \eta L^2 + \eta^2 L^3 \right) \frac{1}{n}\sum_{i = 1}^n \norm{\overline{\xx}^{(t)} -  \xx_i^{(t)}}_2^2 + \frac{L \eta^2 \overline{\sigma}^2}{ 2n}.
	\end{align*}
	Using Lemma \ref{lem:bound_general} to bound the third term and using that $\eta \leq \frac{1}{4L}$ in the second and in the third terms 
	\begin{align*}
	\EE{t + 1}{f(\overline{\xx}^{(t + 1)})} \leq f(\overline{\xx}^{(t)}) - \frac{\eta}{4} \norm{\nabla f(\overline{\xx}^{(t)})}_2^2 +  \eta^3 \frac{9 L^2 G^2}{c^2}  + \eta^2\frac{L \overline{\sigma}^2}{ 2 n} ,
	\end{align*}
	Rearranging terms and averaging over $t$ 
	\begin{align*}
	\frac{1}{T + 1}\sum_{t = 0}^{T}\norm{\nabla f(\overline{\xx}^{(t)})}_2^2 &\stackrel{\eqref{eq:norm_of_sum_of_two}}{\leq} \frac{4}{\eta}\frac{1}{T + 1}\sum_{t = 0}^T\left(\E{f(\overline{\xx}^{(t)})} - \E{f(\overline{\xx}^{(t + 1)})} \right) + \eta^2 \frac{36 G^2  L^2}{c^2} + \eta \frac{2L \overline{\sigma}^2}{n}\\
	& \leq \frac{4}{\eta (T + 1)} \left(f(\overline{\xx}^{(0)}) - f^\star\right) + \eta  \frac{2 \overline{\sigma}^2 L}{n} + \eta^2 \frac{36 G^2 L^2}{c^2}  \qedhere
	\end{align*}
\end{proof}

\subsection{Corollaries}
\label{sec:arbitraryT}
To obtain final convergence rate we carefully tune the stepsize. For this we consider first an auxiliary lemma.
\begin{lemma}\label{lem:tuning_stepsize}
	For any parameters $r_0 \geq 0, b \geq 0, e \geq 0, d \geq 0$ there exists constant stepsize $\eta \leq \frac{1}{d}$ such that 
	\begin{align*}
	\Psi_T :=  \frac{r_0}{\eta(T + 1)} + b \eta  + e \eta^2 \leq 2  \left(\frac{b r_0}{T + 1}\right)^{\frac{1}{2}} +  2 e^{1/3}\left(\frac{r_0}{T + 1}\right)^{\frac{2}{3}} + \frac{d r_0}{T + 1}
	\end{align*}
\end{lemma}
\begin{proof}
	Choosing $\eta = \min\left\{\left(\frac{r_0}{b(T + 1)}\right)^{\frac{1}{2}}, \left(\frac{r_0}{e(T + 1)}\right)^{\frac{1}{3}} , \frac{1}{d} \right\} \leq \frac{1}{d}$ we have three cases
	\begin{itemize}
		\item $\eta  = \frac{1}{d}$ and is smaller than both $\left(\frac{r_0}{b(T + 1)}\right)^{\frac{1}{2}}$ and $\left(\frac{r_0}{e(T + 1)}\right)^{\frac{1}{3}} $, then 
		\begin{align*}
		\Psi_T &\leq  \frac{d r_0}{T + 1} + \frac{b}{d} + \frac{e}{d^2} \leq \left(\frac{b r_0}{T + 1}\right)^{\frac{1}{2}} + \frac{d r_0}{T + 1} + e^{1/3}\left(\frac{r_0}{T + 1}\right)^{\frac{2}{3}}
		\end{align*}
		\item $\eta = \left(\frac{r_0}{b(T + 1)}\right)^{\frac{1}{2}} < \left(\frac{r_0}{e(T + 1)}\right)^{\frac{1}{3}} $, then
		\begin{align*}
		\Psi_T &\leq  2 \left(\frac{r_0b}{T + 1 }\right)^{\frac{1}{2}}   + e  \left(\frac{r_0}{b(T + 1)}\right) \leq   2 \left(\frac{r_0b}{T + 1}\right)^{\frac{1}{2}} + e^{\frac{1}{3}} \left(\frac{r_0}{(T + 1)}\right)^{\frac{2}{3}},
		\end{align*}
		\item The last case, $\eta = \left(\frac{r_0}{e(T + 1)}\right)^{\frac{1}{3}} < \left(\frac{r_0}{b(T + 1)}\right)^{\frac{1}{2}} $
		\begin{align*}
		\Psi_T &\leq  2 e^{\frac{1}{3}} \left(\frac{r_0}{(T + 1)}\right)^{\frac{2}{3}} + b \left(\frac{r_0}{e(T + 1)}\right)^{\frac{1}{3}} \leq 2 e^{\frac{1}{3}} \left(\frac{r_0}{(T + 1)}\right)^{\frac{2}{3}} + \left(\frac{b r_0}{T + 1}\right)^{\frac{1}{2}} \qedhere
		\end{align*}
	\end{itemize}
\end{proof}
\begin{corollary}[Generalized Theorem~\ref{th:non-convex-sigma}]\label{cor:main_result}
	Under Assumptions \ref{assump:W}--\ref{assump:avg} with constant stepsize $\eta$ tuned as in Lemma~\ref{lem:tuning_stepsize}, the averaged iterates $\overline{\xx}^{(t)} = \frac{1}{n}\sum_{i=1}^n \xx_i^{(t)}$ of Algorithm~\ref{alg:arbitrary_avg} satisfy:
	\begin{align*}
	\frac{1}{T + 1}\sum_{t = 0}^{T}\norm{\nabla f(\overline{\xx}^{(t)})}_2^2
	& \leq 4  \sqrt{\frac{2 L  \overline{\sigma}^2 }{n (T + 1)}} +  17 \left(\frac{ G L F_0}{c (T + 1 )}\right)^{\frac{2}{3}} + \frac{16 L F_0}{T + 1}
	\end{align*}
	where  $c$ denotes convergence rate of underlying averaging scheme, $F_0 = f(\overline{\xx}^{(0)}) - f^\star$.
\end{corollary}
\begin{proof}
	The result follows from Theorem~\ref{th:non-convex-sigma-general} and Lemma~\ref{lem:tuning_stepsize} with $r_0 = 4 \left(f(\overline{\xx}^{(0)}) - f^\star\right) $, $b = \frac{2 \overline{\sigma}^2 L}{n}$, $e = \frac{36 G^2 L^2}{c^2}$ and $d = 4L$. 
\end{proof}

The first term shows a linear speed up compared to SGD on one node, whereas the underlying averaging scheme affects only the second-order term. Substituting the convergence rate for exact averaging with $W$ ($c = \rho$) gives the rate $\cO(\nicefrac{1}{\sqrt{nT}} + \nicefrac{1}{(T\sgap)^{\frac{2}{3}}})$. %

\algopt with the underlying \algcons averaging scheme converges at the rate $\cO(\nicefrac{1}{\sqrt{nT}} + \nicefrac{1}{(T\sgap^2 \compr)^{\frac{2}{3}}})$. The dependence on $\sgap$ (eigengap of the mixing matrix $W$) is worse than in the exact case. This might either just be an artifact of our proof technique or a consequence of supporting arbitrary high compression. %

The corollary gives guarantees for the averaged vector of parameters $\overline{\xx}$, however in a decentralized setting it is very expensive and sometimes impossible to average all the parameters distributed across several machines, especially when the number of machines and the model size is large. We can get similar guarantees on the individual iterates $\xx_i$ as e.g.\ in~\citep{Assran:2018sdggradpush}. We summarize these briefly below. 
\begin{corollary}[Convergence of local weights]\label{lem:local_weights_convergence_sigma}
	Under the same setting as in Corollary~\ref{cor:main_result},
	\begin{align*}
	\frac{1}{T + 1} \sum_{t = 0}^{T} \frac{1}{n}\sum_{i=1}^n\Big\| \nabla f\bigl(\xx_i^{(t)}\bigr)\Big\|_2^2 &\leq 8 \sqrt{\frac{2 L  \overline{\sigma}^2 }{n (T + 1)}} +  37 \left(\frac{ G L F_0}{c (T + 1 )}\right)^{\frac{2}{3}} + \frac{32 L F_0}{T + 1}
	\end{align*}
\end{corollary}
\begin{proof}[Proof of Corollary~\ref{lem:local_weights_convergence_sigma}]
	\begin{align*}
	\frac{1}{T + 1} \sum_{t = 0}^{T} \frac{1}{n}\sum_{i=1}^n\norm{\nabla f(\xx_i^{(t)})}_2^2 &\leq \frac{1}{T + 1} \sum_{t = 0}^{T} \frac{1}{n}\sum_{i=1}^n \left(2\norm{\nabla f(\xx_i^{(t)}) - \nabla f(\overline{\xx}^{(t)})}_2^2 + 2 \norm{\nabla f(\overline{\xx}^{(t)})}_2^2\right)\\
	& \leq  \frac{1}{T + 1} \sum_{t = 0}^{T} \frac{1}{n}\sum_{i=1}^n \left(2 L^2 \norm{\xx_i^{(t)} -\overline{\xx}^{(t)}}_2^2 + 2 \norm{\nabla f(\overline{\xx}^{(t)})}_2^2\right)
	\end{align*}
	where we used $L$-smoothness of $f$. Using Theorem~\ref{th:non-convex-sigma-general} and tuning the stepsize as in Lemma~\ref{lem:tuning_stepsize} we get the statement of the corollary.
\end{proof}

Choosing the stepsize differently, we can also get the following convergence rate for $T =\Omega(nL^2)$:

\begin{corollary}
	Under Assumptions \ref{assump:W}--\ref{assump:avg} with constant stepsize $\eta =\sqrt{\frac{n}{T + 1}}$ for  $T \geq 16 nL^2$, the averaged iterates $\overline{\xx}^{(t)} = \frac{1}{n}\sum_{i=1}^n \xx_i^{(t)}$ of Algorithm~\ref{alg:arbitrary_avg} satisfy:
	\begin{align*}
	\frac{1}{T + 1}\sum_{t = 0}^{T}\norm{\nabla f(\overline{\xx}^{(t)})}_2^2
	& \leq \frac{4\left(f(\overline{\xx}^{(0)}) - f^\star\right)  + 2 \overline{\sigma}^2 L}{\sqrt{n (T + 1)}} + \frac{36 G^2 nL^2}{(T + 1)c^2}
	\end{align*}
	where  $c$ denotes convergence rate of underlying averaging scheme.	
\end{corollary}

\section{Useful Inequalities}

\begin{lemma}\label{remark:norm_of_sum}
	For arbitrary set of $n$ vectors $\{\aa_i\}_{i = 1}^n$, $\aa_i \in \R^d$
	\begin{equation}\label{eq:norm_of_sum}
	\norm{\sum_{i = 1}^n \aa_i}^2 \leq n \sum_{i = 1}^n \norm{\aa_i}^2 \,.
	\end{equation}
\end{lemma}
\begin{lemma}\label{remark:scal_product}
	For given two vectors $\aa, \bb \in \R^d$
	\begin{align}\label{eq:scal_product}
	&2\lin{\aa, \bb} \leq \gamma \norm{\aa}^2 + \gamma^{-1}\norm{\bb}^2\,, & &\forall \gamma > 0 \,.
	\end{align}
\end{lemma}
\begin{lemma}\label{remark:norm_of_sum_of_two}
	For given two vectors $\aa, \bb \in \R^d$ %
	\begin{align}\label{eq:norm_of_sum_of_two}
	\norm{\aa + \bb}^2 \leq (1 + \alpha)\norm{\aa}^2 + (1 + \alpha^{-1})\norm{\bb}^2,\,\, & &\forall \alpha > 0\,.
	\end{align}
	This inequality also holds for the sum of two matrices $A,B \in \R^{n \times d}$ in Frobenius norm.
\end{lemma}

\section{Compression Schemes}\label{sect:compr}
We implement the compression schemes detailed below. %

\begin{itemize}
 \item $\operatorname{gsgd}_b$ \citep{Alistarh2017:qsgd}. The unbiased $\operatorname{gsgd}_b \colon \R^d \to \R^d$ compression operator (for $b > 1)$ is given as
 \begin{align*}
   \operatorname{gsgd}_b (\xx) := \norm{\xx}_2  \cdot \operatorname{sig}(\xx) \cdot 2^{-(b-1)} \cdot \left\lfloor \frac{2^{(b-1)} \abs{\xx}}{\norm{\xx}_2} + \uu \right\rfloor
 \end{align*}
 where $\uu \sim_{u.a.r.} [0,1]^d$ is a random dithering vector and $\operatorname{sig}(\xx)$ assigns the element-wise sign: $(\operatorname{sig}(\xx))_i = 1$ if $(\xx)_i \geq 0$ and $(\operatorname{sig}(\xx))_i = -1$ if $(\xx)_i < 0$. As the value in the right bracket will be rounded to an integer in $\{0,\dots, 2^{(b-1)}-1\}$, each coordinate can be encoded with at most $(b-1)+1$ bits (1 for the sign). For more efficent encoding schemes cf.~\citet{Alistarh2017:qsgd}. 
 
 A biased version is given as
 \begin{align*}
   \operatorname{gsgd}_b (\xx) := \frac{\norm{\xx}_2}{\tau}  \cdot \operatorname{sig}(\xx) \cdot 2^{-(b-1)} \cdot \left\lfloor \frac{2^{(b-1)} \abs{\xx}}{\norm{\xx}_2} + \uu \right\rfloor
 \end{align*} 
 for $\tau = 1 + \min\left\{\frac{d}{2^{2(b-1)}}, \frac{\sqrt{d}}{2^{(b-1})} \right\} $  and is a $\delta = \frac{1}{\tau}$ compression operator~\citep{Koloskova:2019choco}.
 \item $\operatorname{random}_a$ \citep{Wangni2018:sparsification}. Let $\uu \in \{0,1\}^d$ be a masking vector, sampled uniformly at random from the set $\{\uu \in \{0,1\}^d : \norm{\uu}_1 = \lfloor a d \rfloor\}$. Then the unbiased $\operatorname{random}_a \colon \R^d \to \R^d$ operator is defined as
 \begin{align*}
  \operatorname{random}_a(\xx) := \frac{d}{\lfloor a d \rfloor} \cdot \xx \odot \uu\,.
 \end{align*}
 The biased version is given as
 \begin{align*}
 \operatorname{random}_a(\xx) := \xx \odot \uu\,,
 \end{align*}
 and is a $\delta = a$ compression operator~\citep{Stich2018:sparsifiedSGD}. 
 
Only $32 {\lfloor a d \rfloor}$ bits are required to send $\operatorname{random}_a(\xx)$ to another node---all the values of non-zero entries (we assume that entries are represented as \texttt{float32} numbers). Receiver can recover positions of these entries if it knows the random seed of uniform sampling operator used to select these entries. This random seed could be communicated once on preprocessing stage (before starting the algorithm).
 \item $\operatorname{top}_a$ \citep{Alistarh2018:topk,Stich2018:sparsifiedSGD}.
 The biased $\operatorname{top}_a \colon \R^d \to \R^d$ operator is defined as
 \begin{align*}
    \operatorname{top}_a(\xx) := \xx \odot \uu(\xx)\,,
 \end{align*}
 where $\uu(\xx) \in \{0,1\}^d$, $\norm{\uu}_1 = \lfloor a d \rfloor$ is a masking vector with $(\uu)_i = 1$ for indices $i \in \pi^{-1}(\{1,\dots,\lfloor a d \rfloor\})$ where the permutation $\pi$ is such that $\abs{(\xx)_{\pi(1)}} \geq \abs{(\xx)_{\pi(2)}} \geq \cdots \geq \abs{(\xx)_{\pi(d)}}$.
 The $\operatorname{top}_a$ operator is a $\delta =a$ compression operator~\citep{Stich2018:sparsifiedSGD}.
 
  In the case of $\operatorname{top}_a$ compression $2 \cdot 32 {\lfloor a d \rfloor}$ bits are required because along with the values we need to send positions of these values. 
 \item $\operatorname{sign}$ \citep{Bernstein2018:sign,KarimireddyRSJ2019feedback}.
 The biased (scaled) $\operatorname{sign} \colon \R^d \to \R$ compression operator is defined as
 \begin{align*}
   \operatorname{sign}(\xx) := \frac{\norm{\xx}_1}{d} \cdot \operatorname{sgn}(\xx) \,.
 \end{align*}
 The $\operatorname{sign}$ operator is a $\delta = \frac{\norm{\xx}_1^2}{d\norm{\xx}_2^2} $ compression operator~\citep{KarimireddyRSJ2019feedback}.

In total for the $\operatorname{sign}$ compression we need to send only $d + 32$ bits---one bit for every entry in $\xx$ and 32 bits for $\norm{\xx}_1$.
\end{itemize}

\section{\algopt with Momentum}\label{sec:momentum}
Algorithm~\ref{alg:choco_with_momentum} demonstrates how to combine \choco with weight decay and momentum.
Nesterov momentum can be analogously adapted for our decentralized setting.

\section{Error Feedback Interpretation of \algopt}\label{sec:error-feedback}
To better understand how does \algopt work, we can interpret it as an error feedback algorithm \citep{Stich2018:sparsifiedSGD,KarimireddyRSJ2019feedback,stich2019error}. We can equivalently rewrite \algopt (Algorithm~\ref{alg:choco}) as Algorithm~\ref{alg:choco_error_feedback}.
The common feature of error feedback algorithms is that quantization errors are saved into the internal memory, which is added to the compressed value at the next iteration. In \algopt the value we want to transmit is the difference $\xx_i^{(t)} - \xx_i^{(t -1)}$, which represents the evolution of local variable $\xx_i$ at step $t$. Before compressing this value on line 4, the internal memory is added on line 3 to correct for the errors. Then, on line 5 internal memory is updated. 
Note that $\mm_i^{(t)} = \xx^{(t - 1)}_i - \hat{\xx}_i^{(t)}$ in the old notation.
\begin{figure*}[tb]
	\setlength{\mycorrect}{\widthof{$\hat{\xx}^{(t)}$}}
	\algsetup{
		linenodelimiter = {\makeatletter\tikzmark{\arabic{ALC@line}}\makeatother:\phantom{$\hat{\xx}^{(t)}$}\hspace{-\mycorrect}}
	}
	\algsetup{
		linenodelimiter = {\makeatletter\tikzmark{\arabic{ALC@line}}\makeatother:}
	}
	\begin{algorithm}[H]
		\caption{\choco~\citep{Koloskova:2019choco} as Error Feedback}
		\text{\textbf{input:} Initial values $\xx_i^{(0)} \in \R^d$ on each node $i \in [n]$, consensus stepsize $\gamma$, SGD stepsize $\eta$,}
		\text{\qquad \quad comm. graph $G = ([n], E)$ and mixing matrix $W$, initialize $\hat{\xx}_i^{(0)} = \xx_i^{(-1)} := \0$, $\forall i \in [n]$}  \par\vspace{1mm}
		\hspace*{\SpaceReservedForComments}{}%
		\begin{minipage}{\dimexpr\linewidth-\SpaceReservedForComments\relax}
			\fontsize{10}{14}\selectfont %
			\begin{algorithmic}[1]
				\FOR[{{\it in parallel for all workers $i \in [n]$}}]{$t$\textbf{ in} $0\dots T-1$}
				\STATE $\xx_i^{(t)} := \xx_i^{(t - \frac{1}{2})} + \gamma \textstyle\sum_{j: \{i, j\}\in E} w_{ij} \bigl(\hat{\xx}^{(t)}_j \!- \hat{\xx}^{(t)}_i\bigr)$ \hfill $\triangleleft$ modified gossip averaging
				\STATE $\vv_i^{(t)} = \xx_i^{(t)} - \xx_i^{(t - 1)} + \mm_i^{(t)}$
				\STATE  $\qq_i^{(t)} := Q(\vv_i^{(t)})$	\hfill$\triangleleft$ compression\\
				\STATE  $\mm_i^{(t + 1)} = \vv_i^{(t)} - \qq_i^{(t)}$ \hfill$\triangleleft$ memory update\\
				\FOR{neighbors $j \colon \{i,j\} \in E$ (including $\{i\} \in E$)} 
				\STATE Send $\qq_i^{(t)}$ and receive $\qq_j^{(t)}$ %
				\hfill$\triangleleft$ communication 
				\STATE $\hat{\xx}^{(t+1)}_j := \qq^{(t)}_j + \hat{\xx}_j^{(t)}$  \hfill$\triangleleft$ local update
				\ENDFOR
				
				\STATE Sample $\xi_i^{(t)}$, compute gradient $\gg_i^{(t)} \!:= \nabla F_i(\xx_i^{(t)}\!, \xi_i^{(t)})$\!
				\STATE $\xx_i^{(t + \frac{1}{2})} := \xx_i^{(t)} - \eta \gg_i^{(t)}$ %
				\hfill$\triangleleft$ stochastic gradient update
				\ENDFOR 
			\end{algorithmic}\label{alg:choco_error_feedback}
		\end{minipage}
	\end{algorithm}
	\vspace{-0.5cm}
\end{figure*}

\section{Detailed Experimental Setup and Tuned Hyperparameters}\label{sect:parameters}
We precise the procedure of model training as well as the hyper-parameter tuning in this section.

\paragraph{Social Network Setup.}
For the comparison we consider \algopt with $\operatorname{sign}$ compression (this combination achieved the compromise between accuracy and compression level in Table~\ref{tb:toy})), decentralized SGD without compression, and centralized SGD without compression.
We train two models, firstly \texttt{ResNet20} \citep{He2016:Resnet} ($0.27$ million parameters) for image classification on the \texttt{Cifar10} dataset (50K/10K training/test samples)~\citep{Krizhevsky2012:cifar10} and secondly, a three-layer \texttt{LSTM} architecture \citep{Hochreiter1997:LSTM} ($28.95$ million parameters) for a language modeling task on WikiText-2 (600 training and 60 validation articles with a total of $2'088'628$ and $217'646$ tokens respectively)~\citep{merity2016pointer}. For the language modeling task, we borrowed and adapted the general experimental setup of~\citet{merity2017regularizing}, where we use a three-layer LSTM with hidden dimension of size $650$. The loss is averaged over all examples and timesteps.
The BPTT length is set to $30$.
We fine-tune the value of gradient clipping ($0.4$), and the dropout ($0.4$) is only applied on the output of LSTM.

We train both of \texttt{ResNet20} and LSTM for $300$ epochs, unless mentioned specifically.
The per node mini-batch size is $32$ for both datasets.
The momentum (with factor $0.9$) is only applied on the \texttt{ResNet20} training.

\paragraph{Social Network and a Datacenter details.}

For all algorithms, we gradually warmup~\citep{goyal2017accurate} the learning rate from a relative small value (0.1) to the fine-tuned initial learning rate for the first $5$ training epochs.
During the training procedure, the tuned initial learning rate is decayed by the factor of $10$ when accessing $50\%$ and $75\%$ of the total training epochs.
The learning rate is tuned by finding the optimal initial learning rate (after the scaling).

The optimal $\hat{\eta}$ is searched in a pre-defined grid 
and we ensure that the best performance was contained in the middle of the grids.
For example, if the best performance was ever at one of the extremes of the grid, we would try new grid points.
Same searching logic applies to the consensus stepsize.

Table~\ref{tab:choco_hpyerparameters_ring_cifar10_resnet20} demonstrates the fine-tuned hpyerparameters of \algopt for training ResNet-20 on Cifar10,
while Table~\ref{tab:dcd_ecd_hpyerparameters_ring_cifar10_resnet20} reports our fine-tuned hpyerparameters of our baselines.
Table~\ref{tab:choco_hpyerparameters_social} demonstrates the fine-tuned hpyerparameters of \algopt for training ResNet-20/LSTM on a social network topology.

We estimate the runtime information (depicted in Figure~\ref{fig:datacenter}) of different methods from three trials of the evaluation on Google Cloud (Kubernetes Engine).
More precisely, we create the cluster on Google Cloud for three times and each time we estimate the time per mini-batch of different methods (through the first two training epochs).

\begin{table}[!h]
	\centering
	\caption{\small{
		Tuned hyper-parameters of \algopt for training \texttt{ResNet-20} on \texttt{Cifar10},
		corresponding to the ring topology with $8$ nodes in Table~\ref{tb:toy}.
		We randomly split the training data between nodes and shuffle it after every epoch.
		The per node mini-batch size is $128$ and the degree of each node is $3$.
	}}\vspace{2mm}
	\label{tab:choco_hpyerparameters_ring_cifar10_resnet20}
	\begin{tabular}{ccc} 
	\toprule
	\textbf{Compression schemes} & \textbf{Learning rate} & \textbf{Consensus stepsize} \\ \midrule
	QSGD (16-bit) & 1.60 & 0.2 \\
	QSGD (8-bit)  & 0.96 & 0.2 \\
	QSGD (4-bit)  & 1.60 & 0.075 \\
	QSGD (2-bit)  & 0.96 & 0.025 \\
	Sparsification (random-50\%) & 2.40 & 0.45 \\
	Sparsification (random-10\%) & 1.20 & 0.075 \\
	Sparsification (random-1\%) & 0.48 & 0.00625 \\
	Sparsification (top-50\%) & 1.60 & 0.45 \\
	Sparsification (top-10\%) & 1.60 & 0.15 \\
	Sparsification (top-1\%) & 1.20 & 0.0375 \\
	Sign+Norm & 1.60 & 0.45 \\ \bottomrule
	\end{tabular}
\end{table}

\begin{table}[!h]
	\centering
	\caption{\small{
		Tuned hyper-parameters of \algopt,
		corresponding to the social network topology with $32$ nodes in Table~\ref{tab:social}.
		We randomly split the training data between the nodes and keep this partition fixed during the entire training (no shuffling).
		The per node mini-batch size is $32$ and the maximum degree of the node is $14$.
	}}\vspace{2mm}
	\label{tab:choco_hpyerparameters_social}
	\resizebox{.7\textwidth}{!}{%
	\begin{tabular}{ccc}
	\toprule
	\textbf{Configuration} 							& \textbf{Learning rate} 	& \textbf{Consensus stepsize} 	\\ \midrule
	ResNet-20, Cifar10, Sign+Norm 					& 1.0 						& 0.5 							\\
	LSTM, WikiText-2, Sign+Norm 					& 25 						& 0.6   						\\ \bottomrule
	\end{tabular}%
	}
\end{table}

\begin{table}[!h]
	\centering
	\caption{\small{
		Tuned hyper-parameters of DCD, ECD, and DeepSqueeze for training \texttt{ResNet-20} on \texttt{Cifar10},
		corresponding to the ring topology with $8$ nodes in Table~\ref{tb:toy}.
		We randomly split the training data between nodes and shuffle it after every epoch.
		The per node mini-batch size is $128$ and the degree of each node is $3$.
		We only report the hpyerparameters corresponding to results that can reach to reasonable performance in our experiments.
	}}\vspace{2mm}
	\label{tab:dcd_ecd_hpyerparameters_ring_cifar10_resnet20}
	\begin{tabular}{lcc} 
	\toprule
	\textbf{Compression schemes} 			& \textbf{Learning rate}		& \textbf{Consensus stepsize} \\ \midrule
	DCD, QSGD (16-bit) 						& 2.40  							& - \\
	DCD, QSGD (8-bit)  						& 1.20  							& - \\
	DCD, Sparsification (random-50\%) 		& 0.80  							& - \\
	DCD, Sparsification (top-50\%) 			& 1.20  							& - \\
	DCD, Sparsification (top-10\%) 			& 1.60  							& - \\
	DCD, Sparsification (top-1\%)			& 2.40  							& - \\
	ECD, QSGD (16-bit) 						& 0.96  							& - \\
	ECD, QSGD (8-bit) 						& 1.20  							& - \\ \midrule
	DeepSqueeze, QSGD (4-bit)				& 0.60 							& 0.01  \\
	DeepSqueeze, QSGD (2-bit)				& 0.80 							& 0.005  \\
	DeepSqueeze, Sparsification (top-50\%)	& 0.80							& 0.05   \\
	DeepSqueeze, Sparsification (top-10\%)	& 0.60							& 0.01  \\
	DeepSqueeze, Sparsification (top-1\%)	& 0.40							& 0.005  \\
	DeepSqueeze, Sparsification (random-1\%)& 0.80							& 0.0005  \\
	DeepSqueeze, Sign+Norm					& 0.48							& 0.01  \\
	\bottomrule
	\end{tabular}
\end{table}

\clearpage
\section{Additional Plots} \label{sec:additional_plots}
To complement our results for scaling to a large number of nodes, 
we here additionally depict the learning curves (e.g. test accuracy) for the training on 64 nodes.
We also mark the levels used for Fig.~\ref{fig:social}. 

\begin{figure*}[h!]
	\vspace{-1em}
	\centering
	\subfigure%
	{
		\includegraphics[width=0.31\textwidth,]{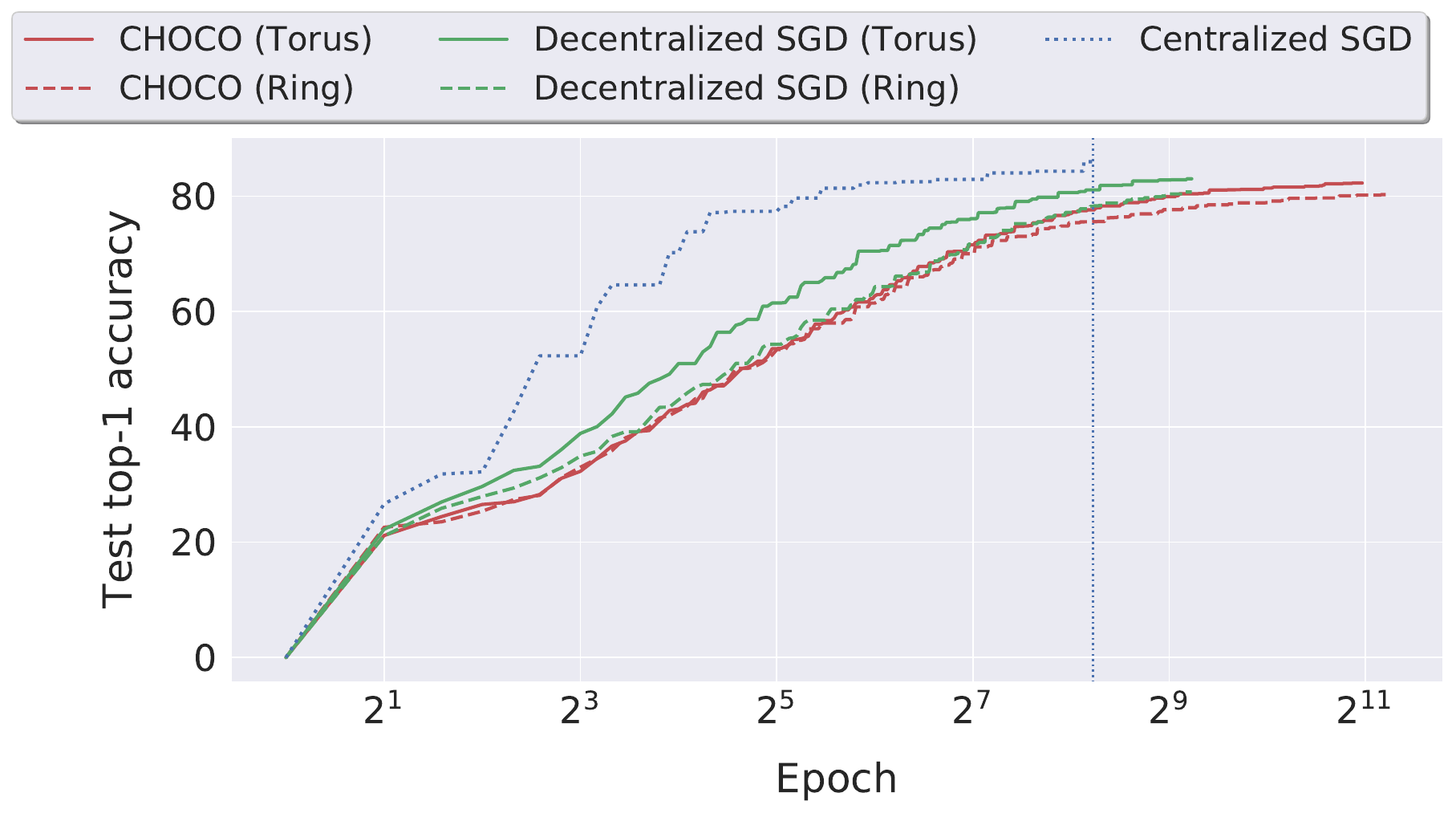}
	}
	\subfigure%
	{
		\includegraphics[width=0.37\textwidth,]{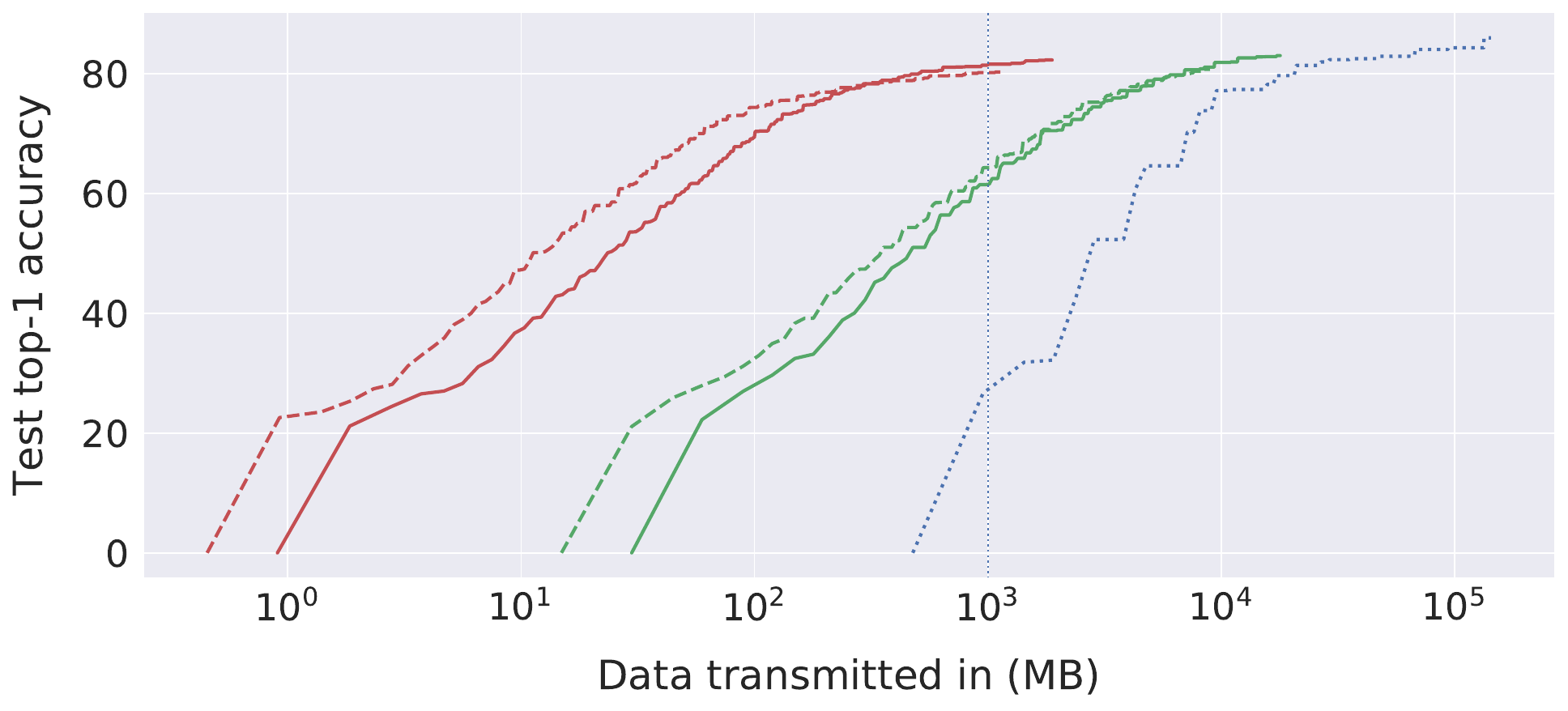}
	}
	\hfill
	
	\vspace{-0.5em}  
	\caption{\normalsize{
			Scaling of \algopt with $\operatorname{sign}$ compression to large number of devices on \texttt{Cifar10} dataset. Convergence curves for 64 nodes. Vertical lines corresponds to the epoch/bits budget used in Fig.~\ref{fig:social}.
	}}
	\label{fig:social_app}
	\vspace{-0.5em}
\end{figure*} 

\begin{table}[!h]
	\centering
	\caption{\small The exact epoch for the same bits budget in Fig.~\ref{fig:social}.}
	\label{tab:epoch_for_same_bits_budget}
	\begin{tabular}{ccccc}
	\toprule
						  & $n=4$ & $n=16$ & $n=36$ & $n=64$ \\ \midrule
	Centralized           & 5     & 6      & 6      & 6      \\
	Decentralized (Ring)  & 7     & 17     & 32     & 54     \\
	Decentralized (Torus) & 6     & 10     & 18     & 29     \\
	CHOCO (Ring)          & 105   & 408    & 904    & 1588   \\
	CHOCO (Torus)         & 55    & 206    & 454    & 796    \\
	\bottomrule
	\end{tabular}
\end{table}

\begin{table}[!h]
	\centering
	\caption{\small The exact transmitted bits (in MB) for the same epoch budget in Fig.~\ref{fig:social}.}
	\label{tab:bits_for_same_epoch_budget}
	\begin{tabular}{ccccc}
	\toprule
						  & $n=4$  & $n=16$ & $n=36$ & $n=64$ \\ \midrule
	Centralized           & 139683 & 140041 & 144299 & 142899 \\
	Decentralized (Ring)  & 69841  & 17505  & 8016   & 4554   \\
	Decentralized (Torus) & 139683 & 35010  & 16033  & 9109   \\
	CHOCO (Ring)          & 2208   & 564    & 253    & 144    \\
	CHOCO (Torus)         & 4417   & 1129   & 506    & 288   \\
	\bottomrule
	\end{tabular}
\end{table}

We additionally visualize the learning curves for the social network topology in Fig.~\ref{fig:social_additional_resnet20_cifar10} and Fig.~\ref{fig:social_additional_lstm_wikitext2}. 

\begin{figure*}[!h]
	\centering
	    \subfigure[\small{Training top-1 accuracy.\vspace{-2mm}}]{
	        \includegraphics[width=0.31\textwidth,]{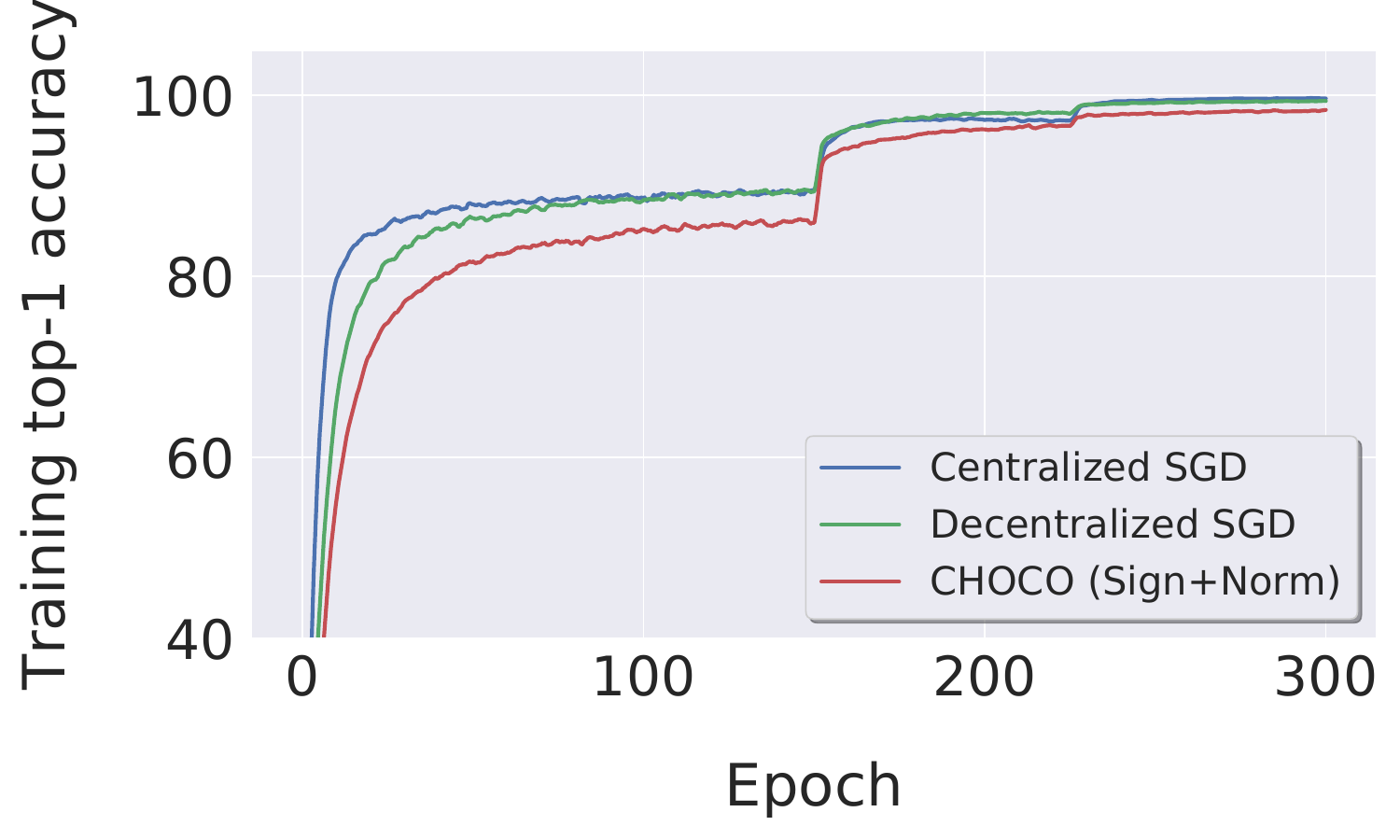}
	        \label{fig:resnet20_cifar10_k32_bs32_social_topology_tr_top1_vs_epoch}
			}
			\hfill
	    \subfigure[\small{Training top-1 accuracy.\vspace{-2mm}}]{
	        \includegraphics[width=0.31\textwidth,]{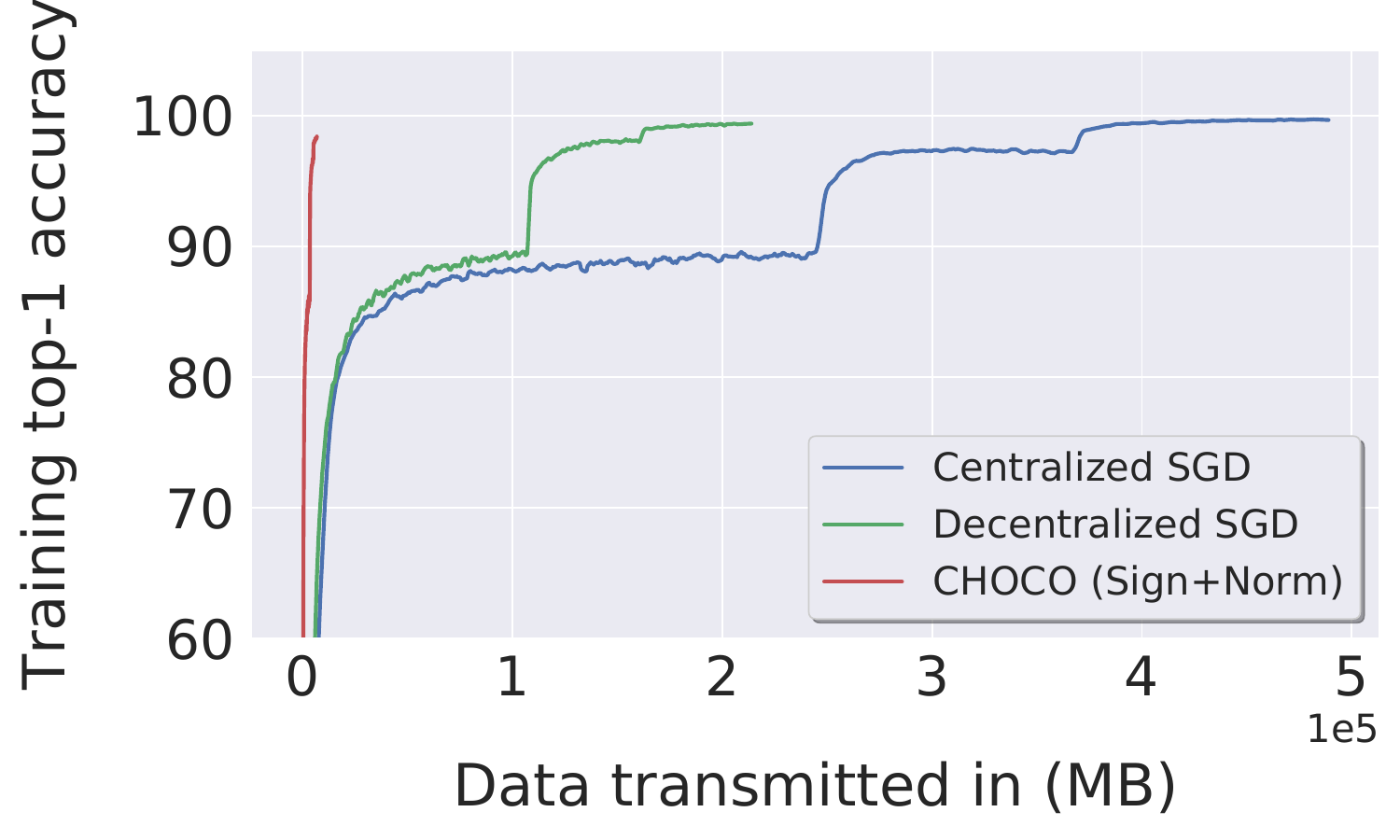}
	        \label{fig:resnet20_cifar10_k32_bs32_social_topology_tr_top1_vs_bits}
			}
			\hfill
			\subfigure[\small{Test top-1 accuracy.\vspace{-2mm}}]{
				\includegraphics[width=0.31\textwidth,]{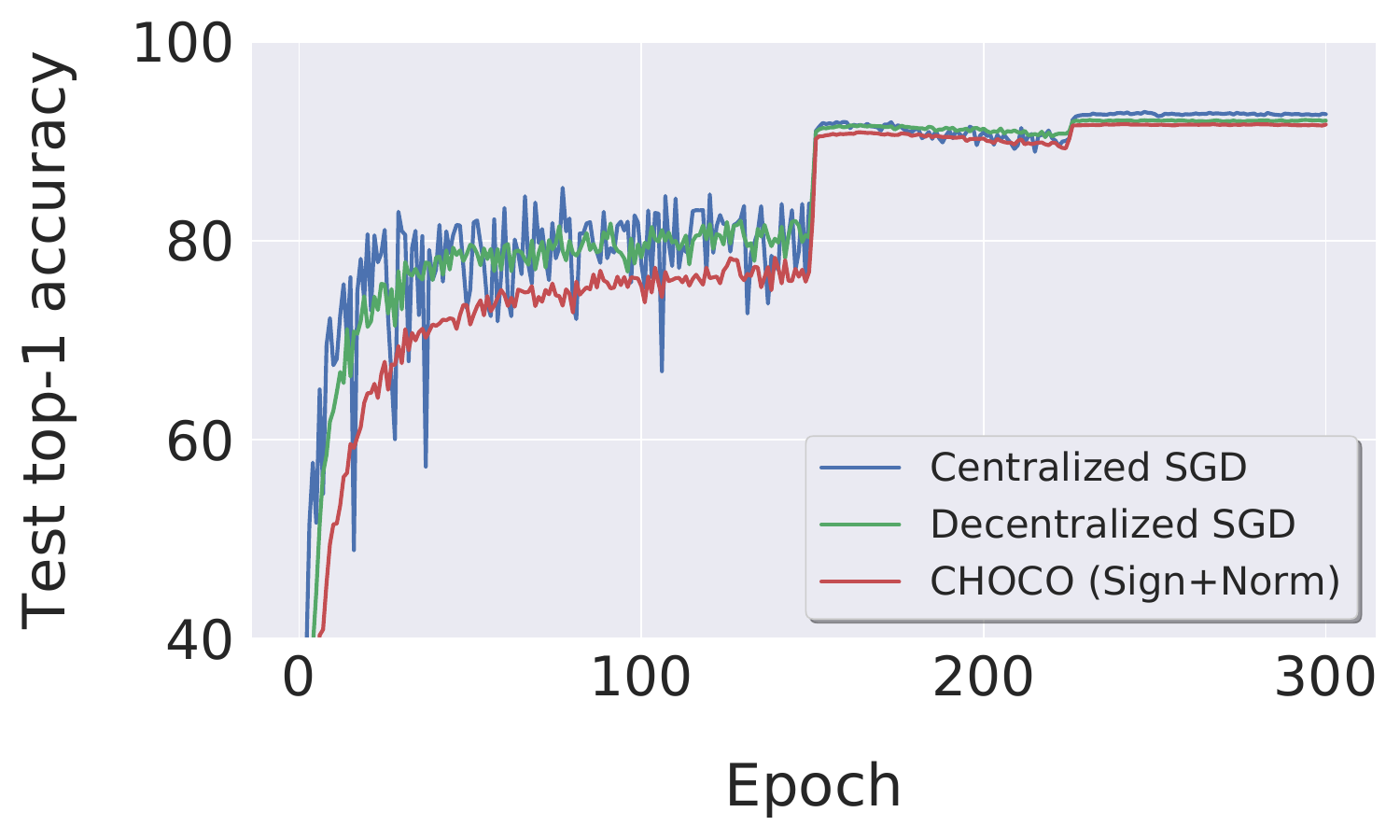}
				\label{fig:resnet20_cifar10_k32_bs32_social_topology_te_top1_vs_epoch}
			}

	\caption{\small{
			Training \texttt{ResNet-20} on \texttt{CIFAR-10} with decentralized algorithm on a real world social network topology.
			The topology has $32$ nodes and we assume each node can only access a disjoint subset of the whole dataset. %
			The local mini-batch size is $32$.
	}}
	\label{fig:social_additional_resnet20_cifar10}
\end{figure*}

\begin{figure*}[!h]
	\centering
	\subfigure[\small{Test loss.\vspace{-2mm}}]{
		\includegraphics[width=0.475\textwidth,]{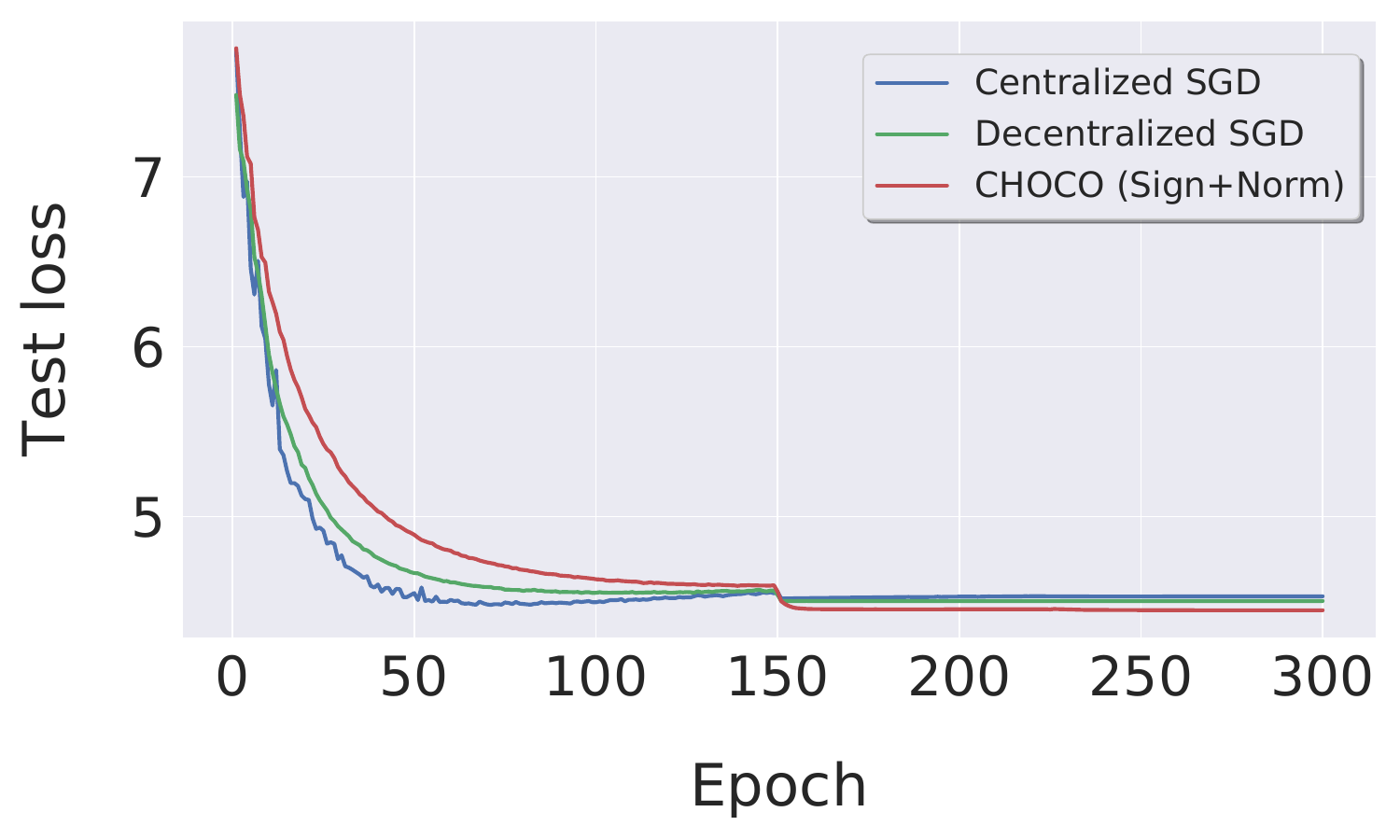}
		\label{fig:lstm_wikitext2_social_topology_te_loss_vs_epoch}
	}
	\hfill
	\subfigure[\small{Test perplexity.\vspace{-2mm}}]{
		\includegraphics[width=0.475\textwidth,]{figures/lstm_wikitext2_social_topology_te_ppl_vs_bits.pdf}
		\label{fig:lstm_wikitext2_social_topology_te_ppl_vs_bits}
	}
	\caption{\small{
		Training \texttt{LSTM} on \texttt{WikiText2} with decentralized algorithm on a real world social network topology.
		The topology has $32$ nodes and we assume each node can only access a disjoint subset of the whole dataset. %
		The local mini-batch size is $32$.
	}}
	\label{fig:social_additional_lstm_wikitext2}
\end{figure*}

We additionally provide the learning curves of training top-1, top-5 accuracy and test top-5 accuracy for the datacenter experiment in Fig.~\ref{fig:datacenter_additional_resnet50_imagenet}.
\begin{figure*}[!h]
	\centering
	\subfigure[\small{Training top-1 accuracy.\vspace{-2mm}}]
	{
		\includegraphics[width=0.31\textwidth,]{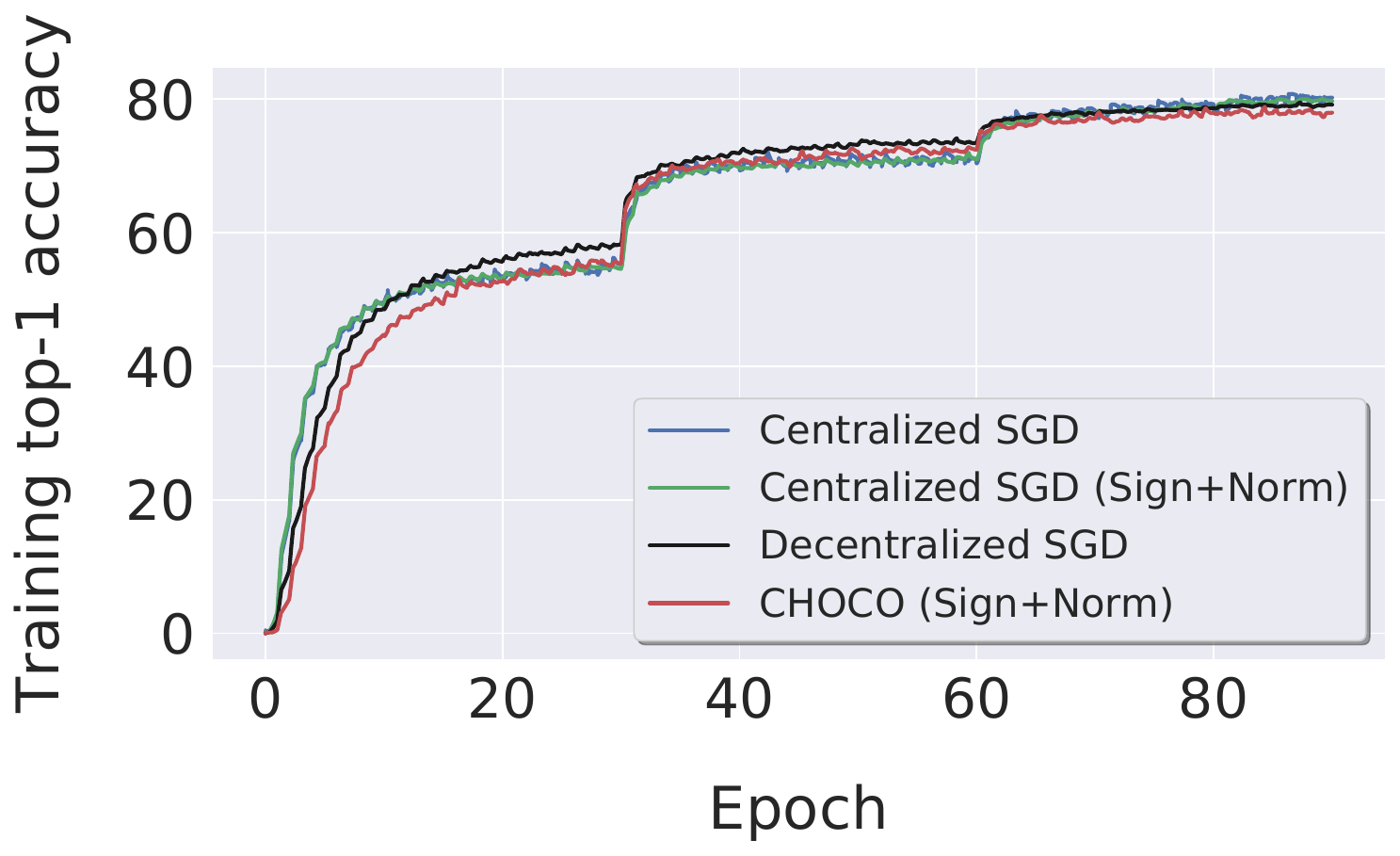}
		\label{fig:resnet20_cifar10_k32_bs32_social_topology_tr_loss_vs_epoch}
	}
	\hfill
	\subfigure[\small{Training top-5 accuracy.\vspace{-2mm}}]
	{
		\includegraphics[width=0.31\textwidth,]{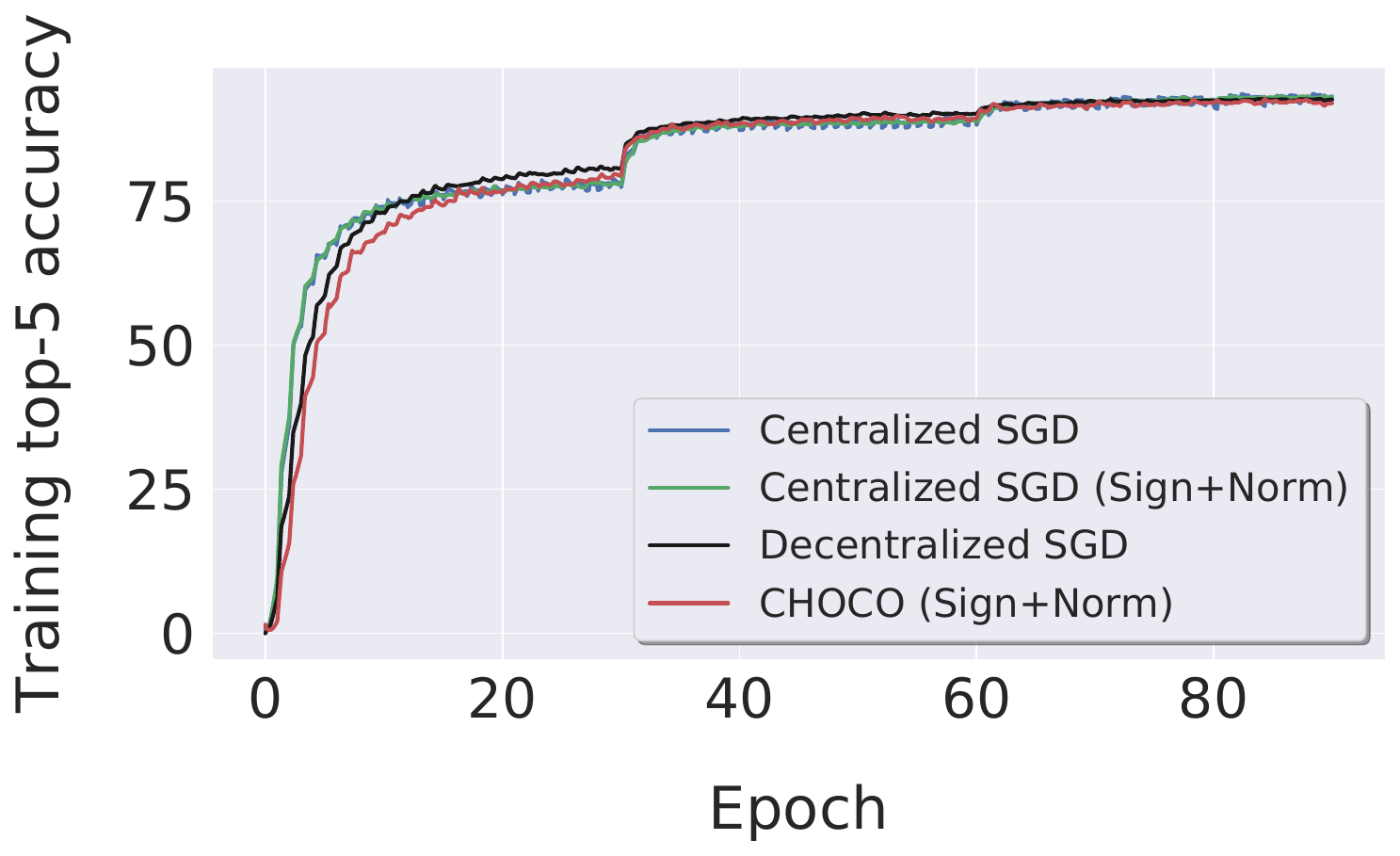}
		\label{fig:resnet20_cifar10_k32_bs32_social_topology_tr_top5_vs_epoch}
	}
	\hfill
	\subfigure[\small{Test top-1 accuracy.\vspace{-2mm}}]
	{
		\includegraphics[width=0.31\textwidth,]{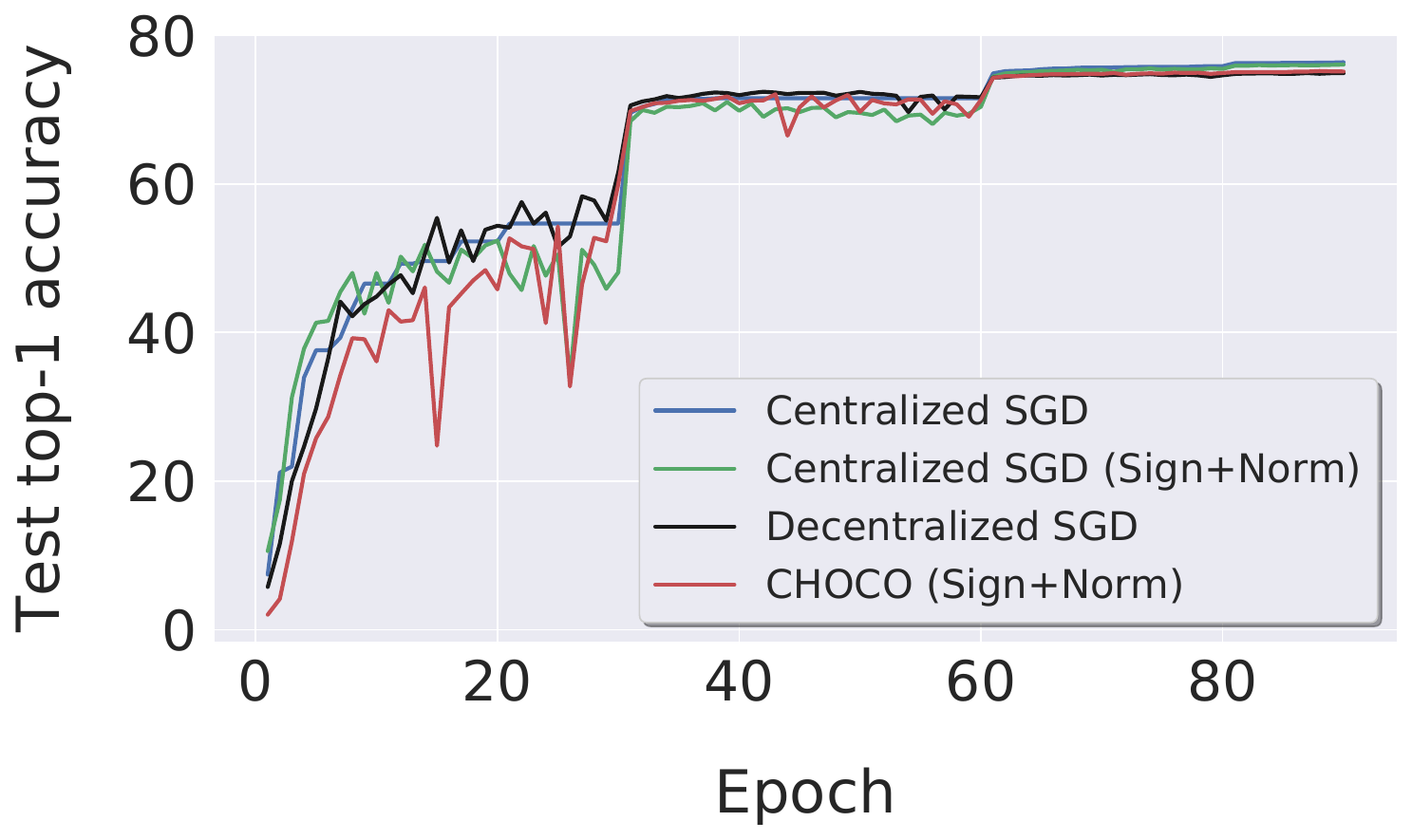}
		\label{fig:resnet20_cifar10_k32_bs32_social_topology_te_top1_vs_bits}
	}
	\caption{\small{
			Large-scale training: \texttt{ResNet-50} on \texttt{ImageNet} in the datacenter.
	}}
	\label{fig:datacenter_additional_resnet50_imagenet}
\end{figure*}

%% file: iclr2020_choco.bbl
\begin{thebibliography}{59}
\providecommand{\natexlab}[1]{#1}
\providecommand{\url}[1]{\texttt{#1}}
\expandafter\ifx\csname urlstyle\endcsname\relax
  \providecommand{\doi}[1]{doi: #1}\else
  \providecommand{\doi}{doi: \begingroup \urlstyle{rm}\Url}\fi

\bibitem[Alistarh et~al.(2017)Alistarh, Grubic, Li, Tomioka, and
  Vojnovic]{Alistarh2017:qsgd}
Dan Alistarh, Demjan Grubic, Jerry Li, Ryota Tomioka, and Milan Vojnovic.
\newblock
  \href{http://papers.nips.cc/paper/6768-qsgd-communication-efficient-sgd-via-gradient-quantization-and-encoding.pdf}{{QSGD}:
  Communication-efficient {SGD} via gradient quantization and encoding}.
\newblock In \emph{NIPS - Advances in Neural Information Processing Systems
  30}, pp.\  1709--1720. Curran Associates, Inc., 2017.

\bibitem[Alistarh et~al.(2018)Alistarh, Hoefler, Johansson, Konstantinov,
  Khirirat, and Renggli]{Alistarh2018:topk}
Dan Alistarh, Torsten Hoefler, Mikael Johansson, Nikola Konstantinov, Sarit
  Khirirat, and Cedric Renggli.
\newblock
  \href{http://papers.nips.cc/paper/7837-the-convergence-of-sparsified-gradient-methods.pdf}{The
  convergence of sparsified gradient methods}.
\newblock In \emph{NeurIPS - Advances in Neural Information Processing Systems
  31}, pp.\  5977--5987. Curran Associates, Inc., 2018.

\bibitem[Assran et~al.(2019)Assran, Loizou, Ballas, and
  Rabbat]{Assran:2018sdggradpush}
Mahmoud Assran, Nicolas Loizou, Nicolas Ballas, and Mike Rabbat.
\newblock \href{http://proceedings.mlr.press/v97/assran19a.html}{Stochastic
  gradient push for distributed deep learning}.
\newblock In \emph{ICML - Proceedings of the 36th International Conference on
  Machine Learning}, volume~97, pp.\  344--353. PMLR, 09--15 Jun 2019.

\bibitem[Berahas et~al.(2019)Berahas, Iakovidou, and Wei]{Berahas2019:adaptive}
Albert~S. Berahas, Charikleia Iakovidou, and Ermin Wei.
\newblock Nested distributed gradient methods with adaptive quantized
  communication.
\newblock \emph{arXiv preprint}, pp.\  arXiv:1903.08149, 2019.

\bibitem[Bernstein et~al.(2018)Bernstein, Wang, Azizzadenesheli, and
  Anandkumar]{Bernstein2018:sign}
Jeremy Bernstein, Yu-Xiang Wang, Kamyar Azizzadenesheli, and Animashree
  Anandkumar.
\newblock \href{http://proceedings.mlr.press/v80/bernstein18a.html}{sign{SGD}:
  Compressed optimisation for non-convex problems}.
\newblock In \emph{ICML - Proceedings of the 35th International Conference on
  Machine Learning}, volume~80 of \emph{Proceedings of Machine Learning
  Research}, pp.\  560--569, Stockholmsmässan, Stockholm Sweden, 10--15 Jul
  2018. PMLR.

\bibitem[Boyd et~al.(2006)Boyd, Ghosh, Prabhakar, and
  Shah]{Boyd2006:randgossip}
Stephen Boyd, Arpita Ghosh, Balaji Prabhakar, and Devavrat Shah.
\newblock \href{https://doi.org/10.1109/TIT.2006.874516}{Randomized gossip
  algorithms}.
\newblock \emph{IEEE/ACM Trans. Netw.}, 14\penalty0 (SI):\penalty0 2508--2530,
  June 2006.

\bibitem[Carli et~al.(2007)Carli, Fagnani, Frasca, Taylor, and
  Zampieri]{Carli2007:noise}
R.~Carli, F.~Fagnani, P.~Frasca, T.~Taylor, and S.~Zampieri.
\newblock Average consensus on networks with transmission noise or
  quantization.
\newblock In \emph{2007 European Control Conference (ECC)}, pp.\  1852--1857,
  July 2007.

\bibitem[Carli et~al.(2010{\natexlab{a}})Carli, Bullo, and
  Zampieri]{Carli2010:codingschemes}
R.~Carli, F.~Bullo, and S.~Zampieri.
\newblock Quantized average consensus via dynamic coding/decoding schemes.
\newblock \emph{International Journal of Robust and Nonlinear Control},
  20:\penalty0 156--175, 2010{\natexlab{a}}.

\bibitem[Carli et~al.(2010{\natexlab{b}})Carli, Frasca, Fagnani, and
  Zampieri]{Carli2010:quantizedconsensus}
R.~Carli, P.~Frasca, F.~Fagnani, and S.~Zampieri.
\newblock Gossip consensus algorithms via quantized communication.
\newblock \emph{Automatica}, 46:\penalty0 70--80, 2010{\natexlab{b}}.

\bibitem[Davis et~al.(1941)Davis, Gardner, and Gardner]{Davis:social_network}
A.~Davis, B.~B. Gardner, and M.~R. Gardner.
\newblock {Deep South}.
\newblock \emph{University of Chicago Press, Chicago, IL.}, May 1941.

\bibitem[Dekel et~al.(2012)Dekel, Gilad-Bachrach, Shamir, and
  Xiao]{Dekel2012:minibatch}
Ofer Dekel, Ran Gilad-Bachrach, Ohad Shamir, and Lin Xiao.
\newblock \href{http://dl.acm.org/citation.cfm?id=2503308.2188391}{Optimal
  distributed online prediction using mini-batches}.
\newblock \emph{J. Mach. Learn. Res.}, 13\penalty0 (1):\penalty0 165--202,
  January 2012.

\bibitem[Deng et~al.(2009)Deng, Dong, Socher, Li, Li, and
  Fei-Fei]{imagenet_cvpr09}
J.~Deng, W.~Dong, R.~Socher, L.-J. Li, K.~Li, and L.~Fei-Fei.
\newblock {ImageNet: A Large-Scale Hierarchical Image Database}.
\newblock In \emph{CVPR09}, 2009.

\bibitem[Dimakis et~al.(2010)Dimakis, Kar, Moura, Rabbat, and
  Scaglione]{Dimakis2010:survey}
A.~G. Dimakis, S.~Kar, J.~M.~F. Moura, M.~G. Rabbat, and A.~Scaglione.
\newblock Gossip algorithms for distributed signal processing.
\newblock \emph{Proceedings of the IEEE}, 98\penalty0 (11):\penalty0
  1847--1864, Nov 2010.

\bibitem[{Doan} et~al.(2018){Doan}, {Theja Maguluri}, and
  {Romberg}]{Thinh:2018AdaptiveDecentralizedQuantized}
Thinh~T. {Doan}, Siva {Theja Maguluri}, and Justin {Romberg}.
\newblock Accelerating the convergence rates of distributed subgradient methods
  with adaptive quantization.
\newblock \emph{arXiv preprint arXiv:1810.13245}, art. arXiv:1810.13245, 2018.

\bibitem[Goyal et~al.(2017)Goyal, Doll{\'a}r, Girshick, Noordhuis, Wesolowski,
  Kyrola, Tulloch, Jia, and He]{goyal2017accurate}
Priya Goyal, Piotr Doll{\'a}r, Ross Girshick, Pieter Noordhuis, Lukasz
  Wesolowski, Aapo Kyrola, Andrew Tulloch, Yangqing Jia, and Kaiming He.
\newblock Accurate, large minibatch {SGD}: Training {ImageNet} in 1 hour.
\newblock \emph{arXiv preprint arXiv:1706.02677}, 2017.

\bibitem[He et~al.(2016)He, Zhang, Ren, and Sun]{He2016:Resnet}
Kaiming He, Xiangyu Zhang, Shaoqing Ren, and Jian Sun.
\newblock Deep residual learning for image recognition.
\newblock \emph{2016 IEEE Conference on Computer Vision and Pattern Recognition
  (CVPR)}, pp.\  770--778, 2016.

\bibitem[He et~al.(2018)He, Bian, and Jaggi]{cola2018nips}
Lie He, An~Bian, and Martin Jaggi.
\newblock
  \href{http://papers.nips.cc/paper/7705-cola-decentralized-linear-learning.pdf}{Cola:
  Decentralized linear learning}.
\newblock In \emph{NeurIPS - Advances in Neural Information Processing Systems
  31}, pp.\  4541--4551. 2018.

\bibitem[Hochreiter \& Schmidhuber(1997)Hochreiter and
  Schmidhuber]{Hochreiter1997:LSTM}
Sepp Hochreiter and Jürgen Schmidhuber.
\newblock Long short-term memory.
\newblock \emph{Neural computation}, 9:\penalty0 1735--80, 12 1997.

\bibitem[Horv{\'a}th et~al.(2019)Horv{\'a}th, Kovalev, Mishchenko,
  Richt{\'a}rik, and Stich]{Horvath2019:vr}
Samuel Horv{\'a}th, Dmitry Kovalev, Konstantin Mishchenko, Peter Richt{\'a}rik,
  and Sebastian~Urban Stich.
\newblock \href{https://arxiv.org/abs/1904.05115}{Stochastic distributed
  learning with gradient quantization and variance reduction}.
\newblock \emph{arXiv preprint arXiv:1904.05115}, 2019.

\bibitem[Kairouz et~al.(2019)Kairouz, McMahan, and et. al.~\emph{including}
  Sebastian U.~Stich]{Kairouz2019:federated}
Peter Kairouz, H.~Brendan McMahan, and et. al.~\emph{including} Sebastian
  U.~Stich.
\newblock \href{https://arxiv.org/abs/1912.04977}{Advances and open problems in
  federated learning}.
\newblock \emph{arXiv preprint}, pp.\  arXiv:1912.04977, 2019.

\bibitem[Karimireddy et~al.(2019)Karimireddy, Rebjock, Stich, and
  Jaggi]{KarimireddyRSJ2019feedback}
Sai~Praneeth Karimireddy, Quentin Rebjock, Sebastian Stich, and Martin Jaggi.
\newblock \href{http://proceedings.mlr.press/v97/karimireddy19a.html}{Error
  feedback fixes {S}ign{SGD} and other gradient compression schemes}.
\newblock In \emph{ICML - Proceedings of the 36th International Conference on
  Machine Learning}, volume~97, pp.\  3252--3261. PMLR, 09--15 Jun 2019.

\bibitem[Kempe et~al.(2003)Kempe, Dobra, and Gehrke]{Kempe2003:gossip}
David Kempe, Alin Dobra, and Johannes Gehrke.
\newblock \href{http://dl.acm.org/citation.cfm?id=946243.946317}{Gossip-based
  computation of aggregate information}.
\newblock In \emph{Proceedings of the 44th Annual IEEE Symposium on Foundations
  of Computer Science}, FOCS '03, pp.\  482--, Washington, DC, USA, 2003. IEEE
  Computer Society.

\bibitem[Koloskova et~al.(2019)Koloskova, Stich, and
  Jaggi]{Koloskova:2019choco}
Anastasia Koloskova, Sebastian Stich, and Martin Jaggi.
\newblock
  \href{http://proceedings.mlr.press/v97/koloskova19a.html}{Decentralized
  stochastic optimization and gossip algorithms with compressed communication}.
\newblock In \emph{ICML - Proceedings of the 36th International Conference on
  Machine Learning}, volume~97, pp.\  3478--3487. PMLR, 2019.

\bibitem[Krizhevsky(2012)]{Krizhevsky2012:cifar10}
Alex Krizhevsky.
\newblock Learning multiple layers of features from tiny images.
\newblock \emph{University of Toronto}, 05 2012.

\bibitem[Lian et~al.(2017)Lian, Zhang, Zhang, Hsieh, Zhang, and
  Liu]{Lian2017:decentralizedSGD}
Xiangru Lian, Ce~Zhang, Huan Zhang, Cho-Jui Hsieh, Wei Zhang, and Ji~Liu.
\newblock
  \href{http://papers.nips.cc/paper/7117-can-decentralized-algorithms-outperform-centralized-algorithms-a-case-study-for-decentralized-parallel-stochastic-gradient-descent.pdf}{Can
  decentralized algorithms outperform centralized algorithms? a case study for
  decentralized parallel stochastic gradient descent}.
\newblock In \emph{NIPS - Advances in Neural Information Processing Systems
  30}, pp.\  5330--5340. Curran Associates, Inc., 2017.

\bibitem[Lin et~al.(2020)Lin, Stich, Patel, and Jaggi]{lin2020dont}
Tao Lin, Sebastian~U. Stich, Kumar~Kshitij Patel, and Martin Jaggi.
\newblock \href{https://openreview.net/forum?id=B1eyO1BFPr}{Don't use large
  mini-batches, use local {SGD}}.
\newblock In \emph{ICLR - International Conference on Learning
  Representations}, 2020.

\bibitem[Lin et~al.(2018)Lin, Han, Mao, Wang, and Dally]{Lin2018:deep}
Yujun Lin, Song Han, Huizi Mao, Yu~Wang, and Bill Dally.
\newblock \href{https://openreview.net/forum?id=SkhQHMW0W}{Deep gradient
  compression: Reducing the communication bandwidth for distributed training}.
\newblock In \emph{ICLR - International Conference on Learning
  Representations}, 2018.

\bibitem[McMahan et~al.(2017)McMahan, Moore, Ramage, Hampson, and
  Arcas]{McMahan:2017fedAvg}
Brendan McMahan, Eider Moore, Daniel Ramage, Seth Hampson, and Blaise Aguera~y
  Arcas.
\newblock {Communication-Efficient Learning of Deep Networks from Decentralized
  Data}.
\newblock In \emph{AISTATS 2017 - Proceedings of the 20th International
  Conference on Artificial Intelligence and Statistics}, pp.\  1273--1282,
  2017.

\bibitem[McMahan et~al.(2016)McMahan, Moore, Ramage, and
  y~Arcas]{McMahan16:FedLearning}
H.~Brendan McMahan, Eider Moore, Daniel Ramage, and Blaise~Ag{\"{u}}era
  y~Arcas.
\newblock \href{http://arxiv.org/abs/1602.05629}{Federated learning of deep
  networks using model averaging}.
\newblock \emph{arXiv preprint arXiv:1602.05629}, 2016.

\bibitem[Merity et~al.(2016)Merity, Xiong, Bradbury, and
  Socher]{merity2016pointer}
Stephen Merity, Caiming Xiong, James Bradbury, and Richard Socher.
\newblock Pointer sentinel mixture models.
\newblock \emph{arXiv preprint arXiv:1609.07843}, 2016.

\bibitem[Merity et~al.(2017)Merity, Keskar, and Socher]{merity2017regularizing}
Stephen Merity, Nitish~Shirish Keskar, and Richard Socher.
\newblock Regularizing and optimizing {LSTM} language models.
\newblock \emph{arXiv preprint arXiv:1708.02182}, 2017.

\bibitem[Mishchenko et~al.(2019)Mishchenko, Gorbunov, Tak\'{a}\v{c}, and
  Richt\'{a}rik]{Mishchenko2019:diana}
Konstantin Mishchenko, Eduard Gorbunov, Martin Tak\'{a}\v{c}, and Peter
  Richt\'{a}rik.
\newblock Distributed learning with compressed gradient differences.
\newblock \emph{arXiv preprint arXiv:1901.09269}, 2019.

\bibitem[Nedi\'{c} et~al.(2018)Nedi\'{c}, Olshevsky, and
  Rabbat]{Nedic2018:toplogy}
A.~Nedi\'{c}, A.~Olshevsky, and M.~G. Rabbat.
\newblock Network topology and communication-computation tradeoffs in
  decentralized optimization.
\newblock \emph{Proceedings of the IEEE}, 106\penalty0 (5):\penalty0 953--976,
  May 2018.

\bibitem[Nedi\'{c} et~al.(2008)Nedi\'{c}, Olshevsky, Ozdaglar, and
  Tsitsiklis]{Nedic2008:quantizationeffects}
Angelia Nedi\'{c}, Alex Olshevsky, Asuman Ozdaglar, and John~N. Tsitsiklis.
\newblock Distributed subgradient methods and quantization effects.
\newblock In \emph{Proceedings of the 47th IEEE Conference on Decision and
  Control, CDC 2008}, pp.\  4177--4184, 2008.

\bibitem[Nesterov(2012)]{Nesterov:2012}
Yurii Nesterov.
\newblock \href{http://epubs.siam.org/doi/abs/10.1137/100802001}{Efficiency of
  coordinate descent methods on huge-scale optimization problems}.
\newblock \emph{SIAM Journal on Optimization}, 22\penalty0 (2):\penalty0
  341--362, 2012.

\bibitem[Recht et~al.(2011)Recht, Re, Wright, and Niu]{Recht2011:hogwild}
Benjamin Recht, Christopher Re, Stephen Wright, and Feng Niu.
\newblock
  \href{http://papers.nips.cc/paper/4390-hogwild-a-lock-free-approach-to-parallelizing-stochastic-gradient-descent.pdf}{Hogwild:
  A lock-free approach to parallelizing stochastic gradient descent}.
\newblock In \emph{NIPS - Advances in Neural Information Processing Systems
  24}, pp.\  693--701. Curran Associates, Inc., 2011.

\bibitem[Reisizadeh et~al.(2018)Reisizadeh, Mokhtari, Hassani, and
  Pedarsani]{Reisizadeh2018:DQGD}
Amirhossein Reisizadeh, Aryan Mokhtari, Hamed Hassani, and Ramtin Pedarsani.
\newblock An exact quantized decentralized gradient descent algorithm.
\newblock \emph{arXiv preprint arXiv:1806.11536}, 2018.

\bibitem[Reisizadeh et~al.(2019)Reisizadeh, Taheri, Mokhtari, Hassani, and
  Pedarsani]{Reisizadeh2019:quantimed}
Amirhossein Reisizadeh, Hossein Taheri, Aryan Mokhtari, Hamed Hassani, and
  Ramtin Pedarsani.
\newblock \href{https://arxiv.org/abs/1907.10595}{Robust and
  communication-efficient collaborative learning}.
\newblock \emph{arXiv e-prints}, pp.\  arXiv:1907.10595, 2019.

\bibitem[Richt{\'a}rik \& Tak{\'a}{\v{c}}(2016)Richt{\'a}rik and
  Tak{\'a}{\v{c}}]{Hydra}
Peter Richt{\'a}rik and Martin Tak{\'a}{\v{c}}.
\newblock Distributed coordinate descent method for learning with big data.
\newblock \emph{Journal of Machine Learning Research}, 17\penalty0
  (75):\penalty0 1--25, 2016.

\bibitem[Scaman et~al.(2017)Scaman, Bach, Bubeck, Lee, and
  Massouli{\'e}]{Scaman2017:optimal}
Kevin Scaman, Francis Bach, S{\'e}bastien Bubeck, Yin~Tat Lee, and Laurent
  Massouli{\'e}.
\newblock \href{http://proceedings.mlr.press/v70/scaman17a.html}{Optimal
  algorithms for smooth and strongly convex distributed optimization in
  networks}.
\newblock In \emph{ICML - Proceedings of the 34th International Conference on
  Machine Learning}, volume~70 of \emph{Proceedings of Machine Learning
  Research}, pp.\  3027--3036, International Convention Centre, Sydney,
  Australia, 06--11 Aug 2017. PMLR.

\bibitem[Scaman et~al.(2018)Scaman, Bach, Bubeck, Massouli\'{e}, and
  Lee]{Scaman2018:non-smooth}
Kevin Scaman, Francis Bach, Sebastien Bubeck, Laurent Massouli\'{e}, and
  Yin~Tat Lee.
\newblock
  \href{http://papers.nips.cc/paper/7539-optimal-algorithms-for-non-smooth-distributed-optimization-in-networks.pdf}{Optimal
  algorithms for non-smooth distributed optimization in networks}.
\newblock In \emph{NeurIPS - Advances in Neural Information Processing Systems
  31}, pp.\  2745--2754. Curran Associates, Inc., 2018.

\bibitem[Seide et~al.(2014)Seide, Fu, Droppo, Li, and Yu]{Seide2015:1bit}
Frank Seide, Hao Fu, Jasha Droppo, Gang Li, and Dong Yu.
\newblock
  \href{http://dblp.uni-trier.de/db/conf/interspeech/interspeech2014.html#SeideFDLY14}{1-bit
  stochastic gradient descent and its application to data-parallel distributed
  training of speech {DNN}s.}
\newblock In \emph{INTERSPEECH}, pp.\  1058--1062. ISCA, 2014.

\bibitem[Stich(2019)]{Stich2018:LocalSGD}
Sebastian~U. Stich.
\newblock Local sgd converges fast and communicates little.
\newblock \emph{ICLR - International Conference on Learning Representations},
  art. arXiv:1805.09767, 2019.

\bibitem[Stich \& Karimireddy(2019)Stich and Karimireddy]{stich2019error}
Sebastian~U Stich and Sai~Praneeth Karimireddy.
\newblock The error-feedback framework: Better rates for {SGD} with delayed
  gradients and compressed communication.
\newblock \emph{arXiv preprint arXiv:1909.05350}, 2019.

\bibitem[Stich et~al.(2017{\natexlab{a}})Stich, Raj, and Jaggi]{StichRJ17safe}
Sebastian~U. Stich, Anant Raj, and Martin Jaggi.
\newblock
  \href{http://papers.nips.cc/paper/7025-safe-adaptive-importance-sampling.pdf}{Safe
  adaptive importance sampling}.
\newblock In \emph{NIPS - Advances in Neural Information Processing Systems
  30}, pp.\  4381--4391. Curran Associates, Inc., 2017{\natexlab{a}}.

\bibitem[Stich et~al.(2017{\natexlab{b}})Stich, Raj, and
  Jaggi]{StichRJ17steepest}
Sebastian~U. Stich, Anant Raj, and Martin Jaggi.
\newblock \href{http://proceedings.mlr.press/v70/stich17a.html}{Approximate
  steepest coordinate descent}.
\newblock In \emph{ICML - 34th International Conference on Machine Learning},
  volume~70, pp.\  3251--3259. PMLR, 2017{\natexlab{b}}.

\bibitem[Stich et~al.(2018)Stich, Cordonnier, and
  Jaggi]{Stich2018:sparsifiedSGD}
Sebastian~U Stich, Jean-Baptiste Cordonnier, and Martin Jaggi.
\newblock
  \href{http://papers.nips.cc/paper/7697-sparsified-sgd-with-memory.pdf}{Sparsified
  {SGD} with memory}.
\newblock In \emph{NeurIPS - Advances in Neural Information Processing Systems
  31}, pp.\  4452--4463. 2018.

\bibitem[Strom(2015)]{Strom:2015wc}
Nikko Strom.
\newblock Scalable distributed {DNN} training using commodity {GPU} cloud
  computing.
\newblock In \emph{INTERSPEECH}, pp.\  1488--1492. ISCA, 2015.

\bibitem[Tang et~al.(2018)Tang, Gan, Zhang, Zhang, and
  Liu]{Tang2018:decentralized}
Hanlin Tang, Shaoduo Gan, Ce~Zhang, Tong Zhang, and Ji~Liu.
\newblock
  \href{http://papers.nips.cc/paper/7992-communication-compression-for-decentralized-training.pdf}{Communication
  compression for decentralized training}.
\newblock In \emph{NeurIPS - Advances in Neural Information Processing Systems
  31}, pp.\  7663--7673. Curran Associates, Inc., 2018.

\bibitem[Tang et~al.(2019)Tang, Lian, Qiu, Yuan, Zhang, Zhang, and
  Liu]{Tang2019:squeeze}
Hanlin Tang, Xiangru Lian, Shuang Qiu, Lei Yuan, Ce~Zhang, Tong Zhang, and
  Ji~Liu.
\newblock \href{https://arxiv.org/abs/1907.07346}{Deepsqueeze: Decentralization
  meets error-compensated compression}.
\newblock \emph{arXiv preprint arXiv:1907.07346}, 2019.

\bibitem[Wang \& Joshi(2018)Wang and Joshi]{Wang2018:cooperativeSGD}
Jianyu Wang and Gauri Joshi.
\newblock \href{http://arxiv.org/abs/1808.07576}{Cooperative {SGD:} {A} unified
  framework for the design and analysis of communication-efficient {SGD}
  algorithms}.
\newblock \emph{arXiv preprint arXiv:1808.07576}, 2018.

\bibitem[Wang et~al.(2019)Wang, Sahu, Yang, Joshi, and Kar]{Wang2019:matcha}
Jianyu Wang, Anit~Kumar Sahu, Zhouyi Yang, Gauri Joshi, and Soummya Kar.
\newblock \href{http://arxiv.org/abs/1905.09435}{{MATCHA:} speeding up
  decentralized {SGD} via matching decomposition sampling}.
\newblock \emph{CoRR}, abs/1905.09435, 2019.

\bibitem[Wangni et~al.(2018)Wangni, Wang, Liu, and
  Zhang]{Wangni2018:sparsification}
Jianqiao Wangni, Jialei Wang, Ji~Liu, and Tong Zhang.
\newblock
  \href{http://papers.nips.cc/paper/7405-gradient-sparsification-for-communication-efficient-distributed-optimization.pdf}{Gradient
  sparsification for communication-efficient distributed optimization}.
\newblock In \emph{NeurIPS - Advances in Neural Information Processing Systems
  31}, pp.\  1306--1316. Curran Associates, Inc., 2018.

\bibitem[Wen et~al.(2017)Wen, Xu, Yan, Wu, Wang, Chen, and
  Li]{Wen2017:terngrad}
Wei Wen, Cong Xu, Feng Yan, Chunpeng Wu, Yandan Wang, Yiran Chen, and Hai Li.
\newblock
  \href{http://papers.nips.cc/paper/6749-terngrad-ternary-gradients-to-reduce-communication-in-distributed-deep-learning.pdf}{Terngrad:
  Ternary gradients to reduce communication in distributed deep learning}.
\newblock In \emph{NIPS - Advances in Neural Information Processing Systems
  30}, pp.\  1509--1519. Curran Associates, Inc., 2017.

\bibitem[Xiao et~al.(2005)Xiao, Boyd, and Lall]{Xiao2005:drift}
L.~Xiao, S.~Boyd, and S.~Lall.
\newblock A scheme for robust distributed sensor fusion based on average
  consensus.
\newblock In \emph{IPSN 2005. Fourth International Symposium on Information
  Processing in Sensor Networks, 2005.}, pp.\  63--70, April 2005.

\bibitem[Xiao \& Boyd(2004)Xiao and Boyd]{Xiao2014:averaging}
Lin Xiao and Stephen Boyd.
\newblock
  \href{http://www.sciencedirect.com/science/article/pii/S0167691104000398}{Fast
  linear iterations for distributed averaging}.
\newblock \emph{Systems \& Control Letters}, 53\penalty0 (1):\penalty0 65--78,
  2004.

\bibitem[Yu et~al.(2019)Yu, Jin, and Yang]{Yu2019momentum}
Hao Yu, Rong Jin, and Sen Yang.
\newblock \href{http://proceedings.mlr.press/v97/yu19d.html}{On the linear
  speedup analysis of communication efficient momentum {SGD} for distributed
  non-convex optimization}.
\newblock In \emph{ICML - Proceedings of the 36th International Conference on
  Machine Learning}, volume~97, pp.\  7184--7193. PMLR, 09--15 Jun 2019.

\bibitem[Yuan et~al.(2012)Yuan, Xu, Zhao, and Rong]{Yuan2012:distributedquant}
Deming Yuan, Shengyuan Xu, Huanyu Zhao, and Lina Rong.
\newblock
  \href{http://www.sciencedirect.com/science/article/pii/S0167691112001193}{Distributed
  dual averaging method for multi-agent optimization with quantized
  communication}.
\newblock \emph{Systems \& Control Letters}, 61\penalty0 (11):\penalty0 1053 --
  1061, 2012.

\bibitem[Zhang et~al.(2016)Zhang, De~Sa, Mitliagkas, and
  R{\'e}]{zhang2016parallelLocalSGD}
Jian Zhang, Christopher De~Sa, Ioannis Mitliagkas, and Christopher R{\'e}.
\newblock Parallel {SGD}: When does averaging help?
\newblock \emph{arxiv preprint arXiv:1606.07365}, 2016.

\end{thebibliography}
